\documentclass[11pt,english]{article}

\usepackage{amssymb}
\usepackage{amsthm}

\usepackage[utf8]{inputenc} 
\usepackage[T1]{fontenc}    
\usepackage{lmodern}
\usepackage{nicefrac}       
\usepackage{mathtools}
\usepackage{lscape}

\usepackage[round]{natbib}

\usepackage{breakurl}            
\usepackage{booktabs}       
\usepackage{amsfonts}       
\usepackage{microtype}      
\usepackage{graphicx}
\usepackage{amsmath,mathrsfs,dsfont, mathtools}
\usepackage{algorithm}
\usepackage{algorithmic}
\usepackage{color}
\usepackage{subfigure}
\usepackage{multirow}
\usepackage{multicol}

\newtheorem{definition}{Definition}
\newtheorem{theorem}{Theorem}
\newtheorem{corollary}{Corollary}
\newtheorem{lemma}{Lemma}

\newtheorem{proposition}{Proposition}

\newtheorem{example}{Example}

\makeatletter
\newcommand*{\indep}{%
	\mathbin{%
		\mathpalette{\@indep}{}%
	}%
}
\newcommand*{\nindep}{%
	\mathbin{
		\mathpalette{\@indep}{\not}
	}%
}
\newcommand*{\@indep}[2]{%
	\sbox0{$#1\perp\m@th$}
	\sbox2{$#1=$}
	\sbox4{$#1\vcenter{}$}
	\rlap{\copy0}
	\dimen@=\dimexpr\ht2-\ht4-.2pt\relax
	\kern\dimen@
	{#2}%
	\kern\dimen@
	\copy0 
}
\makeatother

\title{A Local Method for Identifying Causal Relations\\ under Markov Equivalence}

\author{Zhuangyan Fang\textsuperscript{1} \;~
	Yue Liu\textsuperscript{2} \;~
	Zhi Geng\textsuperscript{3} \\ 
	Shengyu Zhu\textsuperscript{4} \;~
	Yangbo He\textsuperscript{1}\thanks{Correspondence to: heyb@pku.edu.cn .} \\
	~\\
	\textsuperscript{1}Peking University \quad \textsuperscript{2}Renmin University of China \\ \textsuperscript{3}Beijing Technology and Business University \\ \textsuperscript{4}Huawei Noah's Ark Lab}

\begin{document}

\maketitle

\begin{abstract}

Causality is important for designing  interpretable and robust methods in artificial intelligence research. We propose a local approach to identify  whether a  variable  is a  cause of a given target under the framework of causal graphical models of directed acyclic graphs (DAGs).   In general, the causal relation between two variables  may not be identifiable  from observational data  as  many causal DAGs encoding different causal relations  are Markov equivalent.  In this paper, we first introduce a sufficient and necessary graphical condition to check the existence of a causal path from a   variable to a target  in every Markov equivalent DAG. Next, we provide local criteria for identifying whether a variable is a cause/non-cause of a target based only on the local structure instead of   the entire graph.  Finally,  we propose  a local learning algorithm for  this causal query via  learning the local structure of the variable and some additional  statistical independence tests related to the target. Simulation studies show that our local algorithm is efficient and effective, compared with other state-of-art methods.
\end{abstract}

\section{Introduction}\label{sec:intro}


Causality is important for designing  interpretable and robust methods in artificial intelligence research~\citep{miller2019explanation}, and has been used in many fields of artificial intelligence, such as causal transfer learning \citep{zhang2020domain,Bengio2020A} and causality-based algorithmic fairness~\citep{counter, ijcai2019-199}. One of the main problems in many of these studies is to infer whether a treatment variable is a cause of a target variable, or to further identify  the causes/non-causes of  a specified target variable or the effects/non-effects of a given treatment. For example, in a telecommunication network, a single fault (or alarm) in the network can trigger a flood of alarms, and conversely, a recovery of a single fault may clear many alarms. Therefore, knowing the causal relations among the alarms (or faults) is helpful to localize the key failure points for fault recovery in practice.


Directed acyclic graphs (DAGs) can be used to represent causal relationships among variables~\citep{pearl2009causality}. Following Pearl's definition of \emph{inferred causation}~\citep[Definition~2.3.1]{pearl2009causality}, we call $X$ a \emph{cause} of $Y$ and $Y$ an effect of $X$ if $X$ has a directed path to $Y$ in the true DAG. From observational data, however, instead of an exact causal DAG,\footnote{We note that, the recent progresses in identifying the causal relation between two variables indeed provide an opportunity to learn an exact DAG. However, such methods need to pose additional distributional conditions~\citep{Shimizu2006lingam, Zhang2009identifiability, Shimizu2011directlingam, Peters2013identifiability,Peters2014causal}} we  generally  learn a Markov equivalence class of DAGs represented by a completed partially directed acyclic graph (CPDAG). The undirected edges in a CPDAG imply that some causal relations among variables can not be read from the graph directly. Therefore, given a Markov equivalence class of DAGs, a variable $X$ is a \emph{definite cause} of a target $Y$ if $X$ is always a cause of $Y$ in every equivalent DAG, and a variable $X$ is a \emph{definite non-cause} of $Y$ if $X$ is never a cause of $Y$ in any DAG in the class. If $X$ is neither a definite cause nor a definite non-cause of $Y$, $X$ is called a \emph{possible cause} of $Y$. 

Some approaches can be used to identify the type of causal relation between a  treatment and a target. An intuitive approach is first  to learn a Markov equivalence class from observational data, and then  enumerate all DAGs in the class to check whether the treatment is definitely or definitely not a cause of the target in all of these equivalent DAGs. However, the  intuitive approach is inefficient when the number of DAGs in the learned Markov equivalence class is large~\citep{he2015counting}.

Another way is to check the paths from the treatment to the target in a CPDAG. It has been shown that a treatment is a definite non-cause of a target if and only if there is no partially directed path from the treatment to the target \citep[see, e.g.][]{zhang2006, perkovic2017interpreting}. Given   a CPDAG, \citet{Roumpelaki} also introduced  a sufficient   condition for identifying definite causes. However, the necessity of this condition remains a conjecture ~\citep{zhang2006, Mooij2020constraint} and the corresponding  approach could be inefficient since  it needs to learn an entire CPDAG first. 


The third approach is to estimate all possible causal effects of the treatment on the target~\citep{maathuis2009estimating,  perkovic2017interpreting, nandy2017estimating,  fang2020bgida, liu2020cida, liu2020local, Witte2020efficient, Guo2020minimal}. This approach, which we call the causal-effect-based method, determines whether a treatment is a cause of a target by judging whether all possible causal effects are zeros/non-zeros based on a certain criterion or method, such as hypothesis testing.  However, the causal-effect-based method requires additional assumptions,\footnote{We remark that,
$X$ has a zero-valued causal effect on $Y$ does not necessarily mean that there is no directed path from $X$ to $Y$ (See \ref{app:ce} for an example). Nevertheless, with the causal faithfulness assumption as well as some model assumptions such as linear-Gaussianity, the former implies the latter.} and it could be time-consuming as the number of possible effects grows exponentially in the worst case.

In this paper, we study the problem of locally identifying causal relations under Markov equivalence with the assumption that there is no hidden variable or selection bias. That is, given a pair of treatment and target variables, we intend to decide whether the treatment is a definite cause, a possible cause or a definite non-cause of the target only based on a local induced subgraph  and a few independence tests  related to the treatment without learning an entire CPDAG. This local approach is usually more efficient than the global ones that need an entire CPDAG, especially when the underlying causal graph is large.

To this end, we first discuss the existence of a causal path from one variable to another given a CPDAG, and prove the necessity of the condition  in \citet{Roumpelaki} for CPDAGs  in Section \ref{sec:anatomy}. This yields a sufficient and necessary graphical condition to check the existence of a causal path.
Next,  in Section \ref{sec:characterization}  we propose  local identification criteria for definite causes,  possible causes and  definite non-causes separately. These criteria  depend  only on the induced subgraph of the true CPDAG over the adjacent variables of the treatment as well as some queries about d-separations, thus  directly lead to a local learning algorithm given in Section \ref{sec:learn}. For the completeness of the paper, a global algorithm and several causal-effect-based methods  for learning types of causal relations are also provided in Section \ref{sec:learn}. In Section \ref{sec:simulation}, we compare experimentally the proposed local learning method with the global and the causal-effect-based methods, and show the efficiency and efficacy of the proposed method. Finally, we discuss some applications and possible extensions of our work in Section \ref{sec: conclusion}, and give some graph  terminology,  additional algorithms, proofs and additional experimental results in Appendix A, B, C and D, respectively.

\section{Preliminaries and Related Work}\label{sec:background}

In this paper, we use $pa(\textbf{S}, \cal G)$, $ch(\textbf{S}, \cal G)$, $sib(\textbf{S}, \cal G)$, $adj(\textbf{S}, \cal G)$, $an(\textbf{S}, \cal G)$ and $de(\textbf{S}, \cal G)$ to denote the union of the parents, children, siblings (or undirected neighbors), adjacent vertices, ancestors, and descendants of each variable in set $\textbf{S}$ in $\mathcal{G}$, respectively, where ${\cal G}=(\mathbf{V}, \mathbf{E})$ can be a directed, an undirected, or a partially directed graph. The basic graph terminology can be found in \ref{app:graph}. As a convention, we regard a vertex as an ancestor and a descendant of itself.  If $\textbf{S}=\{X\}$ is a singleton set, we will replace $\textbf{S}$ by $X$ for ease of presentation. Let $\cal G$ be a causal \emph{directed acyclic graph} (causal DAG) and $X$ be a vertex in $\cal G$, the vertices in  $an(X, \mathcal{G})\setminus X$ are \emph{causes} of $X$, and the vertices in $pa(X, \mathcal{G})$ are \emph{direct causes} of $X$. If $X$ is a cause of $Y$, then the directed paths from $X$ to $Y$ are called \emph{causal paths}.

\subsection{Causal DAG Models}\label{sec:sec:causalDAG}

The notion of \emph{d-separation} induces a set of conditional independence relations encoded in a DAG \citep{pearl1988probabilistic}. Let $\mathcal{G}$ be a DAG and $\pi = (X = X_0, X_1, ... , X_n = Y)$ be a path from $X$ to $Y$ in $\mathcal{G}$. An intermediate vertex $X_i$ is a \emph{collider} on $\pi$ if $X_{i-1} \rightarrow X_i$ and $X_i \leftarrow X_{i+1}$, otherwise, $X_i$ is a \emph{non-collider} on $\pi$. For three distinct vertices $X_i, X_j$ and $X_k$, if $X_i\rightarrow X_j\leftarrow X_k$ and $X_i$ is not adjacent to $X_k$ in $\cal G$, then the triple $(X_i, X_j, X_k)$ is called a \emph{v-structure} collided on $X_j$ in $\cal G$. Given $\textbf{Z} \subseteq \textbf{V}$, we say $\pi$ is \emph{d-connected (or active)} given $\textbf{Z}$ if $\textbf{Z}$ does not contain any endpoint or non-collider on the path and every collider on the path has a descendant in $\textbf{Z}$. If $\pi$ is not d-connected given $\textbf{Z}$, then  $\pi$ is \emph{blocked} by $\textbf{Z}$.  For pairwise disjoint sets $\textbf{X}, \textbf{Y}, \textbf{Z} \subseteq \textbf{V}$, $\textbf{X}$ and $\textbf{Y}$ are d-separated by $\textbf{Z}$ (denoted by $\textbf{X} \indep \textbf{Y} \mid \textbf{Z}$) if and only if every path between some $X\in \textbf{X}$ and $Y\in \textbf{Y}$ is blocked by $\textbf{Z}$.

Let $\mathcal{J}_{\mathcal{G}}$ be the set of d-separation relations read off from a DAG $\cal G$. Two DAGs $\mathcal{G}_1$ and $\mathcal{G}_2$ are Markov \emph{equivalent} if $\mathcal{J}_{\mathcal{G}_1}=\mathcal{J}_{\mathcal{G}_2}$. \citet{pearl1989conditional} have shown that two DAGs are equivalent if and only if they have the same skeleton and the same v-structures.  A  \emph{Markov equivalence class} or simply \emph{equivalence class}, denoted by $[\mathcal{G}]$, contains all DAGs equivalent to $\mathcal{G}$. A Markov equivalence class $[\mathcal{G}]$ can be uniquely represented by a partially directed graph called \emph{completed partially directed acyclic graph} (CPDAG) $\mathcal{G}^*$. Two vertices  are adjacent in $\mathcal{G}^*$ if and only if they are adjacent in $\mathcal{G}$  and a directed edge occurs in $\mathcal{G}^*$ if and only if it appears in every DAG in $[\mathcal{G}]$~\citep{pearl1989conditional}. For the ease of presentation, we will also use $[\mathcal{G}^*]$ to represent the Markov equivalence class represented by $\mathcal{G}^*$. Given a CPDAG $\mathcal{G}^*$, we use $\mathcal{G}^*_u$ and $\mathcal{G}^*_d$, which consist of all undirected edges and all directed edges in ${\cal G}^*$,  to denote the \emph{undirected subgraph} and the \emph{directed subgraph} of ${\cal G}^*$, respectively. 
\citet{andersson1997characterization} proved that (1) the undirected subgraph $\mathcal{G}^*_u$ of $\mathcal{G}^*$ is the union of disjoint connected chordal graphs (the definition of chordal graph is provided in \ref{app:graph}), and (2) every partially directed cycle in $\mathcal{G}^*$ is an undirected cycle, that is, none of the partially directed cycles in $\mathcal{G}^*$ contains a directed edge.   Each isolated connected undirected subgraph  of $\mathcal{G}^*_u$ is called a  \emph{chain component} of ${\cal G}^*$~\citep{andersson1997characterization, lauritzen2002chain}.

For a given distribution $P$, we use $\textbf{X} \indep_{P} \textbf{Y} \mid \textbf{Z}$  to denote that $\textbf{X}$ is  independent of $\textbf{Y}$ given $\textbf{Z}$ with respect to $P$, where $\textbf{X}, \textbf{Y}, \textbf{Z} \subseteq \textbf{V}$ are pairwise disjoint.
If both $\textbf{X}=\{X\}$ and $\textbf{Y}=\{Y\}$ are singleton sets, we allow  that  $X$ or $Y \in \textbf{Z}$,  and  assume that $X\indep Y\mid \textbf{Z}$ trivially holds  in this case.
Let $\mathcal{J}_{P}$ be the set of all (conditional) independencies that hold with respect to $P$. The main results of this paper are based on the following assumptions: the \emph{causal Markov assumption}, which states that $\textbf{X} \indep \textbf{Y} \mid \textbf{Z} $ in $\mathcal{J}_{\mathcal{G}} $  implies $\textbf{X} \indep_P \textbf{Y} \mid \textbf{Z} $ in $\mathcal{J}_{P}$; the \emph{causal faithfulness assumption}, which states that $\textbf{X} \indep_P \textbf{Y} \mid \textbf{Z} $ in $ \mathcal{J}_{P}$  implies $\textbf{X} \indep \textbf{Y} \mid \textbf{Z}  $ in $ \mathcal{J}_{\mathcal{G}} $; and the assumption that there is no hidden variable or selection bias. A distribution $P$ is called \emph{Markovian} and \emph{faithful} to a DAG $\mathcal{G}$ if $P$ and ${\cal G}$ satisfy the causal Markov assumption and the causal faithfulness assumption. A causal DAG model consists of a DAG $\cal G$ and a joint distribution $P$ over a common vertex set $\mathbf{V}$ such that $P$ satisfies the causal Markov assumption with respect to $\cal G$. $\cal G$ is called the \emph{causal structure} of the model and $P$ is called the \emph{observational distribution} (or simply \emph{distribution})~\citep{hauser2012characterization}.

\subsection{Global and Local Causal Structure Learning}\label{sec:sec:structure_learning}

Causal structure learning methods try to recover the causal structure from data.
Global causal structure learning focuses on learning an entire causal structure over all variables while local causal structure learning aims to recover only a part of the underlying causal structure.

Existing approaches for learning global causal structures roughly fall into two classes: constraint-based and score-based methods. Constraint-based methods, such as the PC algorithm~\citep{spirtes1991algorithm} and the stable PC algorithm~\citep{colombo2014order}, use conditional independence tests to find causal skeleton and then determine the edge directions according to a series of orientation rules~\citep{meek1995causal}. Under the causal Markov and causal faithfulness assumptions,  constraint-based methods can identify causal graphs up to a Markov equivalence class. On the other hand, score-based methods, such as exact search algorithms like dynamic programming \citep{Koivisto2004exact,Singh2005finding} and A* \citep{Yuan2011learning,Xiang2013lasso}, greedy search algorithms like GES \citep{chickeringo2002optimal}, and gradient-based methods like NOTEARS~\citep{zheng2018dags}, evaluate candidate graphs with a predefined score function and search for the optimal DAGs or CPDAGs.

Local learning algorithms usually learn the Markov blanket~\citep[see, e.g.][]{tsamardinos2003algorithms, Tsamardinos2003towards, fu2010} or the parent and child set of a given target~\citep[see, e.g.][]{wang2014discovering, gao2015local, liu2019}. Recently, \citet[Algorithm~3]{liu2020local} extended the MB-by-MB algorithm \citep{wang2014discovering} to learn the chain component containing a given target and the directed edges surrounding  the chain component. This variant of MB-by-MB can thus learn the induced subgraph of the true CPDAG over the target and its neighbors, that is, the parents, siblings and children of the target in the CPDAG.

\subsection{Related Work}
As discussed in Section \ref{sec:intro}, when a learned CPDAG is provided, one can either enumerate all equivalent DAGs, or check the paths in the CPDAG~\citep{zhang2006, Roumpelaki, perkovic2017interpreting}, or use the causal-effect-based method to identify types of causal relations~\citep{maathuis2009estimating, perkovic2017interpreting}.

Many   sufficient conditions are also available  to  identify  some of causal relations without estimating a global causal structure   \citep{cooper97, spirtes2000causation, mani06y, pearl2009causality, Claassen11alogical, colombo2014order, Magliacane2016ancestral}.  
For example, if $X\nindep Y\mid \mathbf{W}\cup Z$ while $X\indep Y\mid \mathbf{W}$, then $Z$ is a definite non-cause of every variable in $X\cup Y \cup \mathbf{W}$~\citep{Claassen11alogical}.   Since these rules are sound but not complete, they may fail to identify the causal relation of a given pair of treatment and target.


  Recently, a related work from \cite{entner13a} proposed sound and complete rules for inferring whether a given variable $X$ has  a causal effect  on another variable $Y$.
Compared with our work, their criteria allow the existence of unmeasured confounders, but also  require two additional assumptions: $Y$ is not a cause of $X$, and neither $X$ nor $Y$ is a cause of other observed variables.

\section{An Anatomy of Causal Relations}\label{sec:anatomy}

In this section, we provide a sufficient and necessary condition to identify definite causal relations, and show that definite causal relations can be divided into two subtypes: explicit  and implicit causal relations.

\subsection{Graphical Criteria for Identifying Types of Causal Relations}\label{sec:sec:graphical}


As mentioned in Section \ref{sec:intro},
given a CPDAG, a variable $X$ is a definite non-cause of another variable $Y$ if and only if there is no partially directed path from $X$ to $Y$~\citep{zhang2006, perkovic2017interpreting}.
\citet[Theorem~3.1]{Roumpelaki} proved that a treatment is a definite cause of a target if there is a directed path from the treatment to the target or the treatment has two chordless partially directed paths to the target on which two vertices adjacent to the treatment are distinct and non-adjacent. In the section, we will show that this condition is also necessary, and before that,  a concept of   \emph{critical set} is introduced as follows.

\begin{definition}[Critical Set]\label{def:critical-set}
	{\rm {\citep[Definition~2]{fang2020bgida}}}
	Let $\mathcal{G}^*$ be a CPDAG, and  $X$ and $Y$ be two distinct vertices in $\mathcal{G}^*$. The critical set of $X$ with respect to $Y$ in $\mathcal{G}^*$ consists of all adjacent vertices of $X$ lying on at least one chordless partially directed path from $X$ to $Y$.
\end{definition}

The definition of chordless partially directed path can be found in \ref{app:graph}. With Definition \ref{def:critical-set}, we have the following lemma.

\begin{lemma}\label{lem:child_critical}
	Let $\mathcal{G}^*$ be a CPDAG. For any two distinct vertices $X$ and $Y$ in $\mathcal{G}^*$, $X$ is a definite cause of $Y$ in the underlying DAG if and only if the critical set of $X$ with respect to $Y$ in $\mathcal{G}^*$  contains a child of $X$ in every DAG ${\cal G}\in[\mathcal{G}^*]$.
\end{lemma}

Lemma \ref{lem:child_critical} follows from Lemma~2 in \citet{fang2020bgida}. It gives a sufficient and necessary condition to decide whether $X$ is a definite cause of $Y$. However, checking the condition given in Lemma \ref{lem:child_critical} also requires to enumerate all equivalent DAGs. To mitigate this problem, we discuss a graphical characteristic of  critical set in the corresponding CPDAG. 


\begin{lemma}\label{lem:critical_has_a_child}
	Let $\mathcal{G}^*$ be a CPDAG and $X, Y$ be two distinct vertices in $\mathcal{G}^*$. Denote by $\mathbf{C}$ the critical set of $X$ with respect to $Y$ in $\mathcal{G}^*$, then $\mathbf{C}\cap ch(X, {\cal G})=\emptyset$ for some ${\cal G} \in [\mathcal{G}^*]$ if and only if $\mathbf{C} = \emptyset$, or $\mathbf{C}$ induces a complete subgraph of $\mathcal{G}^*$ but $\mathbf{C}\cap ch(X, \mathcal{G}^*)= \emptyset$.
\end{lemma}

Based on Lemmas \ref{lem:child_critical} and \ref{lem:critical_has_a_child}, we have the desired sufficient and necessary graphical criterion.

\begin{theorem}\label{thm:graphical_definite_cause}
	Suppose that $\mathcal{G}^*$ is a CPDAG, $X, Y$ are two distinct vertices in $\mathcal{G}^*$, and $\mathbf{C}$ is the critical set of $X$ with respect to $Y$ in $\mathcal{G}^*$. Then, $X$ is a definite cause of $Y$ if and only if $\mathbf{C}\cap ch(X, \mathcal{G}^*)\neq \emptyset$, or $\mathbf{C}$ is non-empty and induces an incomplete subgraph of $\mathcal{G}^*$.
\end{theorem}

The sufficiency of the condition in Theorem \ref{thm:graphical_definite_cause} has been extended to other types of causal graphs by \citet{Roumpelaki} and \citet{Mooij2020constraint}.\footnote{We note that, although ~\citet[Theorem~3.1]{Roumpelaki} also claimed that they have proved the necessity, their proof is flawed. As mentioned by \citet{Mooij2020constraint}, the last part of the proof appears to be incomplete. How to prove the necessity for more general types of causal graphs remains an open problem~\citep{zhang2006}.} With the help of Theorem \ref{thm:graphical_definite_cause}, we can identify the type of causal relation based on a learned CPDAG by enumerating paths and finding critical sets. Below, we give an example to illustrate this idea.

\begin{example}\label{exa:difinitecause}  Consider   the respiratory disease network shown in Figure \ref{fig:motivation_1}. The meanings of the node labels are given in the caption. Let smoking be the treatment and dyspnoea be the target. From Figure \ref{fig1_2} we can see that the partially directed paths from smoking to dyspnoea are ${\rm Smok}-{\rm Lung}\rightarrow{\rm Either}\rightarrow{\rm Dysp}$ and ${\rm Smok}-{\rm Bronc}\rightarrow{\rm Dysp}$. Therefore, the critical set of smoking with respect to dyspnoea is $\{{\rm Lung}, {\rm Bronc}\}$. As ${\rm Lung}$ and ${\rm Bronc}$ are not adjacent, by Theorem \ref{thm:graphical_definite_cause} smoking is a definite cause of dyspnoea. Similarly, the critical set of lung cancer with respect to dyspnoea is $\{{\rm Smok}, {\rm Either}\}$. Since ${\rm Either}$ is a child of ${\rm Lung}$, lung cancer is also a definite cause of dyspnoea.

\begin{figure}[!t]
	\centering
	\subfigure[$\mathcal{G}_{t}$]{
		\begin{minipage}[t]{0.31\linewidth}
			\centering
			\includegraphics[width=1.3in]{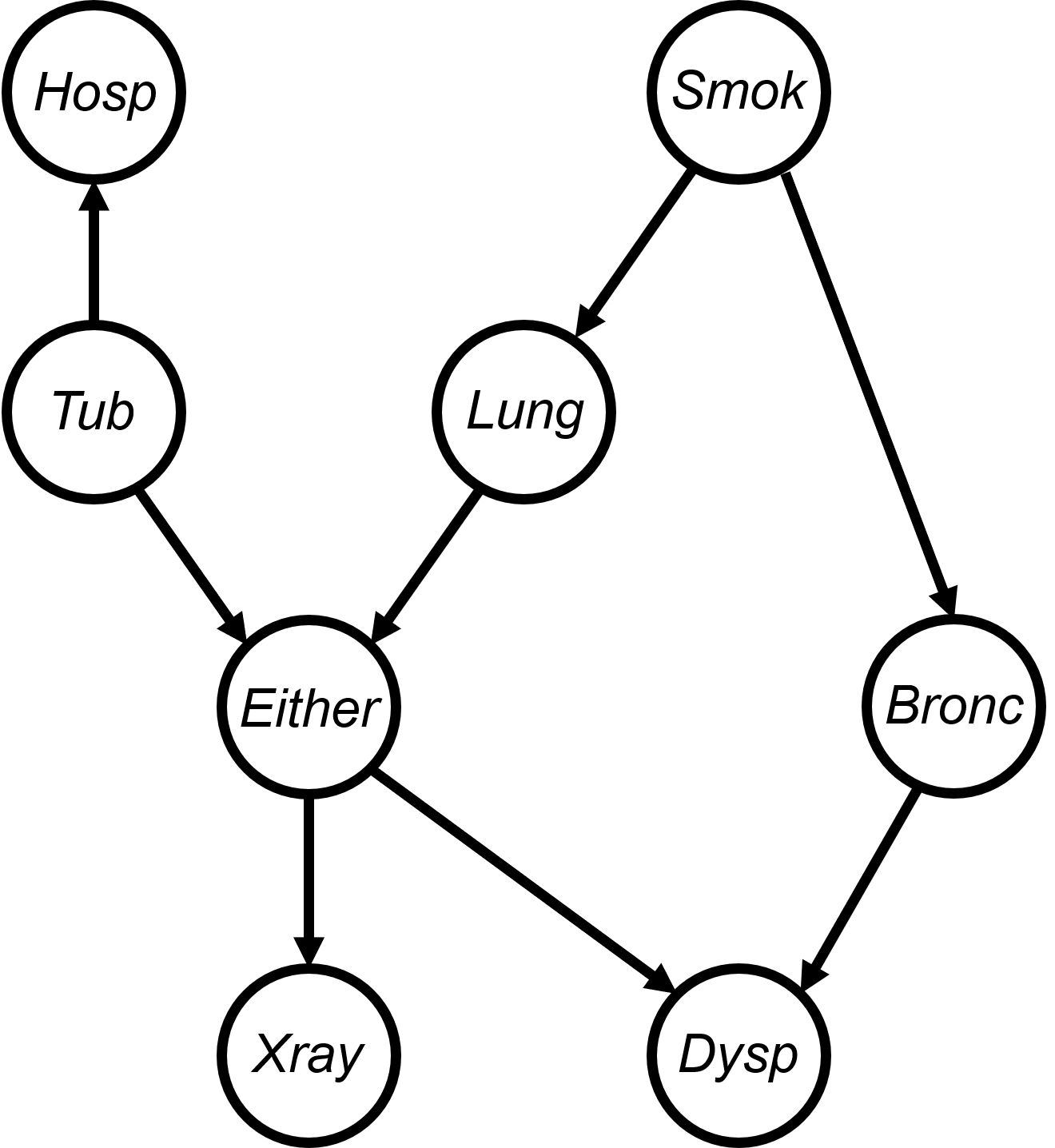}
			\vspace{0.35cm}
			\label{fig1_1}
		\end{minipage}%
	}%
	\subfigure[$\mathcal{G}^*$]{
		\begin{minipage}[t]{0.31\linewidth}
			\centering
			\includegraphics[width=1.3in]{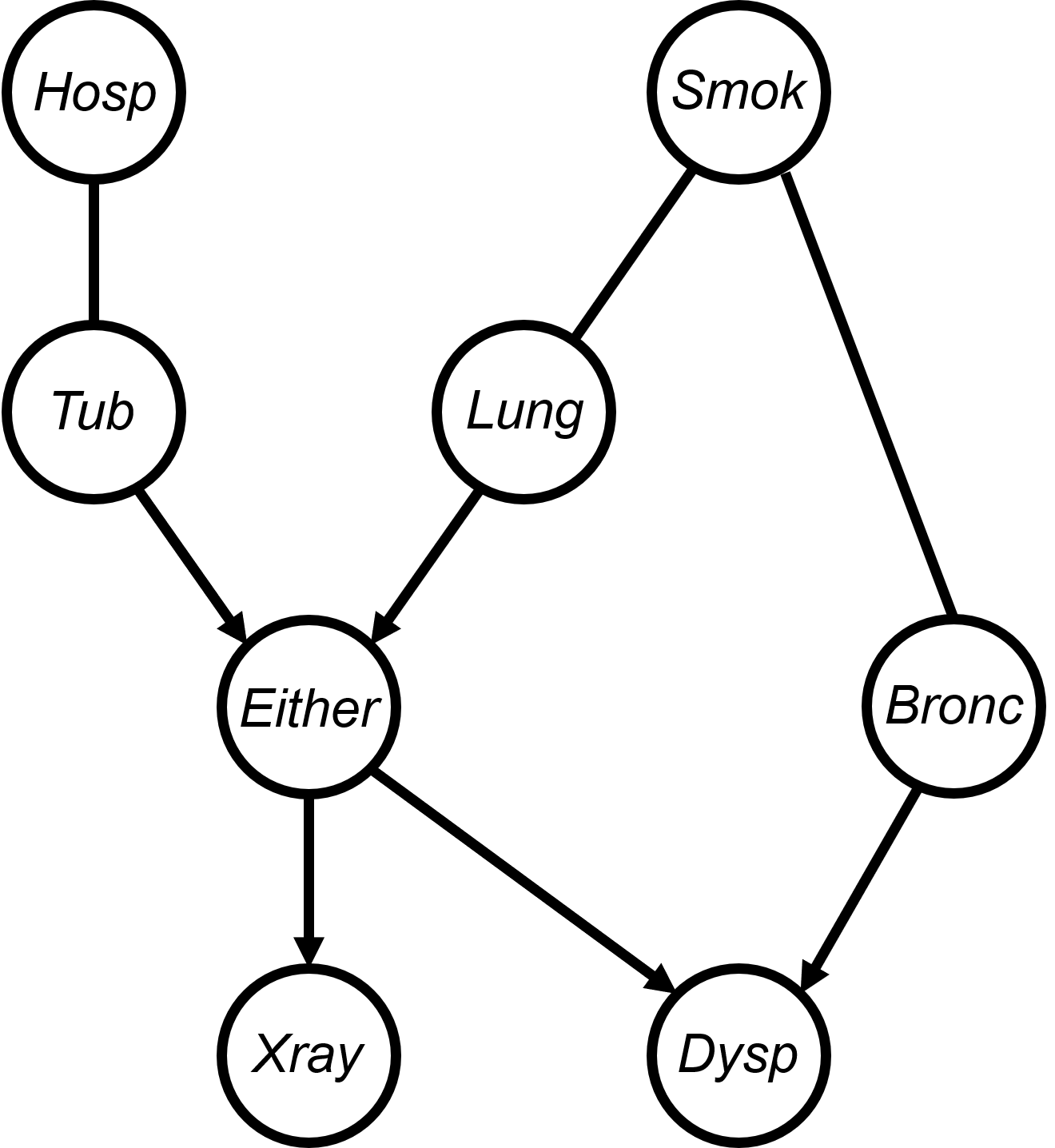}
			\vspace{0.35cm}
			\label{fig1_2}
		\end{minipage}%
	}%
	\subfigure[Markov eqivelence class]{
		\begin{minipage}[t]{0.35\linewidth}
			\centering
			\includegraphics[width=1.25in]{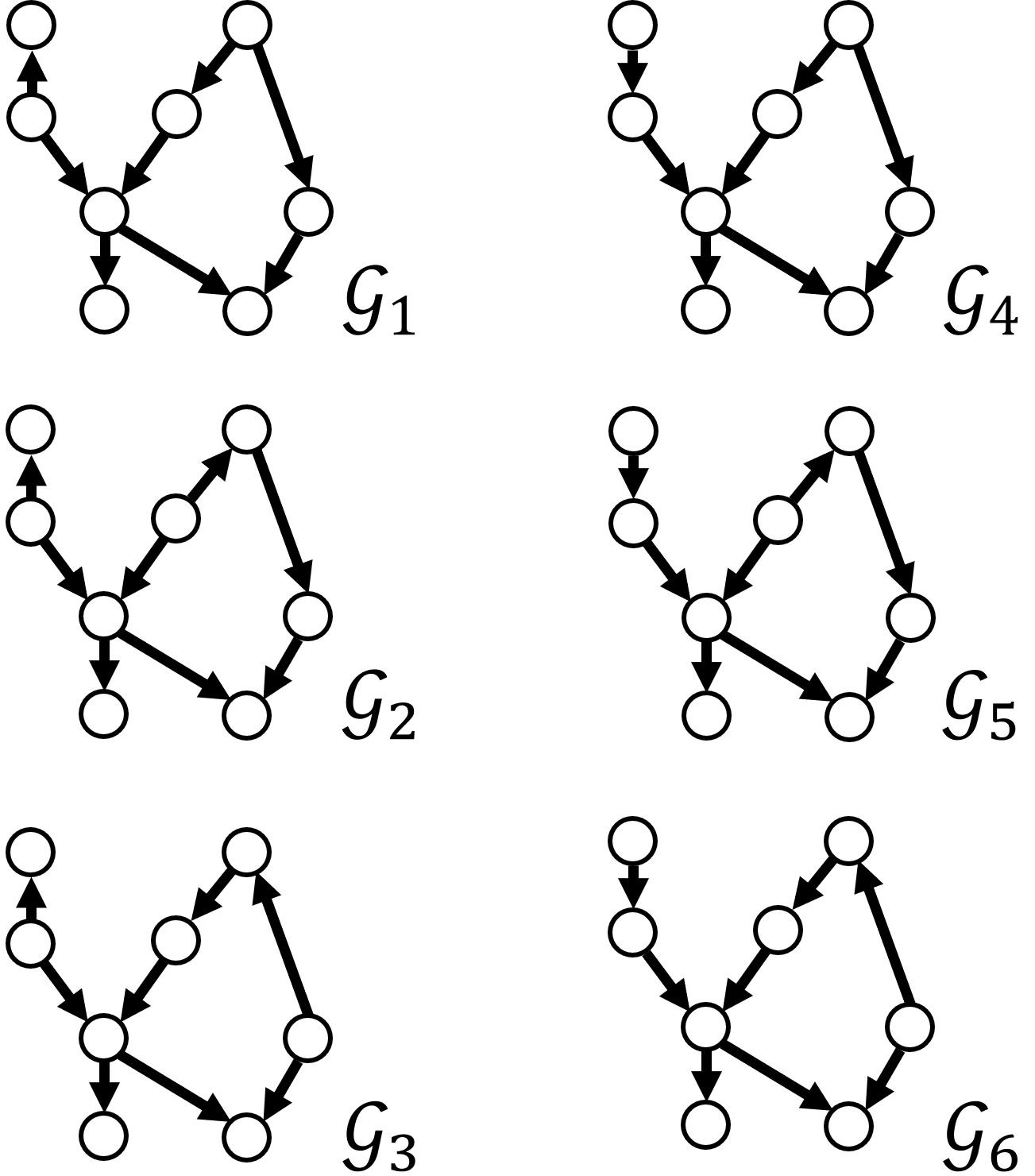}
			\vspace{0.35cm}
			\label{fig1_3}
		\end{minipage}
	}%
	\centering
	\caption{This example is adapted from the ASIA network. The original network structure and related parameters can be found in \citet{Lauritzen1988Local}. Figure \ref{fig1_1} shows the true underlying causal DAG, and Figure \ref{fig1_2} shows the corresponding CPDAG. Figure \ref{fig1_3} enumerates all equivalent DAGs in the Markov equivalence class. The meanings of the node labels are: recently have been to the \textbf{Hosp}ital, test positive for \textbf{Tub}erculosis, \textbf{Smok}ing, test positive for \textbf{Lung} cancer, \textbf{Bronc}hitis, \textbf{Either} have lung cancer or have tuberculosis,  test positive for \textbf{X-ray}, and test positive for \textbf{Dysp}noea.}
	\label{fig:motivation_1}
\end{figure}

\end{example}

\subsection{Explicit and Implicit Causal Relations}\label{sec:sec:concepts}
We now study the properties of definite causal relations, and show that definite causal relations can be divided into two subtypes based on the existence of causal paths in a CPDAG. The results in this section are of key importance to build local characterizations in Section \ref{sec:characterization}, and are also useful for developing an efficient global learning algorithm.


\begin{proposition}\label{prop:definite-cause}
	For two distinct  vertices $X$ and $Y$, if $X$ is a definite cause of $Y$, then $X$ and $Y$ are not in the same chain component.
\end{proposition}

Given a target variable $Y$, Proposition \ref{prop:definite-cause} shows that $Y$ and its  definite causes do not appear in the same chain component. Thus, if a treatment $X$ is a definite cause of a target $Y$, then in ${\cal G}^*$ there must be a partially directed path from $X$ to $Y$ which contains a directed edge. On the other hand, for two distinct vertices lying in the same chain component, we have,

\begin{proposition}\label{prop:chain-comp}
	Two distinct  vertices $X$ and $Y$ are possible causes of each other if and only if they are in the same chain component.
\end{proposition}

Recall that in Figure \ref{fig1_2}, both ${\rm Smok}$ and ${\rm Lung}$ are definite causes of ${\rm Dysp}$. However, in the CPDAG there exists a directed path from ${\rm Lung}$ to ${\rm Dysp}$ while no directed path exists from ${\rm Smok}$ to ${\rm Dysp}$. That is, the cause ${\rm Lung}$ of  ${\rm Dysp}$ is explicit and the cause ${\rm Smok}$ of ${\rm Dysp}$ is implicit in the CPDAG.    This difference motivates the following two concepts.

\begin{definition}[Explicit Cause]\label{def:explicit}
	A variable $X$ is an explicit cause of $Y$ if 
	there is a common causal path from $X$ to $Y$ in every DAG in the Markov equivalence class represented by a CPDAG $\mathcal{G}^*$.
\end{definition}

As there is a common causal path from an explicit cause $X$ to the target $Y$ in every DAG in the Markov equivalence class represented by $\mathcal{G}^*$, there is a directed path from $X$ to $Y$ in $\mathcal{G}^*$, and thus $X$ is a definite cause of $Y$.


\begin{definition}[Implicit Cause]\label{def:implicit}
	A  variable $X$ is an implicit cause of $Y$ if $X$ is a definite cause of $Y$ and there is no common causal path from $X$ to $Y$ in all DAGs in the Markov equivalence class represented by a CPDAG $\mathcal{G}^*$.
\end{definition}

We notice that $X$ is a
definite cause of $Y$ if only if it satisfies one of the two conditions   given in Theorem \ref{thm:graphical_definite_cause}. The first condition, $\mathbf{C}\cap ch(X, \mathcal{G}^*)\neq \emptyset$, is the sufficient and necessary condition for identifying explicit causes, while the second condition corresponds to implicit causes. In Section \ref{sec:characterization}, we will exploit this difference between explicit and implicit causes to develop local characterizations for both of them. Below, we give an illustrative example.


\begin{figure}[!t]
	\centering
	\includegraphics[width=0.35\linewidth]{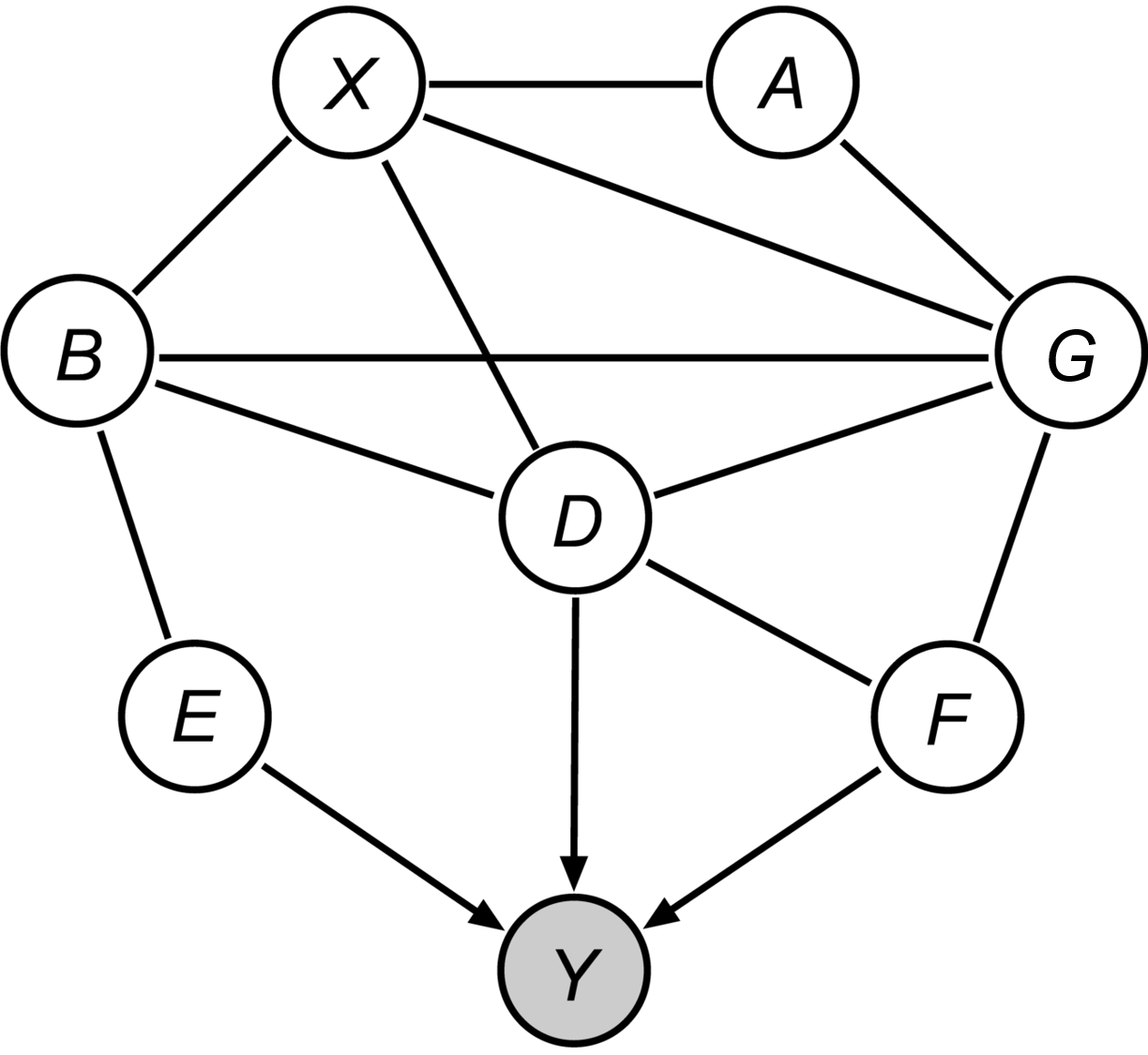}
	\caption{ An example for identifying the types of causal relations}
	\label{fig4}
\end{figure}

\begin{example}\label{e1}
	Consider the causes of the target variable $Y$ based on the CPDAG $\mathcal{G}^*$ in Figure \ref{fig4}. It is clear that all the variables other than $Y$ are definite or possible causes of $Y$. Obviously, $\{E, D, F\}$ are explicit causes of $Y$. For $B$, since $B-E\to Y$, $B-D\to Y$ and $B-G-F\to Y$ are chordless partially directed paths, the critical set of $B$ with respect to $Y$ is $\{E, D, G\}$. As the induced subgraph of $\mathcal{G}^*$ over $\{E, D, G\}$ is not complete, $B$ is a definite cause of $Y$, and $B$ is also implicit. Similarly, $G$ is another implicit cause of $Y$. For $X$ and $A$, the critical set of $X$ and $A$ with respect to $Y$ are $\{B, D, G\}$ and $\{X, G\}$, respectively. Since the corresponding induced subgraphs are complete, by Theorem \ref{thm:graphical_definite_cause}, $X$ and $A$ are not implicit causes of $Y$. Thus, they are possible causes of $Y$.
\end{example}

Despite the difference, explicit and implicit causes also have some interesting connections. The following Proposition \ref{prop:nec-implicit} proves that the existence of an implicit cause implies the existence of at least two explicit causes.

\begin{proposition}\label{prop:nec-implicit}
	  Let $\mathcal{G}^*$ be a CPDAG and $X$ and $Y$ be two  vertices of it in different chain components. If $X$ is the only explicit cause of $Y$ in the chain component to which $X$ belongs, then every vertex in this chain component, except $X$, is a possible cause of $Y$.
\end{proposition}

\section{Local Characterizations of Types of Causal Relations}\label{sec:characterization}

In this section, we introduce the theoretical results on locally characterizing different types of causal relations. Our local characterizations depend on the induced subgraph of the true CPDAG over the treatment's neighbors as well as some queries about d-separation relations. The first result is about definite non-causal relations, as given in Theorem~\ref{thm:non-cause-no-bk}.



\begin{theorem}\label{thm:non-cause-no-bk}
	Let $\mathcal{G}^*$ be a CPDAG.  For any two distinct  vertices $X$ and $Y$ in $\mathcal{G}^*$, $X$ is a definite non-cause of $Y$ if and only if $X \indep Y \mid pa(X, \mathcal{G}^*)$ holds.
\end{theorem}

Theorem \ref{thm:non-cause-no-bk} introduces a local characterization for definite non-causal relations, which is based on the local structure around the treatment $X$ and a single d-separation claim. The d-separation claim $X \indep Y \mid pa(X, \mathcal{G}^*)$ is similar to the following well-known result called \emph{local Markov property} of a causal DAG model: any variable is d-separated from its non-descendants given its parents.  The difference is that in our local characterization, only the parents of  $X$ in the CPDAG are included in the separation set, and we rule out  the siblings of $X$ even if they may be the parents of $X$ in the true causal DAG.  Since in a causal DAG, the non-descendants of a variable are those which are definitely not caused by the variable, Theorem \ref{thm:non-cause-no-bk} can be regarded as an extension of the local Markov property to CPDAGs.


Following Theorem \ref{thm:non-cause-no-bk}, we can distinguish definite and possible causes from definite non-causes with a local causal structure query and a d-separation query. Next, we characterize explicit and implicit causal relations locally in Theorem \ref{thm:explicit-cause-no-bk} and Theorem \ref{thm:implicit-cause-no-bk}, respectively, which together characterize definite causal relations.

\begin{theorem}\label{thm:explicit-cause-no-bk}
	Let $\mathcal{G}^*$ be a CPDAG.  For any two distinct  vertices $X$ and $Y$ in $\mathcal{G}^*$, $X$ is an explicit cause of $Y$ if and only if $X \nindep Y \mid pa(X, \mathcal{G}^*)\cup sib(X, \mathcal{G}^*)$ holds.
\end{theorem}

The local characterization in Theorem \ref{thm:explicit-cause-no-bk} includes a single  d-separation claim,   $X \nindep Y \mid pa(X, \mathcal{G}^*)\cup sib(X, \mathcal{G}^*)$, which means the set $pa(X, \mathcal{G}^*)\cup sib(X, \mathcal{G}^*)$ cannot block all paths from $X$ to $Y$. In the proof of this theorem, we show that this claim is equivalent to that there exists at least one path from $X$ to $Y$  in $\mathcal{G}^*$ on which the node adjacent to $X$ is a child of $X$. Based on \citet[Lemma~7.2]{maathuis2015generalized} and \citet[Lemma~B.1]{perkovic2017interpreting}, this  implies that there is a directed path from $X$ to $Y$  in $\mathcal{G}^*$.

We remark that the sufficiency of Theorem~\ref{thm:explicit-cause-no-bk} is related to the LWF local Markov property~\citep{Frydenberg1990}. Given a chain graph $\cal C$ over a vertex set $\mathbf{V}$ and a distribution $P$ over the same vertex set, $P$ is called LWF local Markovian (or local G-Markovian) to $\cal C$ if $X\indep_P \mathbf{V}\setminus (de(X, {\cal C})\cup pa(X, {\cal C})\cup sib(X, {\cal C})) \mid pa(X, {\cal C})\cup sib(X, {\cal C})$ for any $X\in \mathbf{V}$~\citep{Frydenberg1990}. Since \cite{andersson1997characterization} proved that a CPDAG ${\cal G}^*$ is a chain graph, if a distribution $P$ is LWF local Markovian to ${\cal G}^*$, then $X\nindep_P Y \mid pa(X, {\cal G}^*)\cup sib(X, {\cal G}^*)$ for $Y\notin pa(X, {\cal G}^*)\cup sib(X, {\cal G}^*)$ implies that $Y\in de(X, {\cal G}^*)$. That is, $X$ is an explicit cause of $Y$. In \ref{proof:explict}, we show the sketch of proving the sufficiency of Theorem~\ref{thm:explicit-cause-no-bk} based on the theories of chain graph models.

\begin{theorem}\label{thm:implicit-cause-no-bk}
 Suppose that $\mathcal{G}^*$ is a CPDAG and $\mathcal{M}$ is the set of maximal cliques of the induced subgraph of $\mathcal{G}^*$ over $sib(X, \mathcal{G}^*)$. Then, $X$ is an implicit cause of $Y$ if and only if $X \indep Y \mid pa(X, \mathcal{G}^*)\cup sib(X, \mathcal{G}^*)$ and $X \nindep Y \mid pa(X, \mathcal{G}^*)\cup \textbf{M}$ for any $\textbf{M}\in \mathcal{M}$.
\end{theorem}

The definition of maximal clique is given in \ref{app:graph}. In  Theorem \ref{thm:implicit-cause-no-bk}, the first condition $X \indep Y \mid pa(X, \mathcal{G}^*)\cup sib(X, \mathcal{G}^*)$ makes sure that $X$ is not an explicit cause of $Y$ and the second condition, which is $X \nindep Y \mid pa(X, \mathcal{G}^*)\cup \textbf{M}$ for any $\textbf{M}\in \mathcal{M}$, guarantees that $X$ is not a possible cause of $Y$.  These two  conditions   in Theorem \ref{thm:implicit-cause-no-bk} are  local in the sense that both $sib(X, \mathcal{G}^*)$ and $pa(X, \mathcal{G}^*)$ are subsets of $X$'s neighbors in $\mathcal{G}^*$, and a maximal clique $\textbf{M}$ is also a subset of $sib(X, \mathcal{G}^*)$. Once we obtain the induced subgraph of  $\mathcal{G}^*$ over $adj(X, \mathcal{G}^*)$, we can know $sib(X, \mathcal{G}^*)$ and $\mathcal{M}$, and thus the conditional independence queries can be answered accordingly if we have the oracles.



As mentioned in Section \ref{sec:sec:concepts}, definite causes include both explicit and implicit causes. Therefore, Theorems \ref{thm:explicit-cause-no-bk} and \ref{thm:implicit-cause-no-bk} give a sound and complete local characterization of definite causal relations as follows.

\begin{corollary}\label{cor:definite}
	Suppose that $\mathcal{G}^*$ is a CPDAG and $\mathcal{M}$ is the set of maximal cliques of the induced subgraph of $\mathcal{G}^*$ over $sib(X, \mathcal{G}^*)$. Then, $X$ is a definite cause of $Y$ if and only if $X \nindep Y \mid pa(X, \mathcal{G}^*)\cup sib(X, \mathcal{G}^*)$ or $X \nindep Y \mid pa(X, \mathcal{G}^*)\cup \textbf{M}$ for any $\textbf{M}\in \mathcal{M}$.
\end{corollary}

Together with Theorem \ref{thm:non-cause-no-bk}, Corollary \ref{cor:definite} can be used to identify definite causal relations and definite non-causal relations. This result is local in the sense that it only depends on the local structure around the treatment $X$ and a limited number of d-separation queries. When data is available in practice,  d-separation queries can be answered by performing statistical independence tests. Thus, local characterizations are particularly meaningful for identifying types of causal relations from observational data. 

\section{Algorithms}\label{sec:learn}

In this section, we discuss how to learn the types of causal relations from observational data. A local algorithm, which exploits the local characterizations in Section \ref{sec:characterization} directly, is provided in Section~\ref{sec:sec:local}. For the completeness of the paper, we also provide an efficient global learning method in Section~\ref{sec:sec:global}, and causal-effect-based methods in Section~\ref{sec:sec:ce}.



\subsection{A Local Learning Algorithm}\label{sec:sec:local}

\begin{algorithm}[!t]
	\caption{A local algorithm for identifying the type of causal relation (local ITC) }
	\label{algo:local}
	\begin{algorithmic}[1]
		\REQUIRE
		A treatment $X$, a target $Y$, $pa(X, {\cal G}^*)$, the induced subgraph of ${\cal G}^*$ over $sib(X, {\cal G}^*)$, and independence oracles.
		\ENSURE
		The type of causal relation between X and Y.
		\IF {$X \indep Y \mid pa(X, {\cal G}^*)$}
		\RETURN {$X$ is a \textbf{definite non-cause} of $Y$,}
		\ENDIF
		
		\IF {$X \nindep Y \mid pa(X, {\cal G}^*)\cup sib(X, {\cal G}^*)$}
		\RETURN {$X$ is an \textbf{explicit cause} of $Y$,}
		\ENDIF
		
		\STATE {$\mathcal{M}=$ the set of maximal cliques of $sib(X, {\cal G}^*)$,}
		\IF {exists $\textbf{M}\in \mathcal{M}$ such that $X \indep Y \mid pa(X, {\cal G}^*)\cup \textbf{M}$,}
		\RETURN {$X$ is a \textbf{possible cause} of $Y$,}
		\ENDIF
		\RETURN {$X$ is an \textbf{implicit cause} of $Y$.}
		
	\end{algorithmic}
\end{algorithm}


The main procedure of our local algorithm is summarized in Algorithm \ref{algo:local}.
The input of Algorithm \ref{algo:local} consists of $pa(X, {\cal G}^*)$, the induced subgraph of ${\cal G}^*$ over $sib(X, {\cal G}^*)$, and some independence oracles. The first two arguments, $pa(X, {\cal G}^*)$ and the induced subgraph over $sib(X, {\cal G}^*)$, can be learned locally by using the variant of the MB-by-MB algorithm proposed by \citet[Algorithm~3]{liu2020local}, which is designed for learning the chain component containing a given target variable and the directed edges connected to the variables in the chain component.
The third argument (the independence oracles), as discussed in Section \ref{sec:characterization}, can be replaced by statistical independence tests in practice. Overall, the procedure given in Algorithm \ref{algo:local} is a direct application of the local characterizations in Theorems \ref{thm:non-cause-no-bk}, \ref{thm:explicit-cause-no-bk} and \ref{thm:implicit-cause-no-bk}, and thus we have,





\begin{theorem}\label{thm:algo-local}
	Given a CPDAG ${\cal G}^*$ over $\mathbf{V}$ and the independence oracles faithful to a DAG in $[{\cal G}^*]$, the local ITC (Algorithm \ref{algo:local}) is sound and complete for identifying explicit causes, implicit causes, possible causes and definite non-causes of any variable $Y$ in ${\cal G}^*$.
\end{theorem}

Here, the soundness and completeness mean that the identified causes of each type are all and only those variables satisfying the definition of the corresponding type of cause. For example, the learned explicit causes of $Y$ are all and only the variables in $\mathbf{V}\setminus\{Y\}$ each of which has at least a common directed path to $Y$ in all equivalent DAGs.

The complexity of Algorithm \ref{algo:local} can be measured by the maximum number of conditional independence tests (or d-separation queries). Clearly, the maximum number of conditional independence tests performed by Algorithm \ref{algo:local} is $m+2$, where $m$ is the number of maximal cliques of  $sib(X, {\cal G}^*)$. Fortunately, there are only linearly many maximal cliques (with respect to the number of vertices) in a chordal graph \citep{Rose1975elimination, blair1993chordal}, so the number of conditional independence tests needed in Algorithm \ref{algo:local} is at most $O(|sib(X, {\cal G}^*)|)$.

\subsection{A Global Learning Algorithm}\label{sec:sec:global}

Given a CPDAG, identifying definite non-causal and explicit causal relations is straightforward. To discriminate implicit causal relations from  possible causal relations, we need an approach to find critical sets. The next proposition is particularly useful.

\begin{proposition}\label{prop:equivalent-critical-set}
	For any two distinct vertices $X, Y$ in a CPDAG $\mathcal{G}^*$ such that $X$ is not an explicit cause of $Y$, it holds that
	$\mathbf{C}_{XY}=\cup_{Z\in \mathbf{Z}}\mathbf{C}_{XZ}$,
	where $\mathbf{C}_{UV}$ denotes the critical set of $U$ with respect to $V$, and $\mathbf{Z}$ is the set of ancestors of $Y$ in $\mathcal{G}^*$ which are also in the chain component containing $X$.
\end{proposition}

\begin{algorithm}[!t]
	\caption{Finding the critical set of a given $X$ with respect to a set $\textbf{Z}$}
	\label{algo:critical}
	\begin{algorithmic}[1]
		\REQUIRE
		A chordal graph $\mathcal{G}^*_u$, a variable $X$ in $\mathcal{G}^*_u$, and a variable set $\textbf{Z}\neq \emptyset$ such that $X\notin \textbf{Z}$.
		\ENSURE
		$\textbf{C}$, which is the critical set of $X$ with respect to $\textbf{Z}$ in $\mathcal{G}^*_u$.
		\STATE {Initialize $\textbf{C}=\emptyset$, a waiting queue $\mathcal{S}=[\;]$,   and a set $\mathcal{H}=\emptyset$,}
		\FOR {$\alpha \in adj(X)$}
		\STATE {add $(\alpha, X, \alpha)$ to the end of $\cal S$,}
		\ENDFOR
		\WHILE {$\mathcal{S}$ is not empty}
		\STATE {take the first element $( \alpha, \psi,\tau)$ out of $\cal S$ and add it to $\mathcal{H}$,}

		\IF {$\tau\in \textbf{Z}$}
		\STATE {add $\alpha$ to $\textbf{C}$, remove from $\mathcal{S}$  all triples where the first element is $\alpha$,}
		
		\ELSE
		\FOR {$\beta \in adj(\tau)$ and $\beta\notin adj(\psi)\cup\{\psi\}$ }
		\IF {  $(\alpha, \tau,\beta)  \notin    \mathcal{H}$ and $ ( \alpha, \tau,\beta)  \notin    \mathcal{S}$}
		\STATE {add $( \alpha, \tau,\beta)$ to the end of $\cal S$,}
		\ENDIF
		\ENDFOR	
		\ENDIF
		\ENDWHILE
		\RETURN {$\textbf{C}$}
	\end{algorithmic}
\end{algorithm}

Proposition \ref{prop:equivalent-critical-set} provides a factorization of the critical set of $X$ with respect to $Y$. For simplicity, we call $\cup_{Z\in \mathbf{Z}}\mathbf{C}_{XZ}$ the critical set of $X$ with respect to $\mathbf{Z}$. Algorithm \ref{algo:critical} shows how to find $\cup_{Z\in \mathbf{Z}}\mathbf{C}_{XZ}$ efficiently. Algorithm \ref{algo:critical} runs a breadth-first-search and returns the critical set of $X$ with respect to $\textbf{Z}$ in $\mathcal{G}^*_u$.
In Algorithm \ref{algo:critical}, we start from the siblings of $X$,  then search chordless paths from the siblings   until reaching some $Z_i\in \textbf{Z}$. Every  chordless path starting from a sibling of $X$ is recorded in a  queue $\cal S$ as a  triple like   $(\alpha, \psi, \tau)$, where $\alpha$ and $\tau$ are the start and the end points of the  path, respectively, and $\psi$ is the sibling of  $\tau$ on the path.  If $\tau$ is a member of $\textbf{Z}$, we add $\alpha$ to the critical set $\textbf{C}$ and remove from $\mathcal{S}$ all triples where the first element is $\alpha$, that is, we stop enumerating chordless paths starting with $\alpha$. Otherwise, we extend the chordless path to  the siblings of $\tau$ that are neither  $\psi$ nor siblings of $\psi$ and add the corresponding triples to the queue $\cal S$.   In this algorithm, a set of visited triples, $\cal H$, is introduced to speed up the search by avoiding visiting the same triple twice.

\begin{algorithm}[!t]
	\caption{A global algorithm for identifying the type of causal relation (global ITC). }
	\label{algo:global}
	\begin{algorithmic}[1]
		\REQUIRE
		A CPDAG $\mathcal{G}^*$, a variable $X$ and a target $Y$ in $\mathcal{G}^*$.
		\ENSURE
		The type of causal relation between X and Y.
		
		\IF {$X$ and $Y$ are connected by a path in $\mathcal{G}^*_u$}
		\RETURN {$X$ is a \textbf{possible cause} of $Y$,}
		\ENDIF
		
		\STATE {let   $\textbf{Z}=an(Y, \mathcal{G}^*)$,}
		\IF {$X\in \textbf{Z}$}
		\RETURN {$X$ is an \textbf{explicit cause}  of $Y$,}
		\ENDIF
		
		\STATE {use Algorithm \ref{algo:critical} to find the critical set $\textbf{C}$ of $X$ with respect to $\textbf{Z}$ in $\mathcal{G}^*_u$,}
		\IF {$|\textbf{C}|=0$}
		\RETURN {$X$ is a \textbf{definite non-cause} of $Y$,}
		\ENDIF
		
		\IF {$\textbf{C}$ induces a complete subgraph of $\mathcal{G}^*_u$}
		\RETURN {$X$ is a \textbf{possible cause} of $Y$,}
		\ENDIF
		
		\RETURN {$X$ is an \textbf{implicit cause} of $Y$.}

	\end{algorithmic}
\end{algorithm}

Finally, we present a global learning approach for identifying types of causal relations in Algorithm \ref{algo:global}. Algorithm \ref{algo:global} is global in the sense that it takes an entire CPDAG as input. In Algorithm \ref{algo:global}, we first check whether $X$ and $Y$ are in the same chain component. If they are, $X$ is a possible cause of $Y$ based on Proposition \ref{prop:chain-comp}. Otherwise, we find the set of explicit causes of $Y$ and denote it by $\textbf{Z}$. This can be done by searching for the vertices that are connected to $Y$ in the directed subgraph of $\mathcal{G}^*$. If $X \in \textbf{Z}$, $X$ is an explicit cause of $Y$, otherwise, we  find the critical set $\textbf{C}$ of $X$ with respect to $\textbf{Z}$.   When $\textbf{C}=\emptyset$, we have that  there are no explicit causes of $Y$ in the chain component containing $X$, so $X$ is not a cause of $Y$.  Finally, using Theorem \ref{thm:implicit-cause-no-bk}, Algorithm \ref{algo:global} distinguishes between possible causes and implicit causes.



Since Algorithm \ref{algo:critical} does not visit the same triple like $(\alpha, \psi,\tau)$ twice, where $\alpha$ is a sibling of $X$ and $\tau$ is a sibling of $\psi$ in $\mathcal{G}^*_u$,   the complexity of Algorithm \ref{algo:critical} in the worst case  is $O(|sib(X, \mathcal{G}^*)| \cdot  |\mathbf{E}(\mathcal{G}^*_u)|)$, where $|\mathbf{E}(\mathcal{G}^*_u)|$ is the number of edges in $\mathcal{G}^*_u$. Now we consider the computational complexity of global ITC (Algorithm \ref{algo:global}). We know that the complexity   to  check  the undirected connectivity of $X$ and $Y$ or to find  $an(Y, \mathcal{G}^*)$  is $O(|\mathbf{E}(\mathcal{G}^*)|^2)$,  where $|\mathbf{E}(\mathcal{G}^*)|$ is the number of vertices in ${\cal G}^*$.  Consequently, the complexity of  global ITC is $O(|\mathbf{E}(\mathcal{G}^*)|^2+|sib(X, \mathcal{G}^*)| \cdot  |\mathbf{E}(\mathcal{G}^*_u)|)$. Clearly, the worst case is $O(|\mathbf{E}(\mathcal{G}^*)|^3)$.

\subsection{Causal-Effect-Based Methods}\label{sec:sec:ce}

\begin{algorithm}[!t]
	\caption{The framework of   causal effect testing based algorithms.}
	\label{algo:ce-based}
	\begin{algorithmic}[1]
		\REQUIRE
		A treatment $X$, a target $Y$, a CPDAG $\mathcal{G}^*$ over a vertex set $\bf V$, and a significance level $\alpha$.
		\ENSURE
		The type of causal relation between X and Y.
		
		\STATE{set $\Theta= [\;]$ and  ${\rm P}_{\rm val}= [\;]$,}
		\FOR{each $\mathbf{S}\subset \mathbf{V}$ such that $\mathbf{S}$ is an adjustment set for $(X, Y)$ in some DAG in $[\mathcal{G}^*]$}
		\STATE{estimate the causal effect $\theta$ of $X$ on $Y$ by adjusting for $\mathbf{S}$, and add the causal effect to $\Theta$,}
		\STATE{test   the null hypothesis   $\theta=0$  and   add the corresponding p-value to ${\rm P}_{\rm val}$,}
		\ENDFOR
		
		
		\IF {every p-value in ${\rm P}_{\rm val}$ is less than or equal to $\alpha$}
		\RETURN {$X$ is a \textbf{definite cause} of $Y$,}
		\ENDIF
		
		\IF {every p-value in ${\rm P}_{\rm val}$ is greater than $\alpha$}
		\RETURN {$X$ is a \textbf{definite non-cause}  of $Y$,}
		\ENDIF
		
		\RETURN {$X$ is a \textbf{possible cause} of $Y$.}

	\end{algorithmic}
\end{algorithm}

We now discuss the causal-effect-based methods, which are modifications of the IDA-type algorithms.   For simplicity, we assume that the observed variables follow a linear-Gaussian structural equation model and that the observational distribution is faithful to the underlying DAG. With these assumptions, $X$ has a non-zero total causal effect on $Y$ if and only if there is a  directed path from $X$ to $Y$ in the underlying DAG.  Following the work of~\citet{maathuis2009estimating}, we use
\[ACE(Y\mid do(X=x)) \coloneqq \frac{\partial E(Y\mid do(X=x))}{\partial x},\]
to measure the (average) total causal effect of $X$ on $Y$.\footnote{Here, $do(X=x)$ is the \emph{do-operator} proposed by \citet{pearl2009causality} to denote the intervention on $X$ by forcing $X$ to be $x$. \citet{pearl2009causality} defined that $X$ has a causal effect on $Y$ if there exists an $x\neq x'$ such that $P(Y\mid do(X=x))\neq P(Y\mid do(X=x'))$, where $P(Y\mid do(X=x))$ is the post-intervention distribution of $Y$. On the other hand, it is common to summarize $P(Y\mid do(X=x))$ by its mean \citep{pearl2009causality, maathuis2009estimating}, i.e., the mean of $Y$ w.r.t. $P(Y\mid do(X=x))$, which is denoted by $E(Y\mid do(X=x))$.} As mentioned in the introduction, the idea of the causal-effect-based methods is to estimate all possible causal effects of the treatment on the target first, and then check whether the possible effects are all zeros or non-zeros. If all of the possible effects are evaluated as zeros (non-zeros), then the treatment is a definite non-cause (definite cause) of the target.

However, due to estimation error, an estimated effect may not be exactly zero. In the work of~\citet{maathuis2010nature}, the authors first estimated all possible effects for all pairs of treatment and target, and then summarized each set of possible causal effects by its minimum absolute value. Finally, the minimum values were sorted in descending order. The top ones were evaluated as relatively strong effects. Though this method has been widely applied to real-world problems, it requires to  estimate all possible effects for all pairs of variables, which may bring unnecessary costs if someone only wants to know the causal relation between one pair of treatment and target. Moreover, in this method, the order of an effect depends on the other effects. Thus, this method is suitable for comparing the magnitude of causal effects, rather than  identifying the causal relation of a given pair.

In this paper, we focus on a testing-based solution, whose framework is summarized by Algorithm \ref{algo:ce-based}. After initializing two sequences $\Theta$ and ${\rm P}_{\rm val}$, Algorithm~\ref{algo:ce-based} enumerates all possible causal effects of $X$ on $Y$ and tests the null hypothesis  $\theta=0$ for each estimated effect $\theta$.  Different modifications adopt different enumeration and testing strategies. We introduce four modifications below.

\begin{enumerate}
	\item [(M1)] IDA {\rm +} testing all enumerated effects. Following the original IDA framework~\citep{maathuis2009estimating}, this modification enumerates all possible effects by listing all possible parental sets of the treatment $X$. Thus, line 2 of Algorithm~\ref{algo:ce-based} is replaced by
	
	{\vspace{2mm} \it ``for each $\mathbf{Q}\subset sib(X, {\cal G}^*)$ such that orienting $\mathbf{Q}\to X$ and $X\rightarrow sib(X, \mathcal{G}^{*})\setminus \textbf{Q}$ does not introduce any v-structure collided on $X$, let $\mathbf{S}=\mathbf{Q}\cup pa(X, {\cal G}^*)$ and do ..." \vspace{2mm}}
	
    All enumerated effects are then tested according to line 4 of Algorithm~\ref{algo:ce-based}. In the linear-Gaussian case, estimating the causal effect of $X$ on $Y$ by adjusting for $\mathbf{S}$ is equivalent to estimating the coefficient of $X$ in the linear regression of $Y$ on $X$ and $\mathbf{S}$. Hence, a t-test for the coefficient of $X$ is used to test  the significance of the causal effect of $X$ on $Y$.

    \item [(M2)] IDA {\rm +} testing the minimum and maximum absolute enumerated effects. This modification is inspired by the work of~\citet{maathuis2010nature}. It first enumerates all possible effects by listing all possible parental sets of the treatment $X$. Then, it tests the  effects of $X$ on $Y$  with the minimum and maximum absolute values to obtain two p-values, $p_{\rm min}$ and $p_{\rm max}$,  respectively. Consequently,   if $p_{\rm min} \leq \alpha$, it returns that X is a definite cause of Y, if $p_{\rm max} > \alpha$, it returns that X is a definite non-cause of Y, and otherwise returns that X is a possible cause of Y.  The details of this modification is provided in \ref{app:ce}.

    \item [(M3)] IDA {\rm +} utilizing non-ancestral relations {\rm +} testing all enumerated effects. This modification first lists all possible parental sets of the treatment $X$. Then, to reduce the number of estimations and significance tests, it checks whether $X$ is a non-ancestor of $Y$ before estimating the causal effects of $X$ on $Y$. More formally, we insert a step   between  lines 2 and 3 of Algorithm~\ref{algo:ce-based} as follows.

    {\vspace{2mm} \it ``Orient $\mathbf{Q}\to X$ and $X\rightarrow sib(X, \mathcal{G}^{*})\setminus \textbf{Q}$, complete the orientations using Meek's rules~\citep{meek1995causal}, and use Lemma 3.2 in~\cite{perkovic2017interpreting} to check whether $X$ is a b-possible ancestor of $Y$." \vspace{2mm}}

    The definition of b-possible ancestor can be found in~\citet[Definition~3.3]{perkovic2017interpreting}. We directly set $\theta=0$ and the p-value $p=1$ if  $X$ is not a b-possible ancestor of $Y$. Otherwise, we  estimate  the causal effect of $X$ on $Y$ by adjusting for $\mathbf{S}=\mathbf{Q}\cup pa(X, {\cal G}^*)$, and test   the significance of the estimated effect. The details are provided in~\ref{app:ce}.

    \item [(M4)] IDA {\rm +} utilizing non-ancestral relations {\rm +} testing the minimum and maximum absolute enumerated effects. This modification takes the same enumeration strategy used by the third modification, and uses the same testing strategy as the second method does.
\end{enumerate}

Except for the second and forth modifications, the other two modifications usually compute a list of p-values. These p-values are compared with a given significance level $\alpha$, as shown in lines $6$-$12$ of Algorithm~\ref{algo:ce-based}. The adjustment methods for multiple p-values, such as the Bonferroni correction, may be used to control the false discovery rate.


The input CPDAG of Algorithm \ref{algo:ce-based} can be replaced by the induced subgraph over $pa(X, {\cal G}^*) \cup sib(X, {\cal G}^*)$ for the first two modifications.
Since the induced subgraph over $pa(X, {\cal G}^*) \cup sib(X, {\cal G}^*)$ can be learned locally using the variant of MB-by-MB \citep[Algorithm~3]{liu2020local}, we can combine the first two modifications with the variant of MB-by-MB to make them fully local. However, for the last two modifications, the input CPDAG cannot be replaced. The reason is that these two  modifications need to run Meek's rules to extend the local orientations $\mathbf{Q}\to X$ and $X\rightarrow sib(X, \mathcal{G}^{*})\setminus \textbf{Q}$, and Meek's rules require an entire CPDAG.

We remark that, the causal-effect-based methods are not restricted to the aforementioned four modifications. For example, in the linear-Gaussian case, one can use the optimal IDA~\citep{Witte2020efficient} to replace the original IDA. For each $\mathbf{Q}\subset sib(X, {\cal G}^*)$ such that orienting $\mathbf{Q}\to X$ and $X\rightarrow sib(X, \mathcal{G}^{*})\setminus \textbf{Q}$ does not introduce any v-structure collided on $X$, the optimal IDA first runs Meek's rules to extend the local orientations $\mathbf{Q}\to X$ and $X\rightarrow sib(X, \mathcal{G}^{*})\setminus \textbf{Q}$, and then finds the optimal adjustment set so that the estimation of the causal effect has the smallest asymptotic variance.    Another modification, which we call the  hybrid method, is to  infer whether $X$ is a definite non-cause of $Y$  by checking whether $X$ has a partially directed path to $Y$ in the input CPDAG~\citep{zhang2006, perkovic2017interpreting} first, and then call a causal-effect-based method if $X$ is not a definite non-cause of $Y$.  Compared with the third and forth modifications, this hybrid method utilizes non-ancestral relations  before  listing all  possible parental sets of $X$, and therefore, is
generally more efficient if $X$ is graphically identified as a definite non-cause of $Y$.

To end this section, we theoretically compare the proposed local algorithm to the causal-effect-based methods in terms of computational complexity. In the worst case, the number of causal effect estimations  required by a causal-effect-based method is $2^{|sib(X, {\cal G}^*)|}$, since every causal-effect-based method needs to enumerate the possible causal effects of $X$ on $Y$. At the same time, at most  $2^{|sib(X, {\cal G}^*)|}$ tests on these estimated causal effects  are required  in the worst case. 
Besides, the modifications that utilize non-ancestral relations or the optimal IDA have to run $2^{|sib(X, {\cal G}^*)|}$ times Meek's rules, while the complexity of Meek's rules is polynomial to the number of vertices in the graph.   Therefore, the proposed local method (Algorithm \ref{algo:local}) is  more efficient than the current causal-effect-based methods as the former only needs linearly many hypothesis tests.

\section{ Experiments}\label{sec:simulation}

In this section, we  illustrate and evaluate the proposed methods experimentally  using synthetic data sets generated from linear structural equation models with Erd\"os-R\'enyi random DAGs and the DREAM4 data sets. We compare the local ITC with the global one as well as the four modifications of   causal-effect-based  methods (CE-based for short).

The details of the CE-based  methods are provided in Section~\ref{sec:sec:ce} and \ref{app:ce}. In this section, the four modifications of CE-based methods from M1 to M4 are denoted by  ``IDA + test (all)", ``IDA + test (min/max)", ``IDA + an + test (all)", and `IDA + an + test (min/max)", respectively. We use the Bonferroni correction to adjust p-values for multiple comparisons, and the corresponding methods are denoted by ``multi". In addition, the CE-based methods with the optimal IDA and the hybrid method mentioned in Section~\ref{sec:sec:ce} are also studied experimentally in~\ref{app:app:opt} and \ref{app:app:hybrid} respectively.


In Section \ref{sec:sec:true}, we assume that the true CPDAG or its local structure of interest is available. In this case, the synthetic data sets are only used by the CE-based methods to estimate causal effects, and by the local ITC method to perform conditional independence tests. 
In Section \ref{sec:sec:est}, we further evaluate  the methods based on the structures learned from data.  Three global structure learning algorithms, including the PC algorithm~\citep{spirtes1991algorithm}, the stable PC algorithm~\citep{colombo2014order} and the GES algorithm~\citep{chickering2002learning}, are used to learn CPDAGs, and the variant of MB-by-MB~\citep{liu2020local} is used to learn parents and siblings of the vertices of interest. In all of these experiments,  algorithms like PC, stable PC, GES and IDA are called from \texttt{R-package pcalg}~\citep{Kalisch2012pcalg}, and the Bonferroni correction is called from \texttt{R-package stats}. The significance level $\alpha$ of statistical independence tests  is  $0.001$.\footnote{Experiments show different  significance levels  give   similar results.} All codes were run on a computer with an Intel 2.5GHz CPU and 8 GB of memory.

Let $\text{ER}(n,d)$  denote a random DAG with $n$ vertices and average in-and-out degree $d$. In our experiments, $n$ is chosen from $\{50, 100\}$ and $d$ is chosen from $\{1.5, 2.0, 2.5, 3.0, 3.5, 4.0\}$. For a sampled $\text{ER}(n,d)$ graph, we drew an edge weight $\beta_{ij}$ from a Uniform($[0.8, 1.6]$) or a Uniform($[-1.6, -0.8]\cup[0.8, 1.6]$) distribution for each directed edge $X_i\rightarrow X_j$ in the DAG. Then, we  constructed a  linear structural equation model as follows,
\begin{align}
X_j = \sum_{X_i \in pa(X_j)}\beta_{ij}X_i+\epsilon_j\;, \quad j = 1,...,n\;,
\end{align}
where $\epsilon_1, ..., \epsilon_n$ are independent $\mathcal{N}(0,1)$ noises.

For each combination of $n$, $d$ and the distribution of edge weights, we generated $5,000$ weighted DAGs. Finally, in Section~\ref{sec:sec:true} we drew $N_{\rm effect}\in \{50, 100, 150\}$ samples from this linear model to estimate causal effects and perform conditional independence tests, and in Section~\ref{sec:sec:est} we drew additional $N_{\rm graph}\in \{100, 200, 500\}$ samples to learn the required causal structures. In summary, there were totally $2\times 6 \times 2=24$ graph settings and  $2\times 6\times 2\times 3\times 3=216$ experiment parameter settings, and for each experiment parameter setting, we repeated the experiment  $5,000$ times.

 Given a sampled DAG,  we randomly drew a treatment variable and a target variable, and  compared  their  causal relation (definite cause, definite non-cause, or possible cause) learned from data with the true one read from the corresponding CPDAG of the sampled DAG. The Kappa coefficient~\citep{Jacob1960kappa} as well as the true positive rate (TPR) and the false positive rate (FPR) were used to measure the performance of each method.   The Kappa coefficient is an adjustment of accuracy rate. It is the proportion of agreement after chance agreement is removed from consideration~\citep{Jacob1960kappa}.  The formal definition of the Kappa coefficient is given as follows. For an experiment parameter setting, let $M_{ij}$ be the number of experiments in which the $i$th causal relation  is identified as the $j$th  causal relation.  The Kappa coefficient $\kappa$ is defined as
\[\kappa = \frac{p - q}{1-q},\]
where
\[p = \frac{\sum_{i=1}^{3} M_{ii}}{\sum_{i,j=1}^{3} M_{ij}}, \qquad q = \frac{\sum_{i=1}^{3} (\sum_{j=1}^3 M_{ij})\cdot(\sum_{j=1}^3 M_{ji})     }{(\sum_{i,j=1}^{3} M_{ij})^2}.\]
The Kappa coefficient ranges from $-1$ to $+1$, and the higher the value of Kappa, the better the evaluated method. For ease of visualization, we focus on the kappa coefficients in this section, and the TPRs and FPRs are reported in~\ref{app:app:detailed}.


\subsection{Learning with True Graphs}\label{sec:sec:true}

\begin{figure}[t!]
	\centering
	\subfigure{
		\begin{minipage}[t]{0.9\textwidth}
			\centering
			\includegraphics[width=\textwidth]{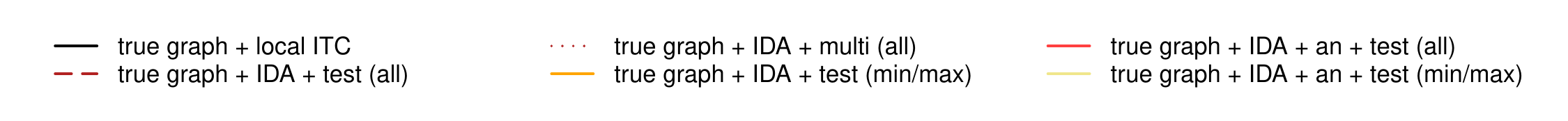}
		\end{minipage}%
	}%
	\vspace{-1em}
	\addtocounter{subfigure}{-1}
	
	\subfigure[$n=50$, $N_{\rm effect}=50$ \label{fig:true:kappa_50_50}]{
		\begin{minipage}[t]{0.3\textwidth}
			\centering
			\includegraphics[width=\textwidth]{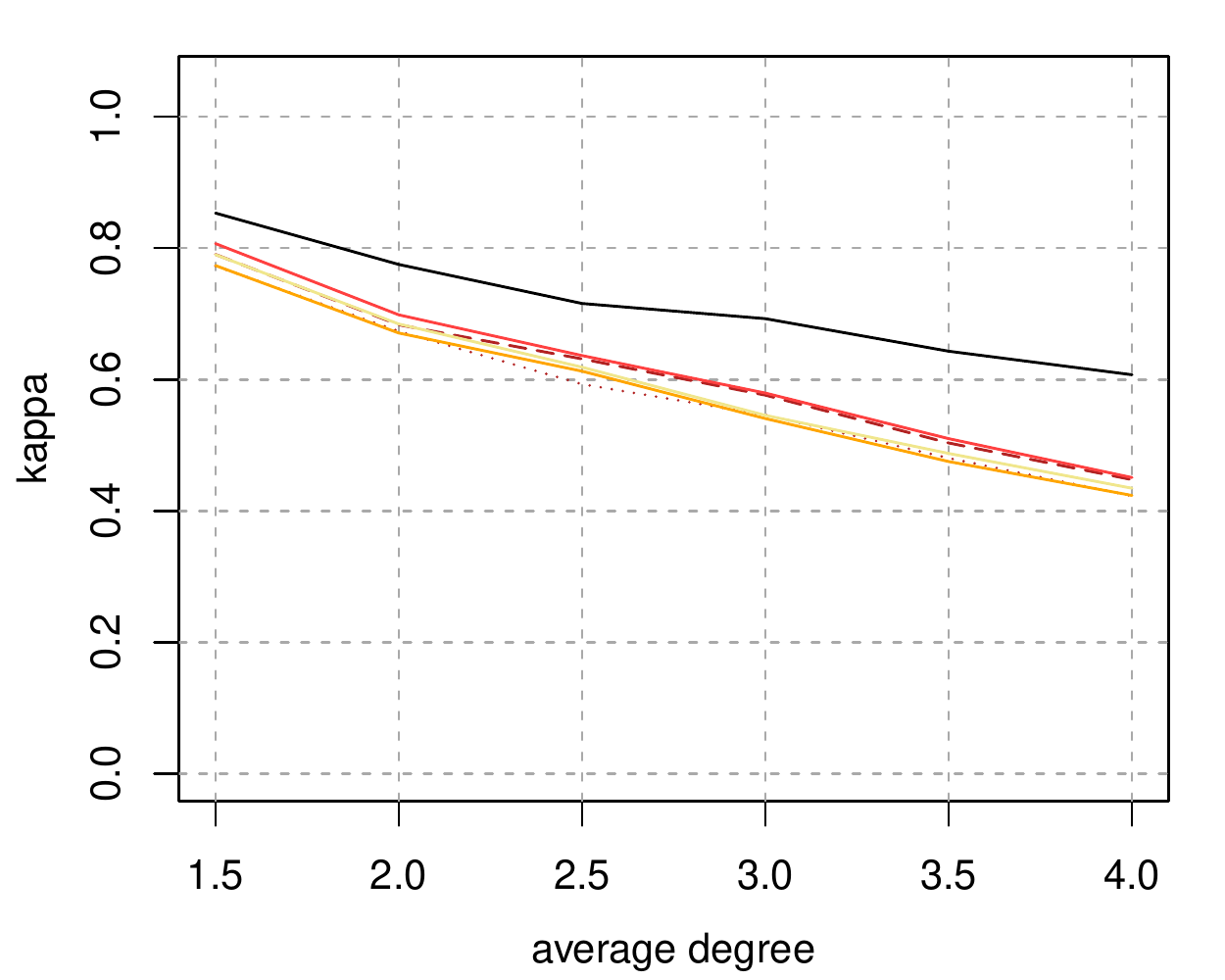}
		\end{minipage}%
	}%
	\hspace{0.01\textwidth}
	\subfigure[$n=50$, $N_{\rm effect}=100$  \label{fig:true:kappa_50_100}]{
		\begin{minipage}[t]{0.3\textwidth}
			\centering
			\includegraphics[width=\textwidth]{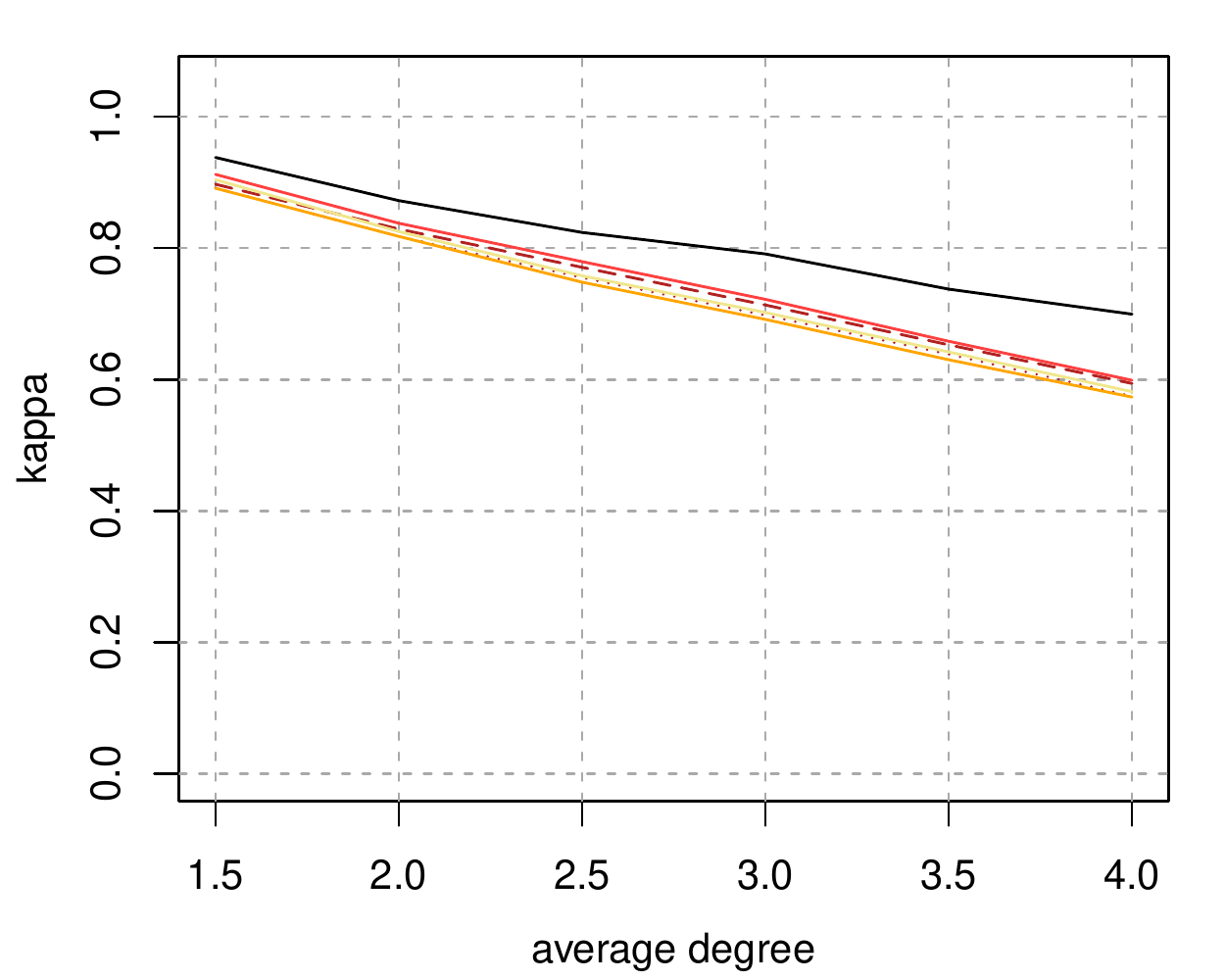}
		\end{minipage}%
	}%
	\hspace{0.01\textwidth}
	\subfigure[$n=50$, $N_{\rm effect}=150$ \label{fig:true:kappa_50_150}]{
		\begin{minipage}[t]{0.3\textwidth}
			\centering
			\includegraphics[width=\textwidth]{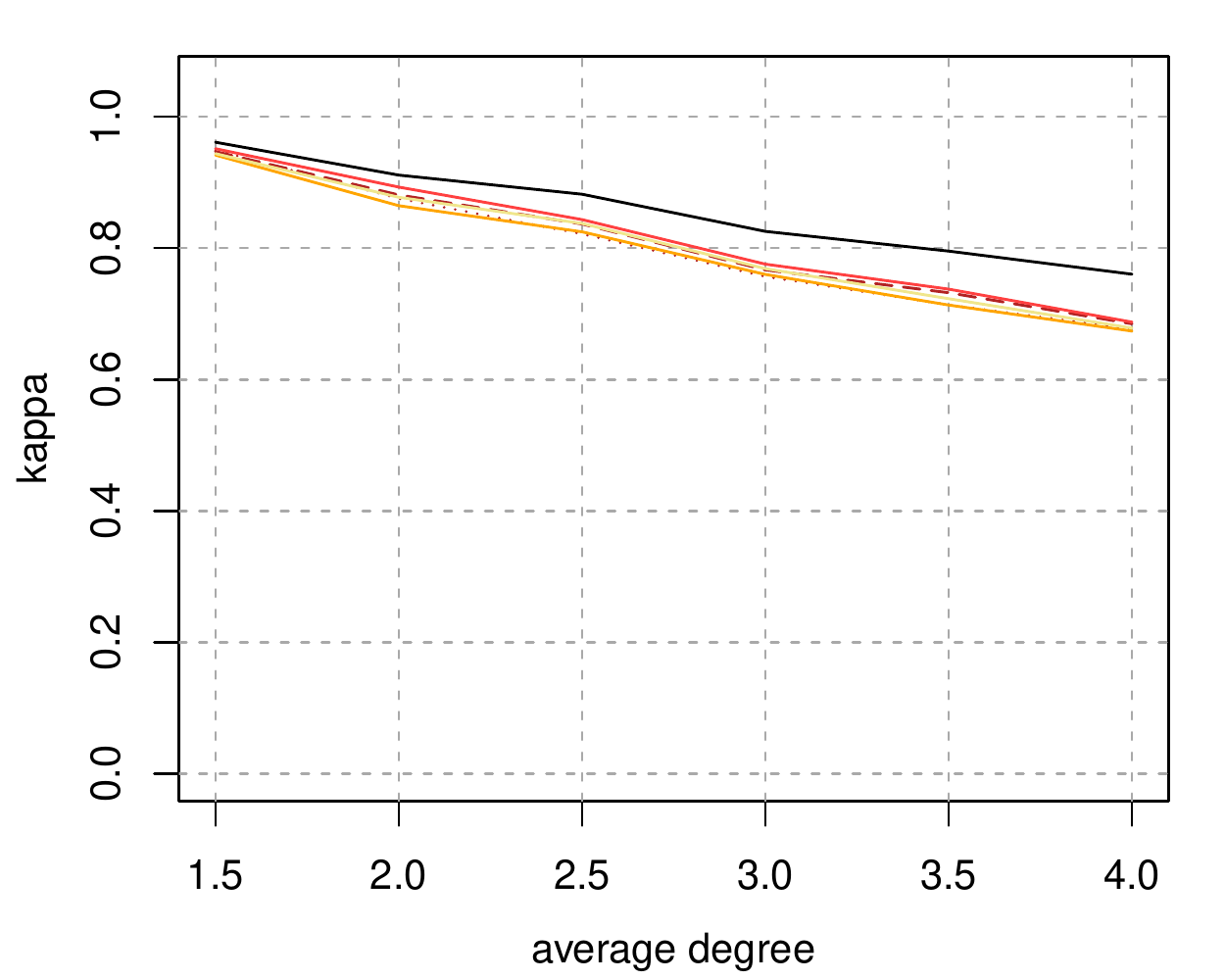}
		\end{minipage}%
	}%
	
	\subfigure[$n=100$, $N_{\rm effect}=50$ \label{fig:true:kappa_100_50}]{
		\begin{minipage}[t]{0.3\textwidth}
			\centering
			\includegraphics[width=\textwidth]{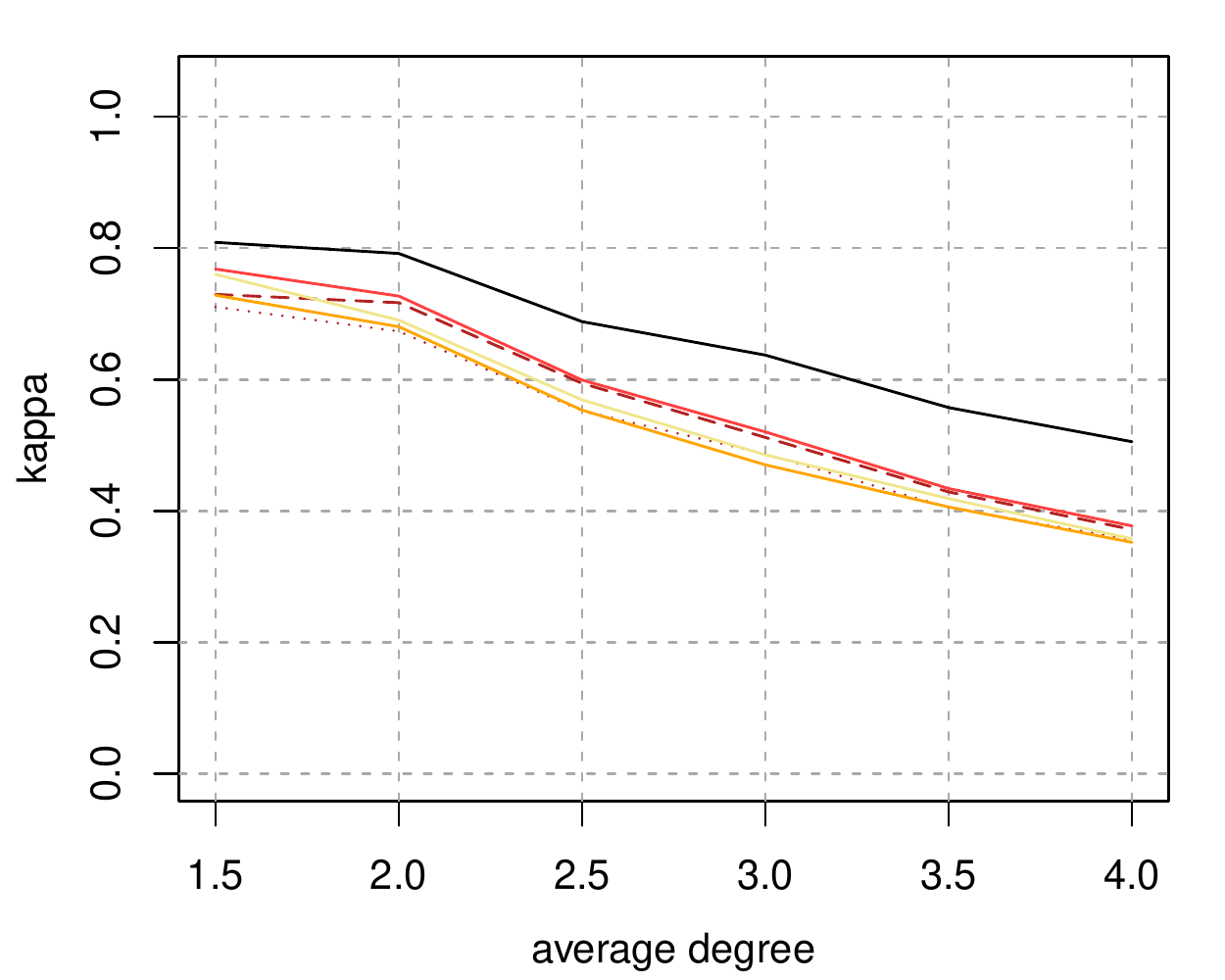}
		\end{minipage}%
	}%
	\hspace{0.01\textwidth}
	\subfigure[$n=100$, $N_{\rm effect}=100$ \label{fig:true:kappa_100_100}]{
		\begin{minipage}[t]{0.3\textwidth}
			\centering
			\includegraphics[width=\textwidth]{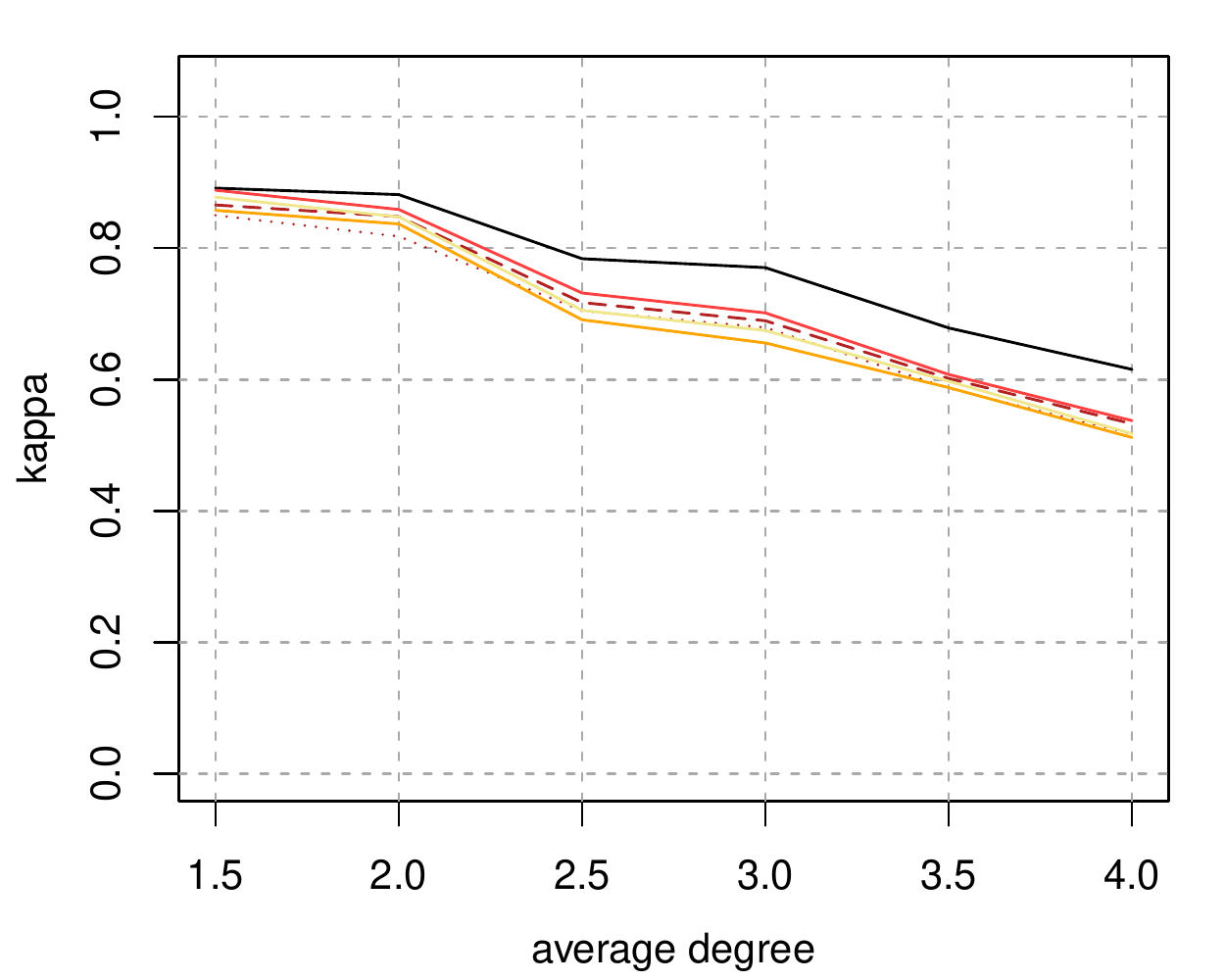}
		\end{minipage}%
	}%
	\hspace{0.01\textwidth}
	\subfigure[$n=100$, $N_{\rm effect}=150$ \label{fig:true:kappa_100_150}]{
		\begin{minipage}[t]{0.3\textwidth}
			\centering
			\includegraphics[width=\textwidth]{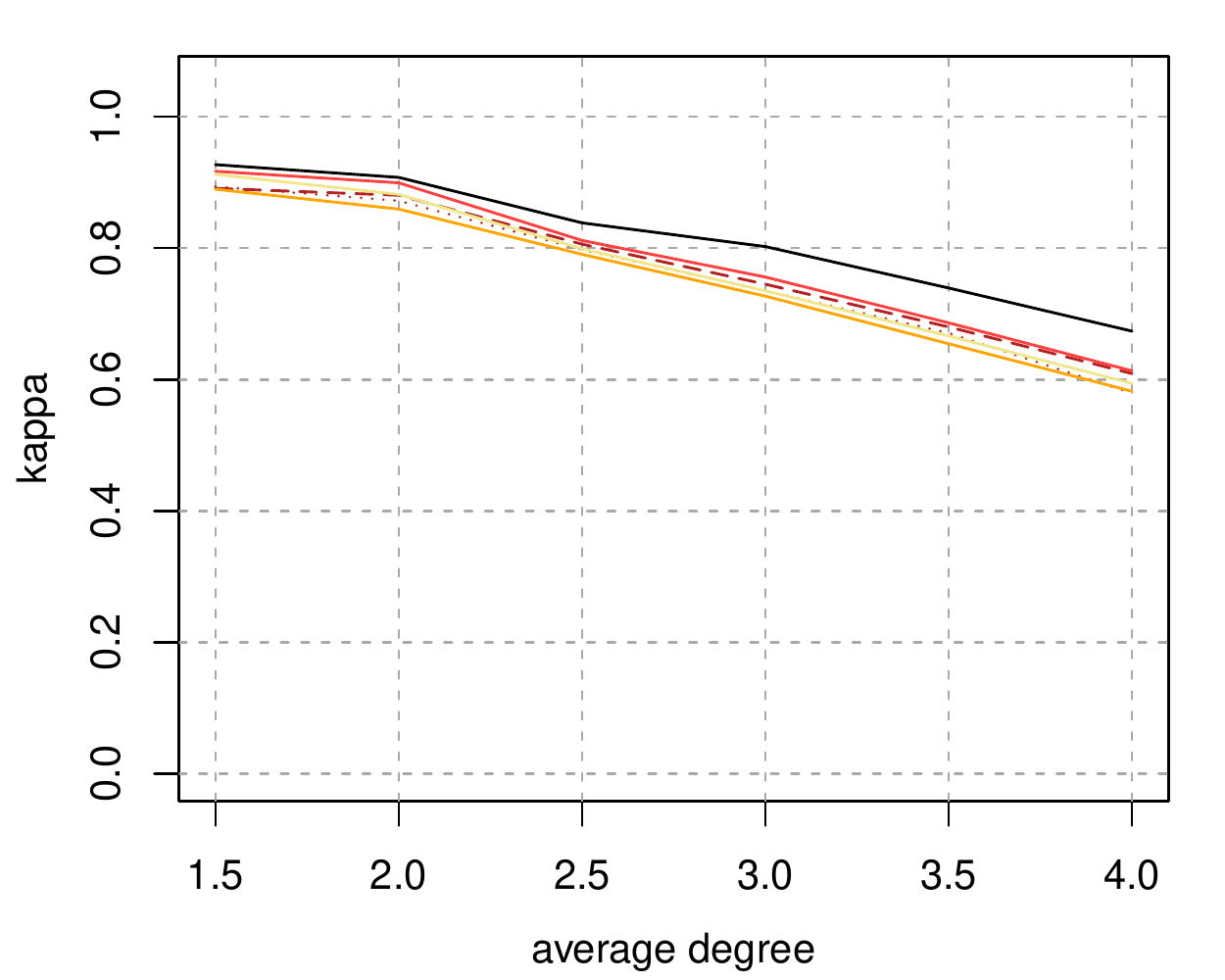}
		\end{minipage}%
	}%
	
	\caption{The Kappa coefficients of different methods on random graphs with positive weights. The true graph structures are provided. }
	\label{fig:true:kappa}
\end{figure}

\begin{figure}[t!]
	\centering
	\subfigure{
		\begin{minipage}[t]{1\textwidth}
			\centering
			\includegraphics[width=\textwidth]{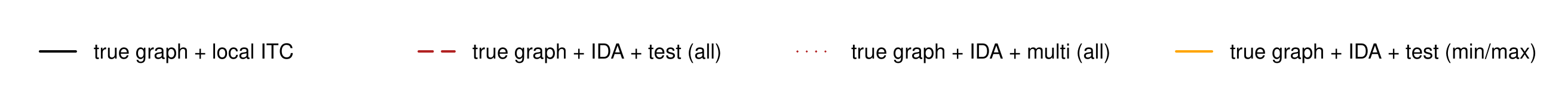}
		\end{minipage}%
	}%
	\vspace{-1em}
	\addtocounter{subfigure}{-1}
	
	\subfigure[$n=50$, $N_{\rm effect}=50$ \label{fig:true:time_50_50}]{
		\begin{minipage}[t]{0.3\textwidth}
			\centering
			\includegraphics[width=\textwidth]{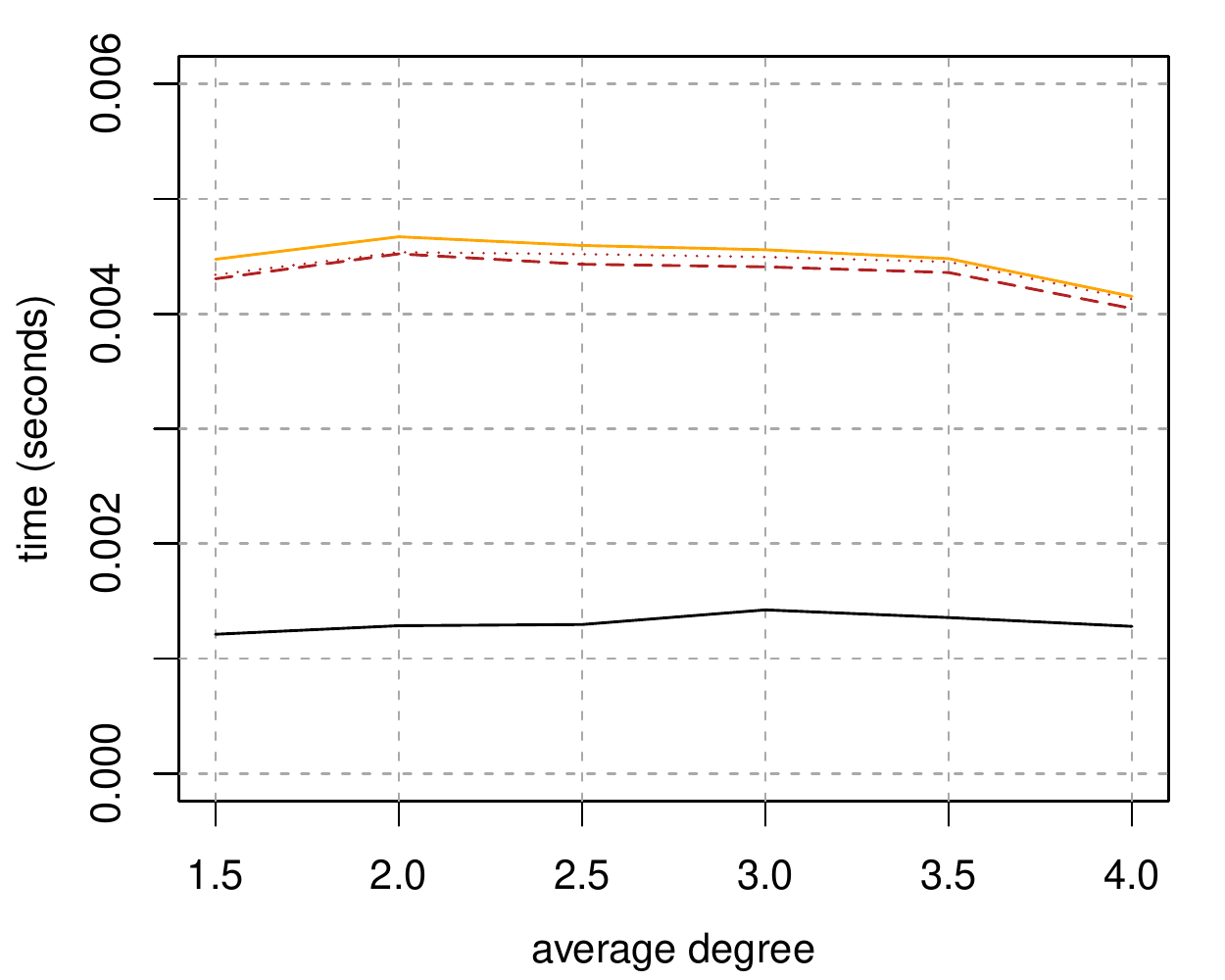}
		\end{minipage}%
	}%
	\hspace{0.01\textwidth}
	\subfigure[$n=50$, $N_{\rm effect}=100$  \label{fig:true:time_50_100}]{
		\begin{minipage}[t]{0.3\textwidth}
			\centering
			\includegraphics[width=\textwidth]{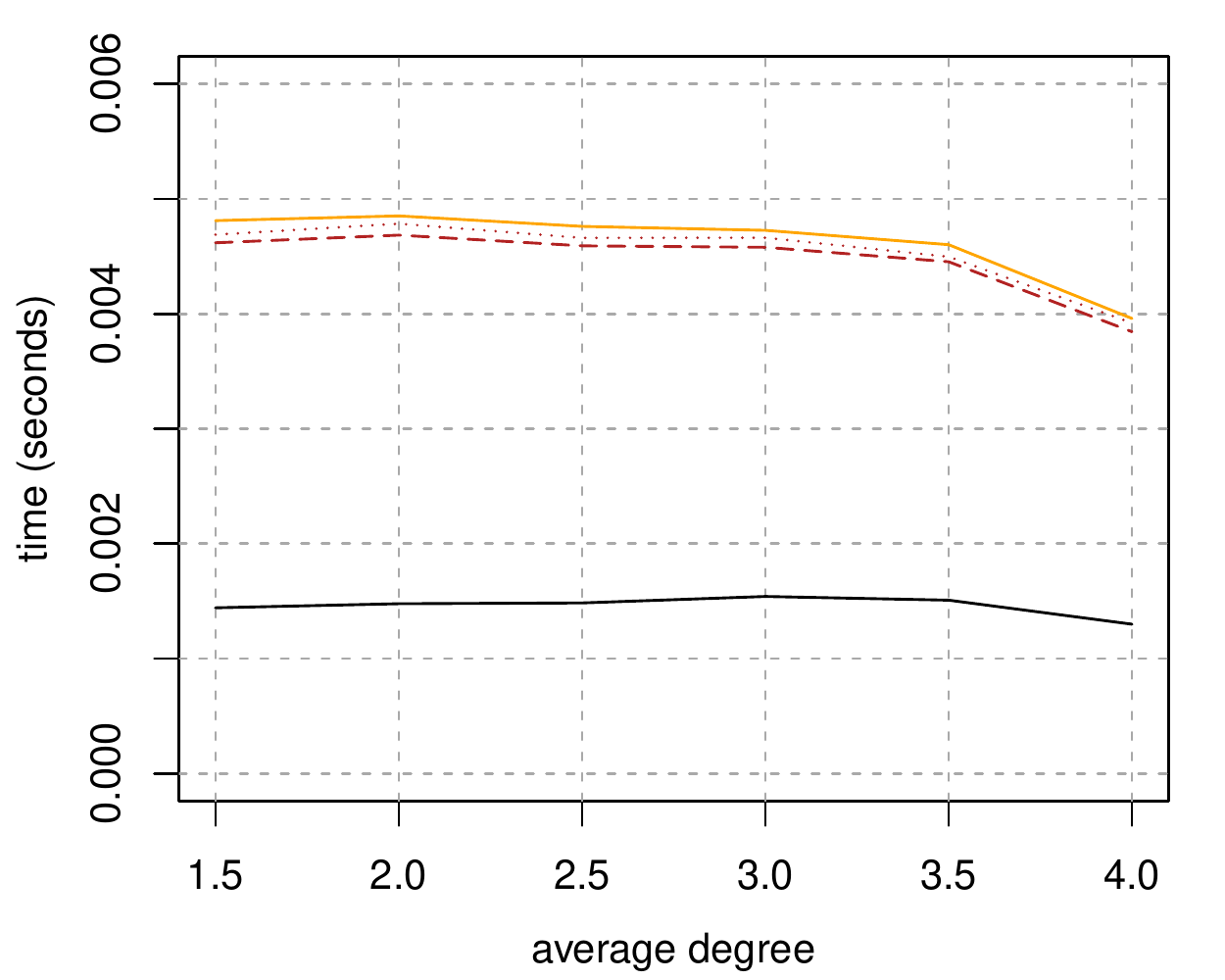}
		\end{minipage}%
	}%
	\hspace{0.01\textwidth}
	\subfigure[$n=50$, $N_{\rm effect}=150$ \label{fig:true:time_50_150}]{
		\begin{minipage}[t]{0.3\textwidth}
			\centering
			\includegraphics[width=\textwidth]{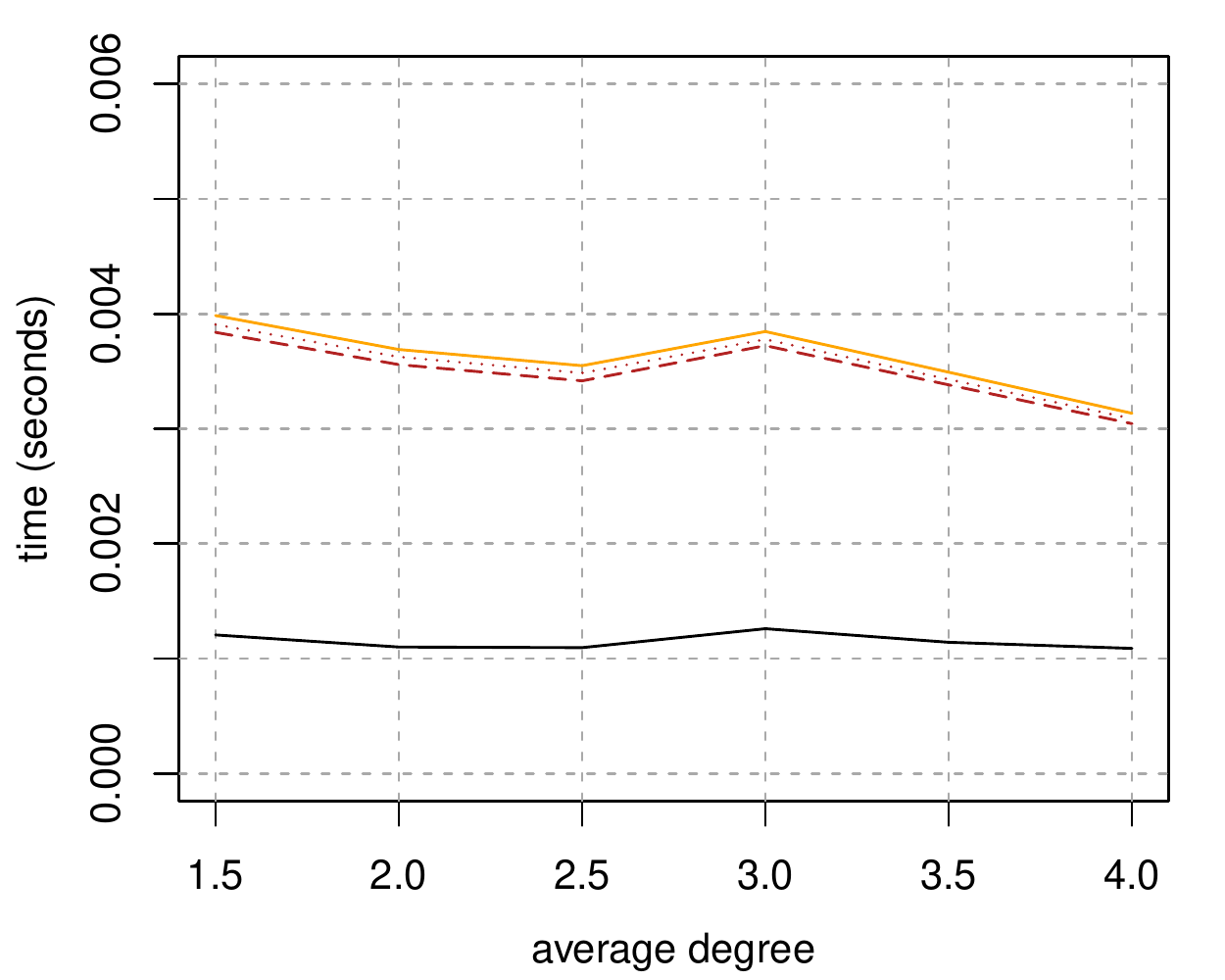}
		\end{minipage}%
	}%
	
	\subfigure[$n=100$, $N_{\rm effect}=50$ \label{fig:true:time_100_50}]{
		\begin{minipage}[t]{0.3\textwidth}
			\centering
			\includegraphics[width=\textwidth]{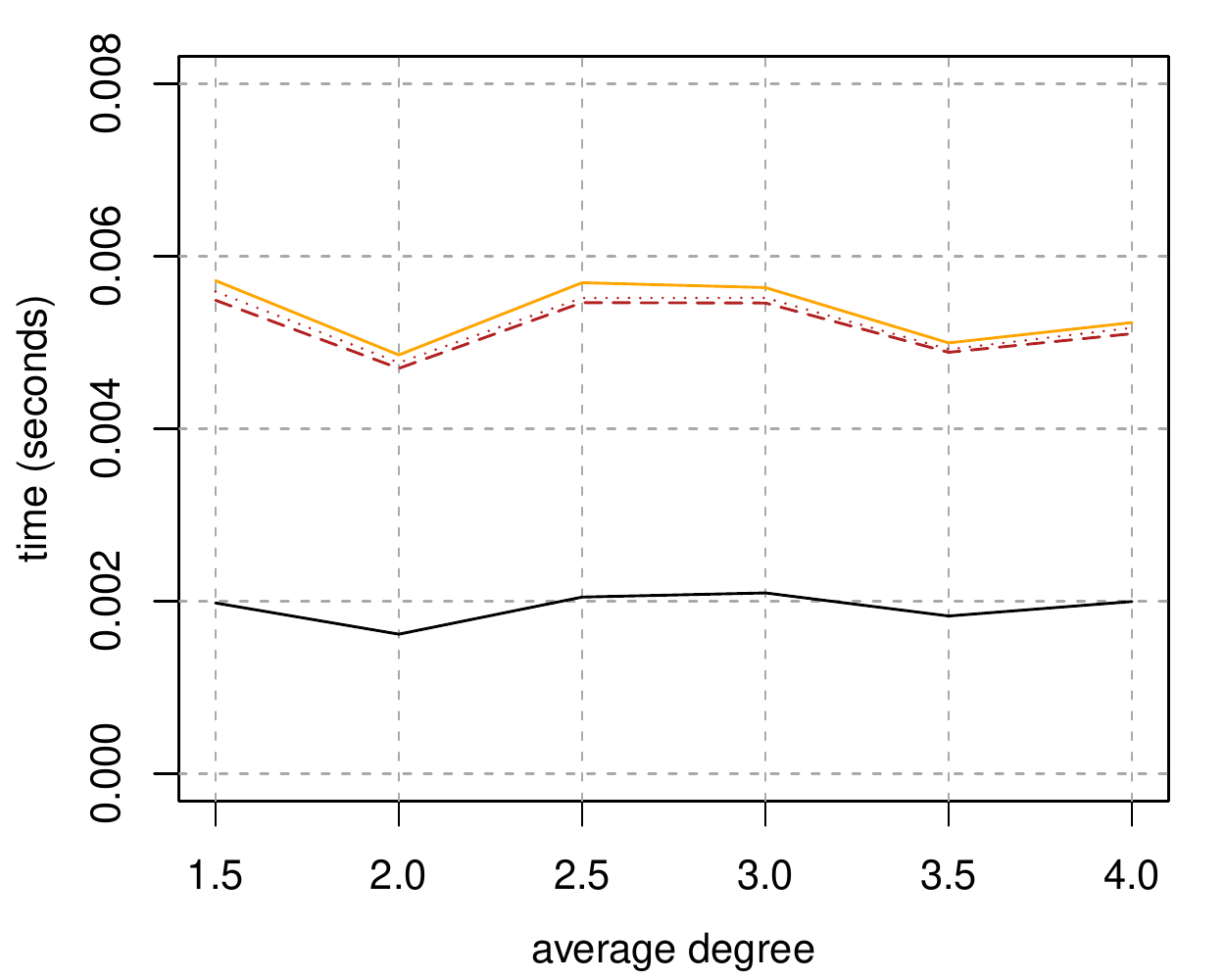}
		\end{minipage}%
	}%
	\hspace{0.01\textwidth}
	\subfigure[$n=100$, $N_{\rm effect}=100$ \label{fig:true:time_100_100}]{
		\begin{minipage}[t]{0.3\textwidth}
			\centering
			\includegraphics[width=\textwidth]{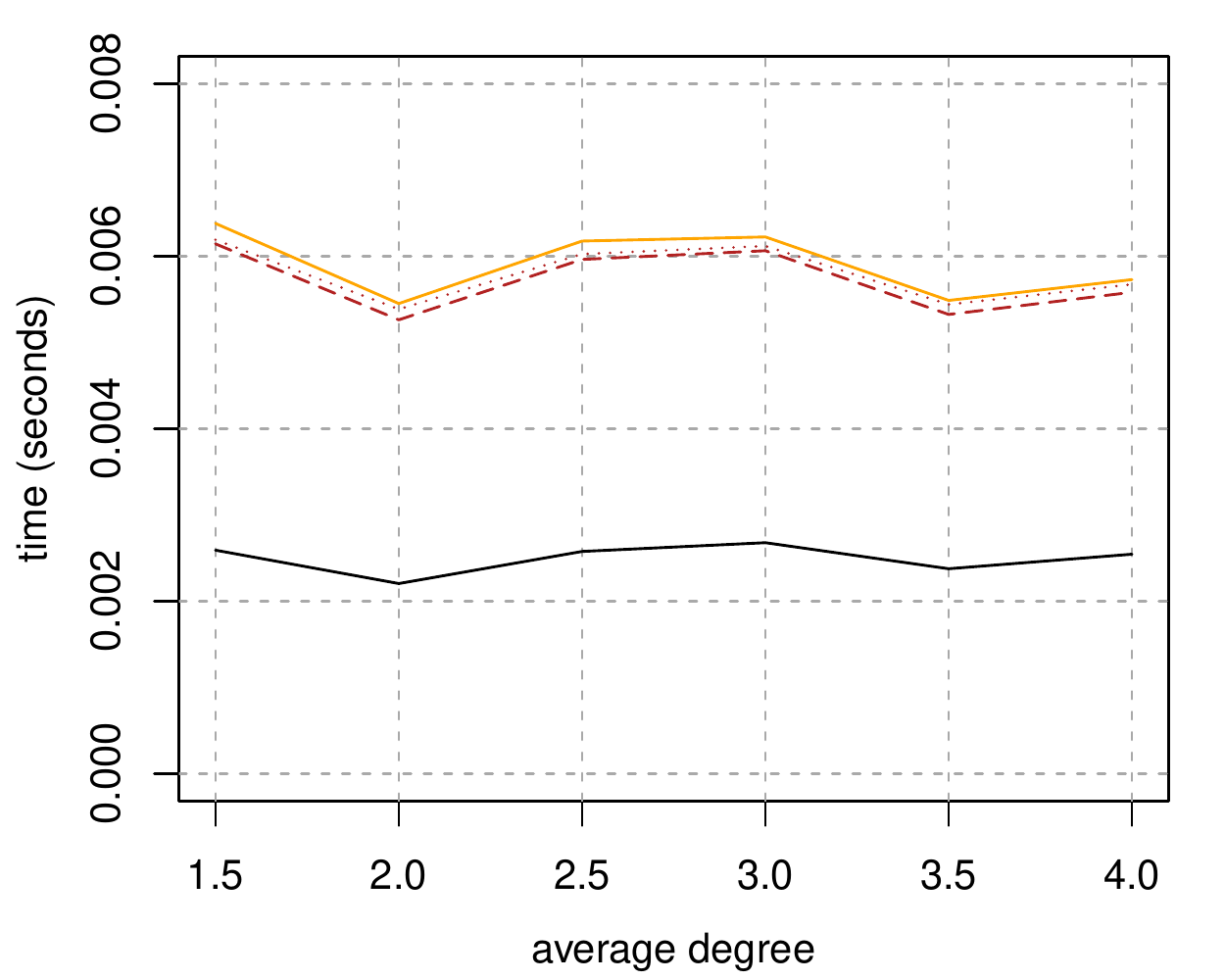}
		\end{minipage}%
	}%
	\hspace{0.01\textwidth}
	\subfigure[$n=100$, $N_{\rm effect}=150$ \label{fig:true:time_100_150}]{
		\begin{minipage}[t]{0.3\textwidth}
			\centering
			\includegraphics[width=\textwidth]{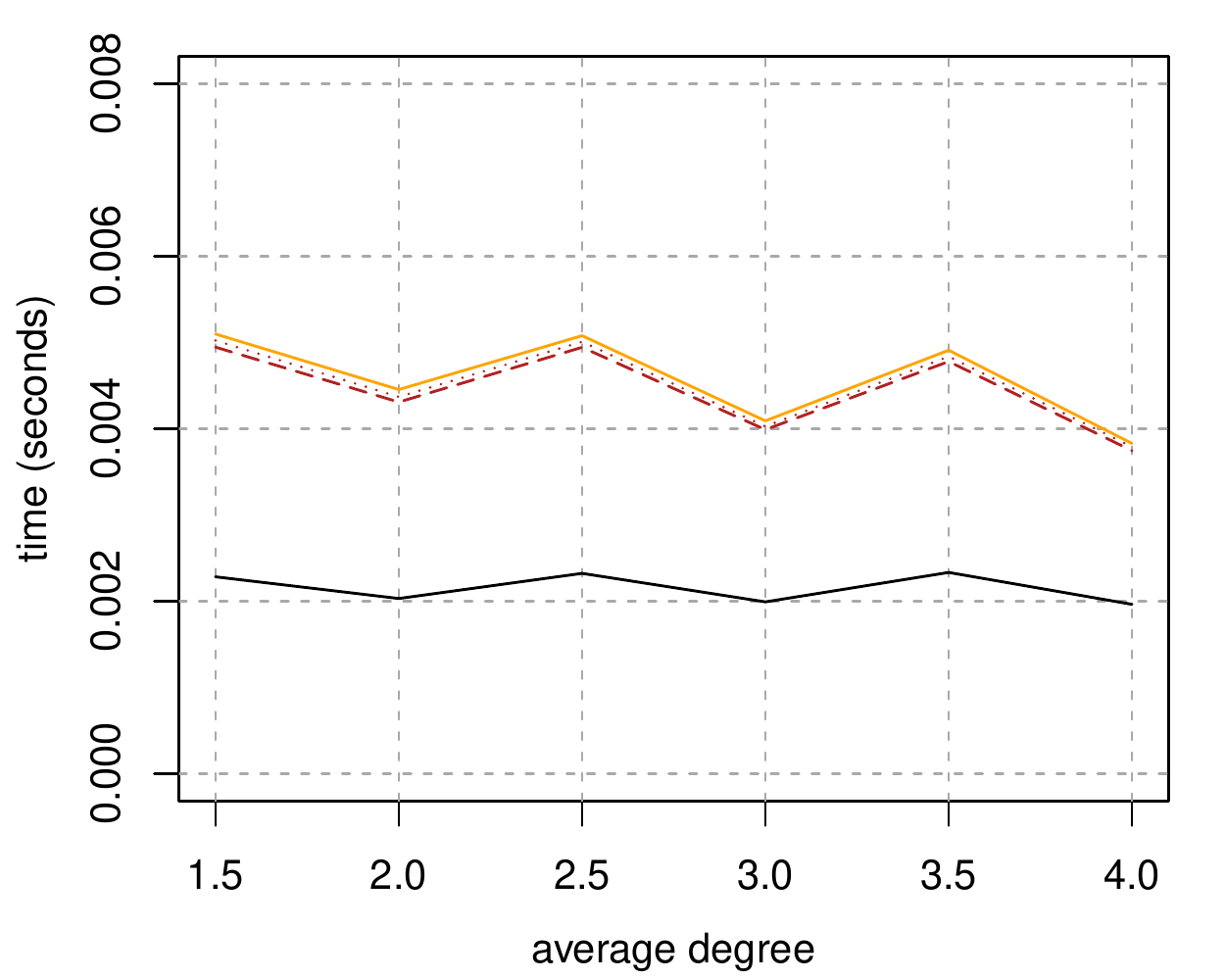}
		\end{minipage}%
	}%
	
	\caption{The CPU time (in seconds) of different methods on random graphs with positive weights. The true graph structures are provided. The CPU time of ``IDA + an + test (all)" and ``IDA + an + test (min/max)" is not shown, as they are more than 50-100 times slower than the other methods.}
	\label{fig:true:time}
\end{figure}

In this section, the true CPDAGs or their local structures are provided to exclude estimation biases caused by graph structure learning  from data.  In this case, the global ITC shown in Algorithm \ref{algo:global} can identify all causal relations correctly since the input CPDAG is true. Except for the global ITC, the local ITC and the CE-based methods need to perform hypothesis tests, which may introduce errors.  To assess  these methods, we run experiments on data with positive weights (Uniform($[0.8, 1.6]$))  as well as a  mixture of negative and positive weights (Uniform($[-1.6, -0.8]\cup[0.8, 1.6]$)). The results on the graphs with positive weights are shown in this section and the rests are in~\ref{app:app:mix}.

\begin{figure}[!h]
	\centering
	\subfigure{
		\begin{minipage}[t]{1\textwidth}
			\centering
			\includegraphics[width=\textwidth]{figs/revision/legend.pdf}
		\end{minipage}%
	}%
	\vspace{-1em}
	\addtocounter{subfigure}{-1}
	
	\subfigure[$n=50$, $N_{\rm effect}=50$ \label{fig:true:time_zoom_50_50}]{
		\begin{minipage}[t]{0.3\textwidth}
			\centering
			\includegraphics[width=\textwidth]{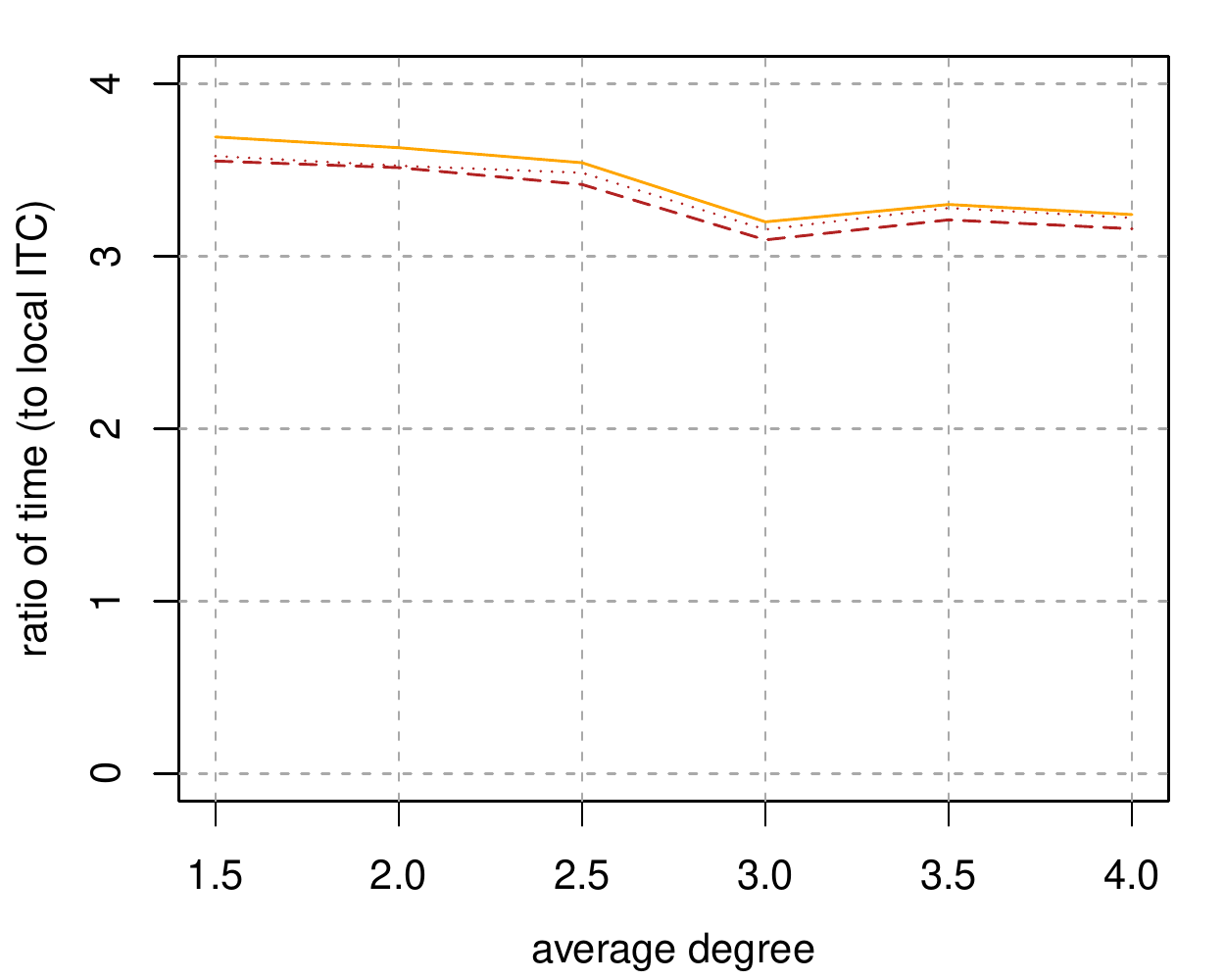}
		\end{minipage}%
	}%
	\hspace{0.01\textwidth}
	\subfigure[$n=50$, $N_{\rm effect}=100$  \label{fig:true:time_zoom_50_100}]{
		\begin{minipage}[t]{0.3\textwidth}
			\centering
			\includegraphics[width=\textwidth]{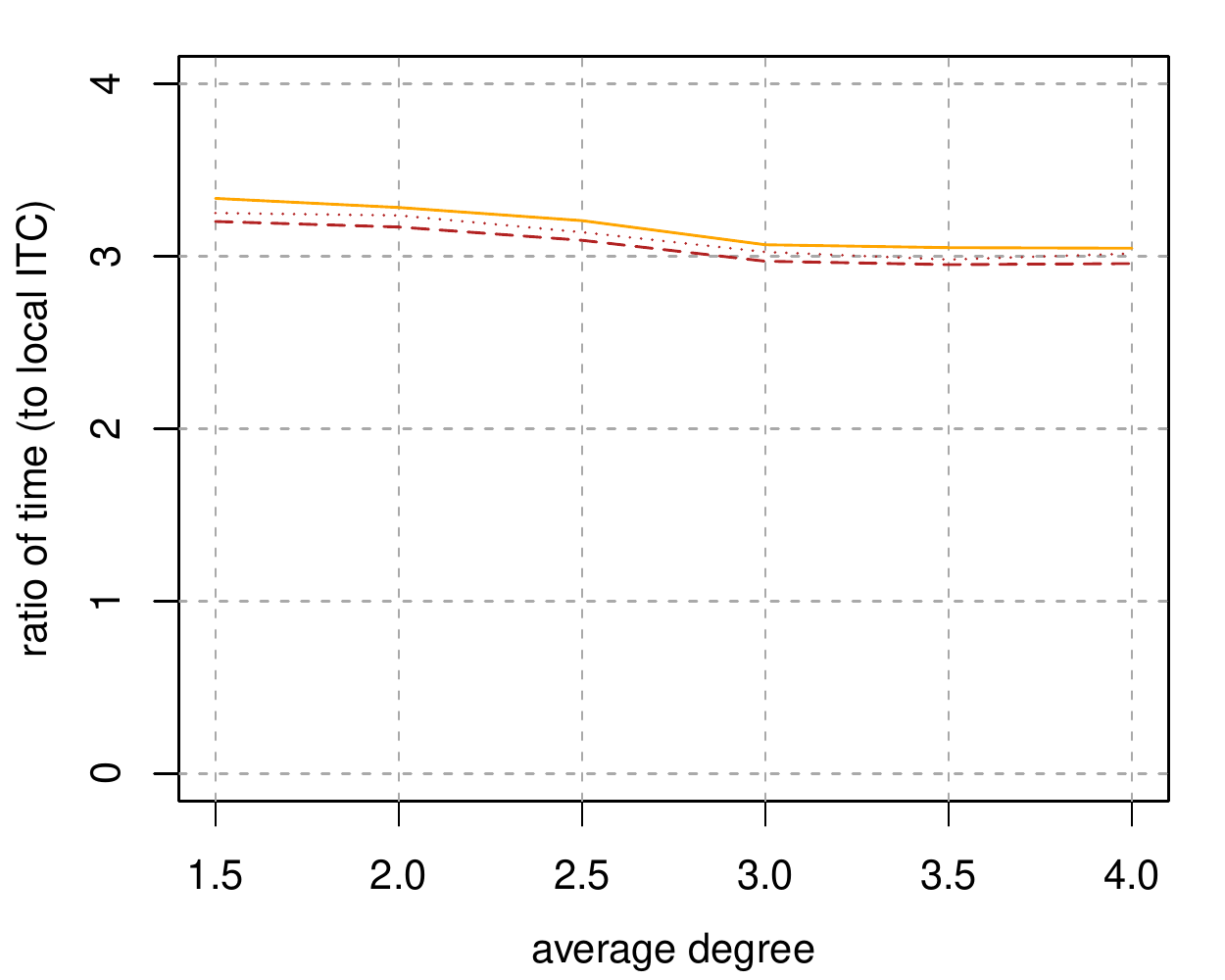}
		\end{minipage}%
	}%
	\hspace{0.01\textwidth}
	\subfigure[$n=50$, $N_{\rm effect}=150$ \label{fig:true:time_zoom_50_150}]{
		\begin{minipage}[t]{0.3\textwidth}
			\centering
			\includegraphics[width=\textwidth]{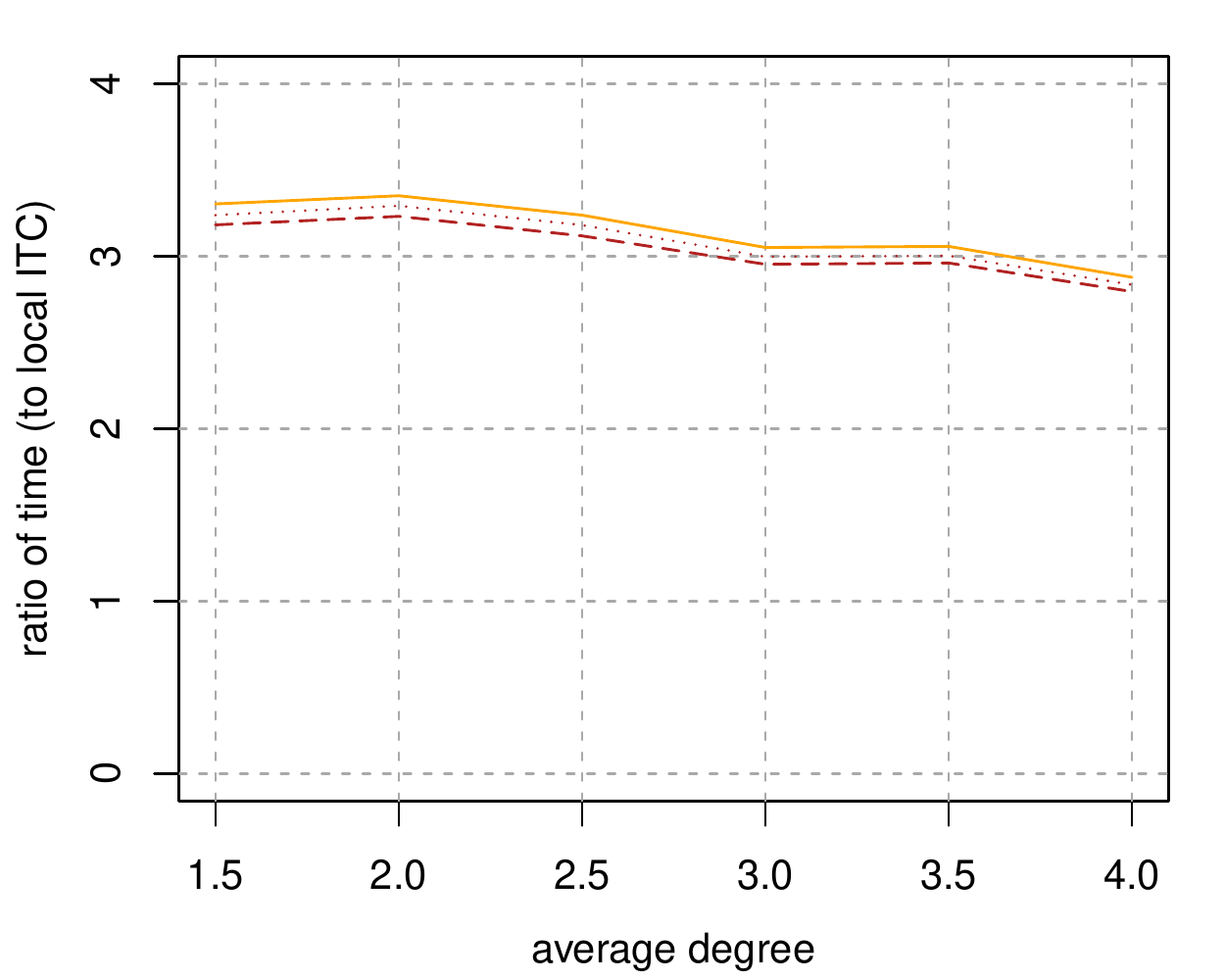}
		\end{minipage}%
	}%
	
	\subfigure[$n=100$, $N_{\rm effect}=50$ \label{fig:true:time_zoom_100_50}]{
		\begin{minipage}[t]{0.3\textwidth}
			\centering
			\includegraphics[width=\textwidth]{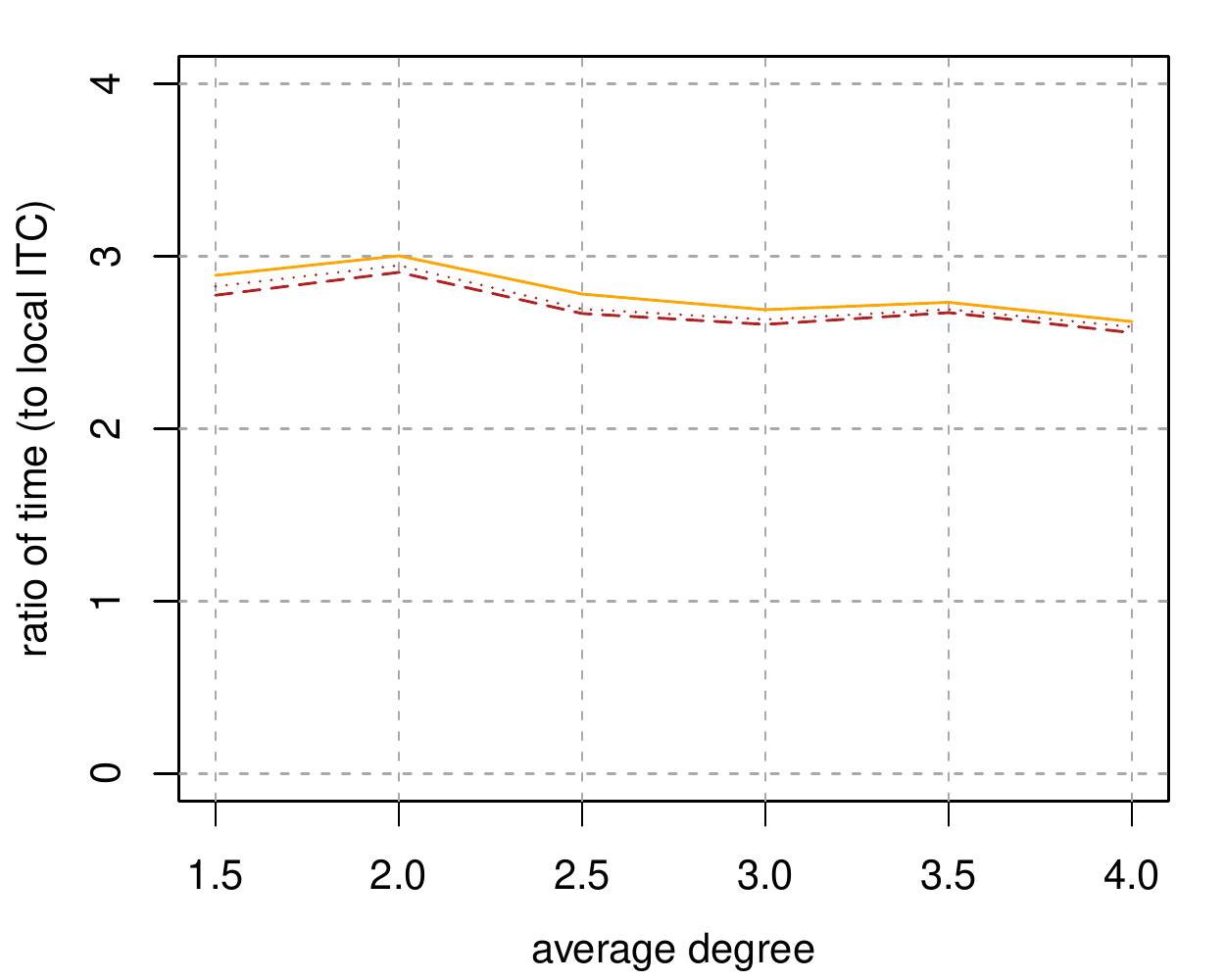}
		\end{minipage}%
	}%
	\hspace{0.01\textwidth}
	\subfigure[$n=100$, $N_{\rm effect}=100$ \label{fig:true:time_zoom_100_100}]{
		\begin{minipage}[t]{0.3\textwidth}
			\centering
			\includegraphics[width=\textwidth]{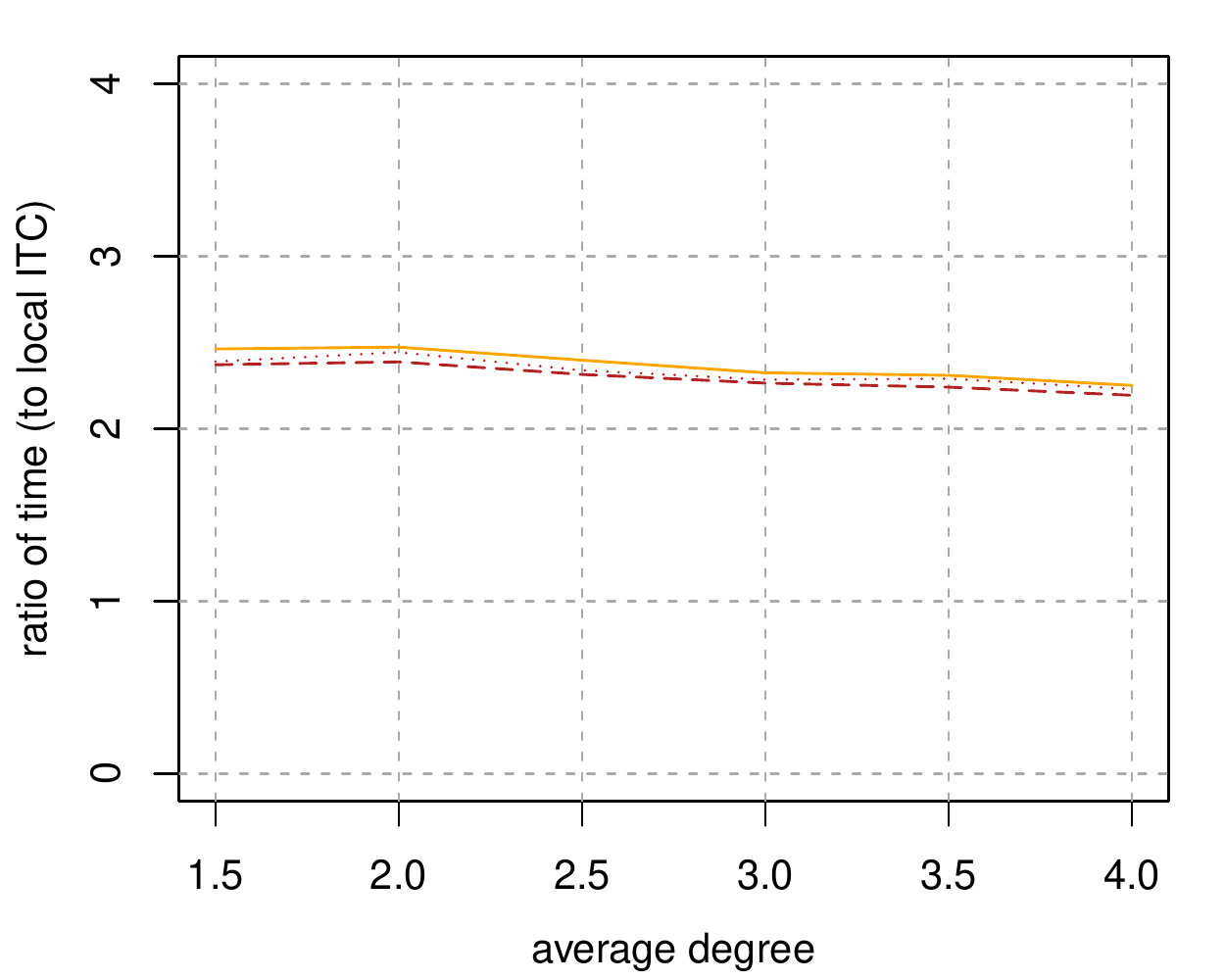}
		\end{minipage}%
	}%
	\hspace{0.01\textwidth}
	\subfigure[$n=100$, $N_{\rm effect}=150$ \label{fig:true:time_zoom_100_150}]{
		\begin{minipage}[t]{0.3\textwidth}
			\centering
			\includegraphics[width=\textwidth]{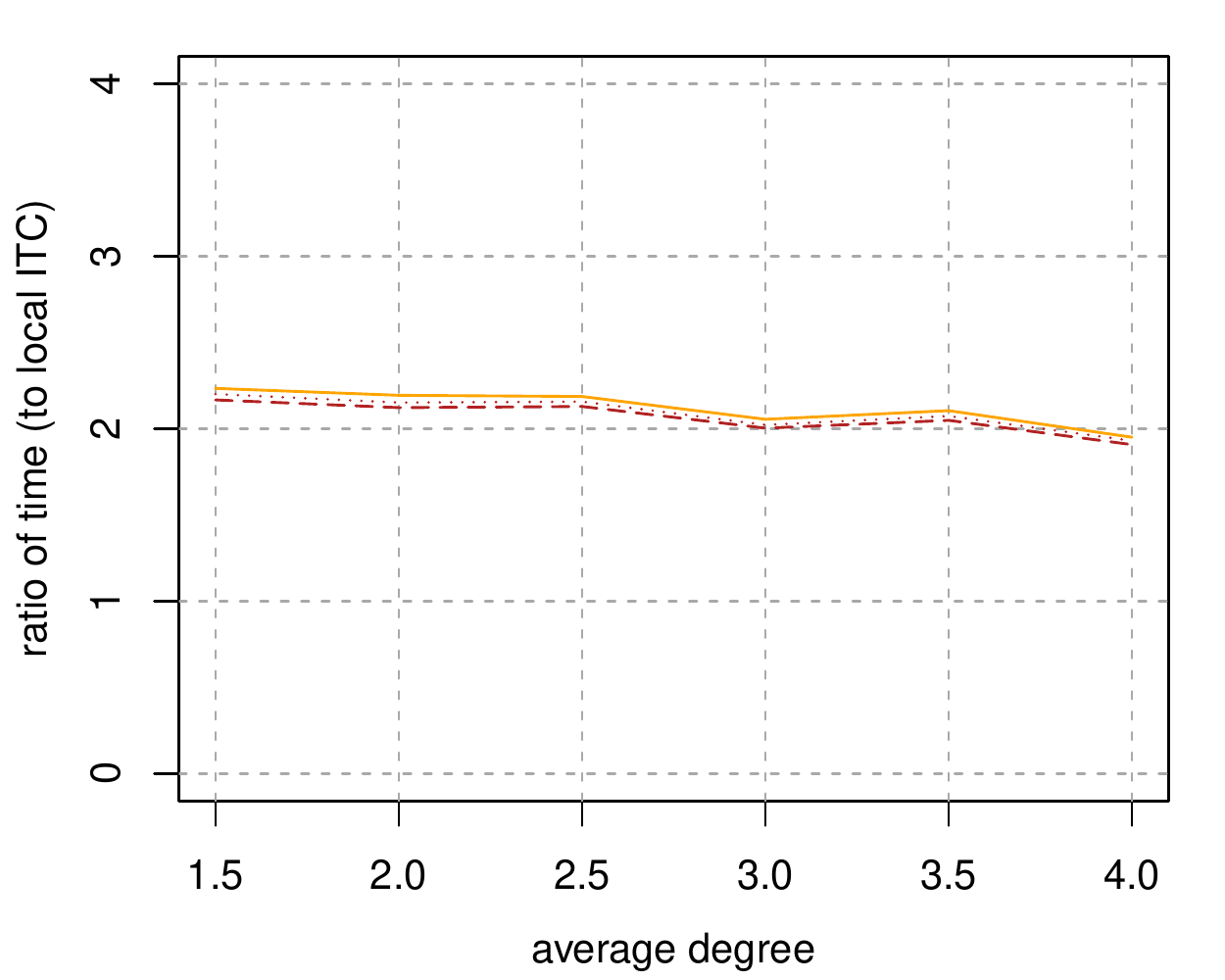}
		\end{minipage}%
	}%
	
	\caption{The ratio of the CPU time of different methods to that of the local ITC on random graphs with positive weights. The true graph structures are provided. The results of ``IDA + an + test (all)" and ``IDA + an + test (min/max)" are not shown, as they are more than 50-100 times slower than the other methods. }
	\label{fig:true:zoom}
\end{figure}

Figure~\ref{fig:true:kappa} shows the Kappa coefficients of the proposed local method and five CE-based methods on random graphs with positive weights. We can see that the proposed  local ITC is significantly better than the CE-based methods, especially when the sample size is small and the average degree is high. Increasing the sample size can improve the performance of all methods, and reduce the difference between the local ITC and the CE-based methods.

In these experiments, probably because the total number of hypothesis tests is not large, adjusting p-values for  multiple comparisons does not bring much improvement. Besides, testing all enumerated effects usually performs better than testing the minimum and maximum absolute effects. Moreover, utilizing non-ancestral information can improve the performance. Consequently, the CE-based-method denoted by ``IDA + an + test (all)" performs the best among the five CE-based-methods.

However, utilizing non-ancestral information will significantly increase the computational time because of the use of Meek's rules. In our experiments, the CE-based methods which utilize non-ancestral relations are 50-100 times slower than the others. To compare the other methods, Figure~\ref{fig:true:time} reports the CPU time (in seconds) and Figure~\ref{fig:true:zoom} further shows the ratio of the time used by each CE-based method to the local ITC. As one can see from the figures, the CE-based methods without using non-ancestral relations are 2-4 times slower than the local ITC.

Benefiting from  fewer hypothesis tests, the  local ITC algorithm is more stable,  more accurate, and more efficient  than the CE-based methods. Additional  evidence also comes from the experiments on models with mixed edge weights. In these experiments, although all Kappa coefficients drop, the Kappa coefficients of the local ITC drop less than those of the CE-based methods. The details can be found in~\ref{app:app:mix}.

\subsection{Learning with Estimated Graphs}\label{sec:sec:est}

\begin{figure}[t!]
	\centering
	\subfigure{
		\begin{minipage}[t]{1\textwidth}
			\centering
			\includegraphics[width=1\textwidth]{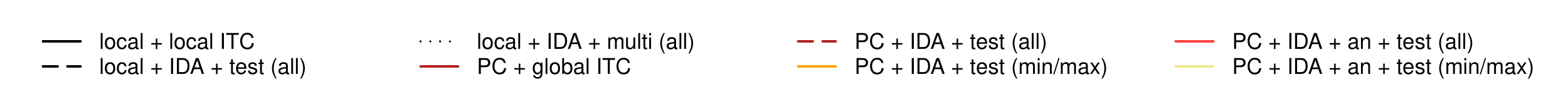}
		\end{minipage}%
	}%
	\vspace{-1.2em}
	\addtocounter{subfigure}{-1}
	
	\subfigure[$n=50$, $N=(100,100)$ \label{fig:learned:kappa_50_100_100}]{
		\begin{minipage}[t]{0.3\textwidth}
			\centering
			\includegraphics[width=\textwidth]{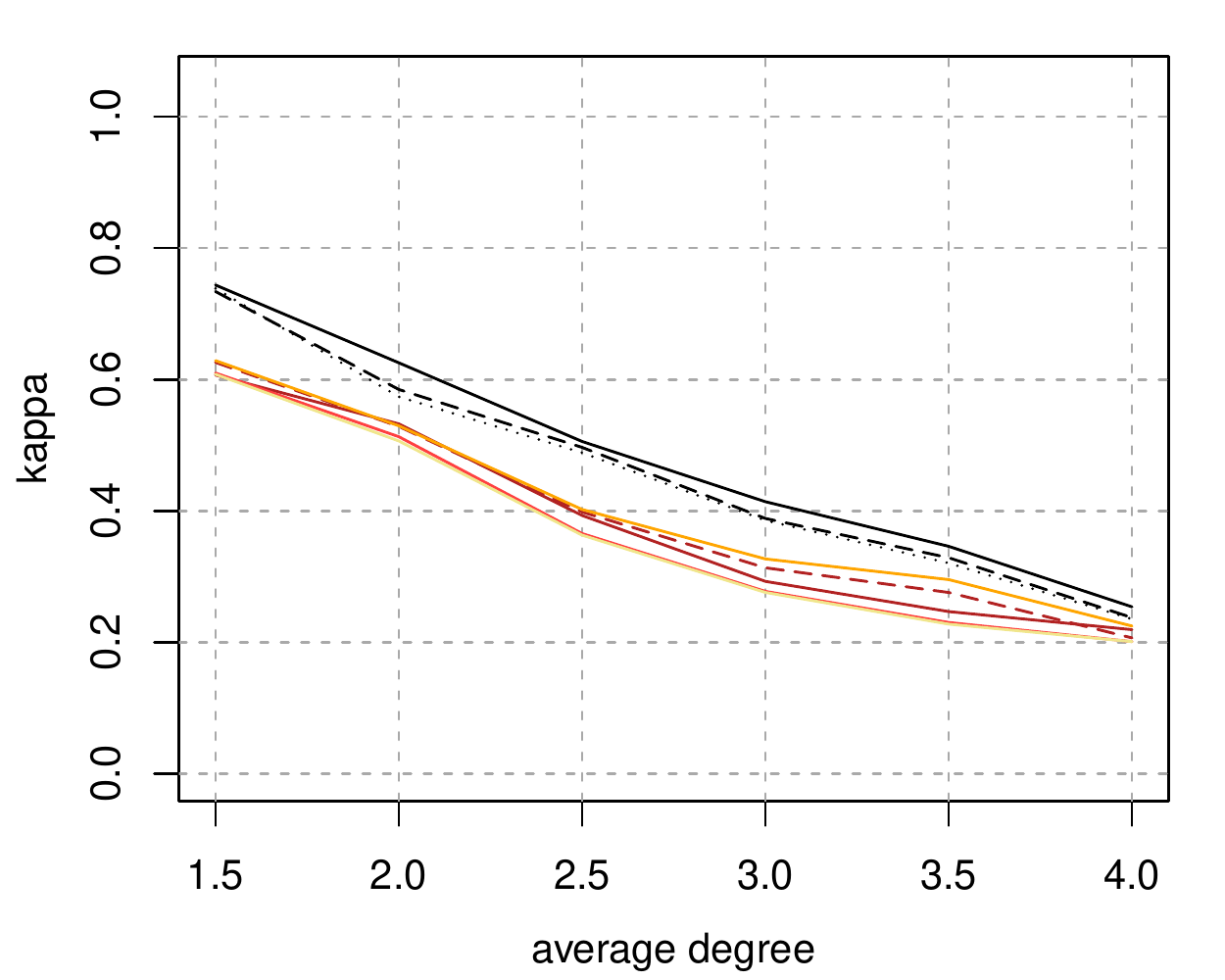}
		\end{minipage}%
	}%
	\hspace{0.01\textwidth}
	\subfigure[$n=50$, $N=(200,100)$ \label{fig:learned:kappa_50_200_100}]{
		\begin{minipage}[t]{0.3\textwidth}
			\centering
			\includegraphics[width=\textwidth]{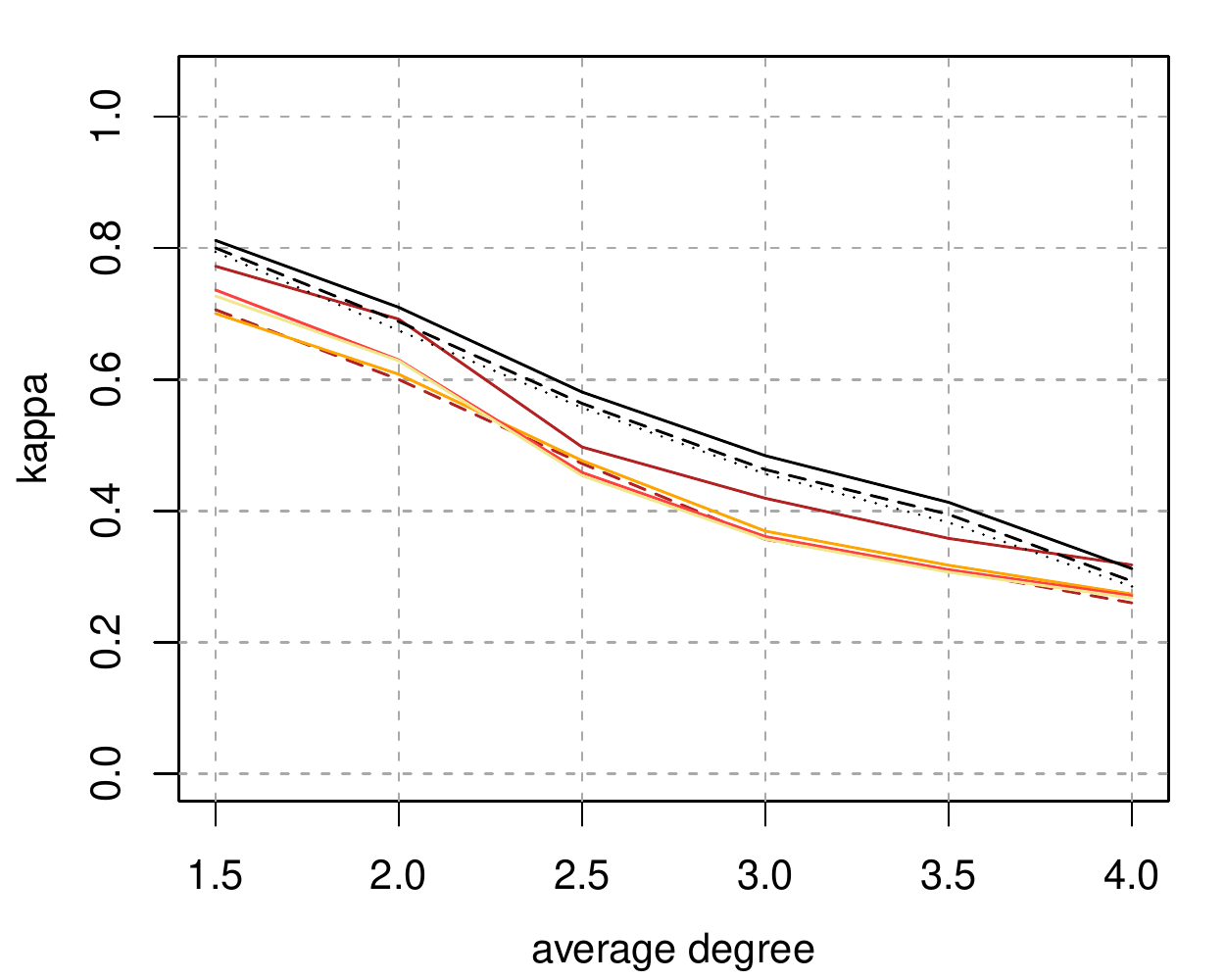}
		\end{minipage}%
	}%
	\hspace{0.01\textwidth}
	\subfigure[$n=50$, $N=(500,150)$ \label{fig:learned:kappa_50_500_150}]{
		\begin{minipage}[t]{0.3\textwidth}
			\centering
			\includegraphics[width=\textwidth]{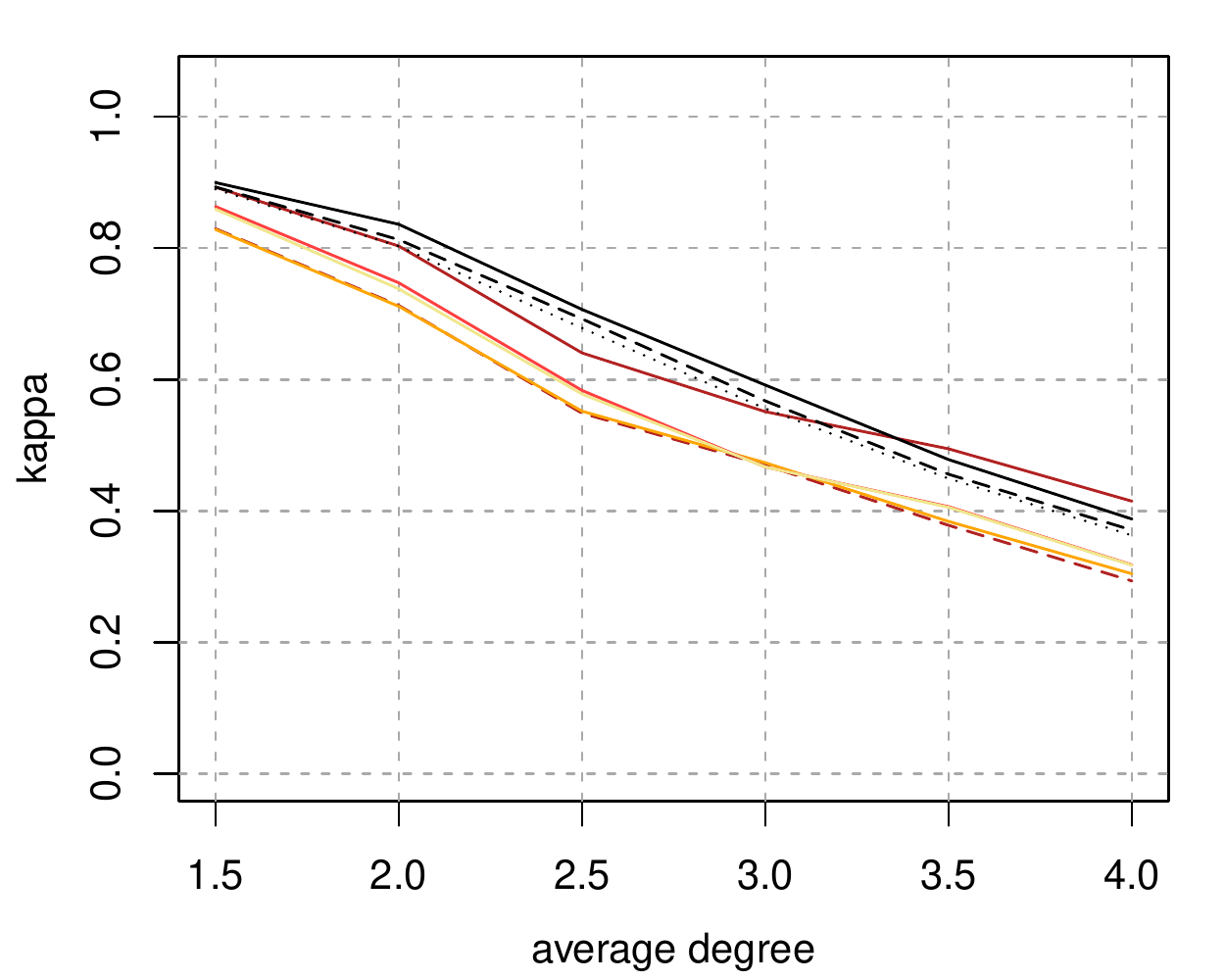}
		\end{minipage}%
	}%
	
	\subfigure[$n=100, N=(100,100)$ \label{fig:learned:kappa_100_50}]{
		\begin{minipage}[t]{0.3\textwidth}
			\centering
			\includegraphics[width=\textwidth]{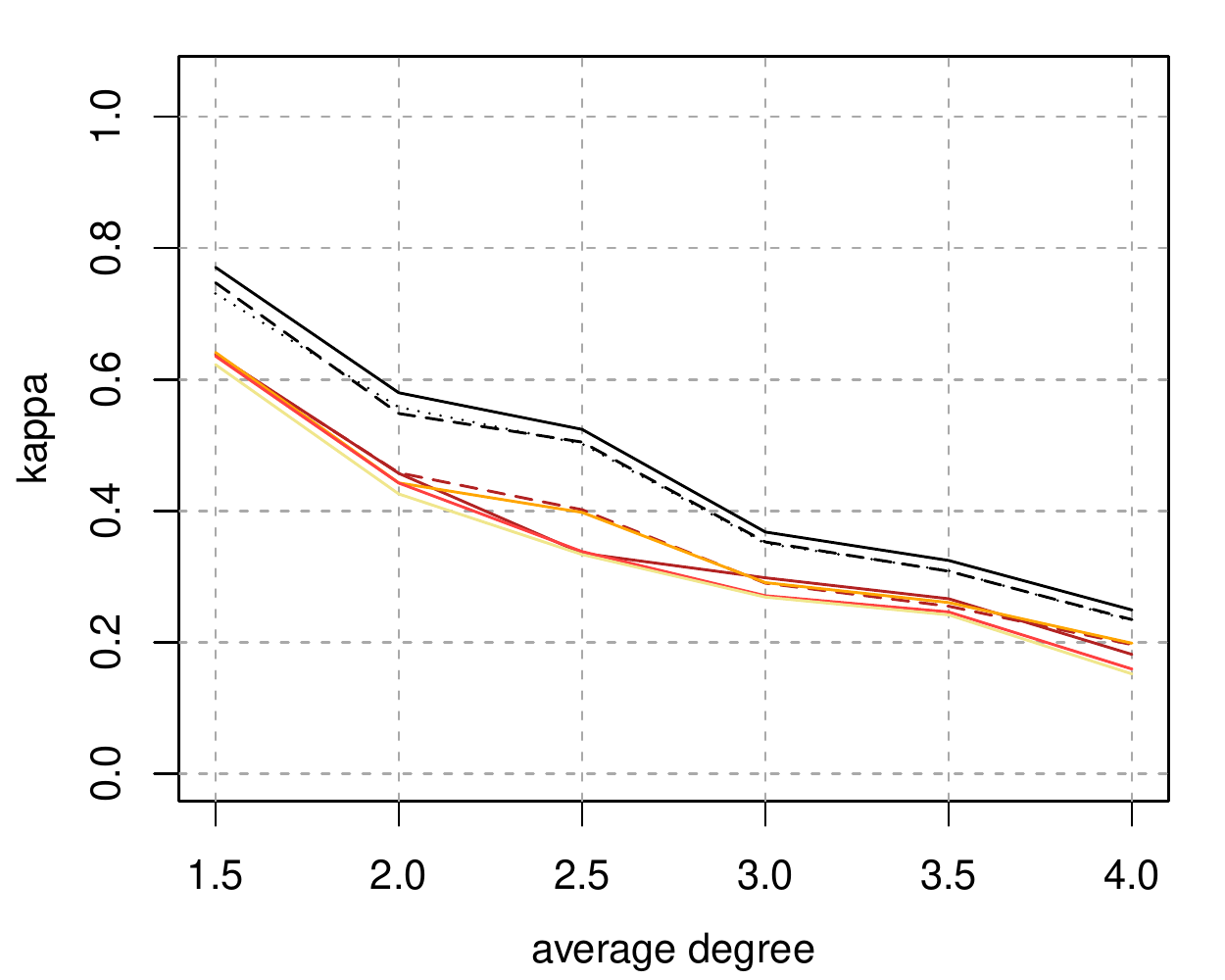}
		\end{minipage}%
	}%
	\hspace{0.01\textwidth}
	\subfigure[$n=100,N=(200,100)$ \label{fig:learned:kappa_100_100}]{
		\begin{minipage}[t]{0.3\textwidth}
			\centering
			\includegraphics[width=\textwidth]{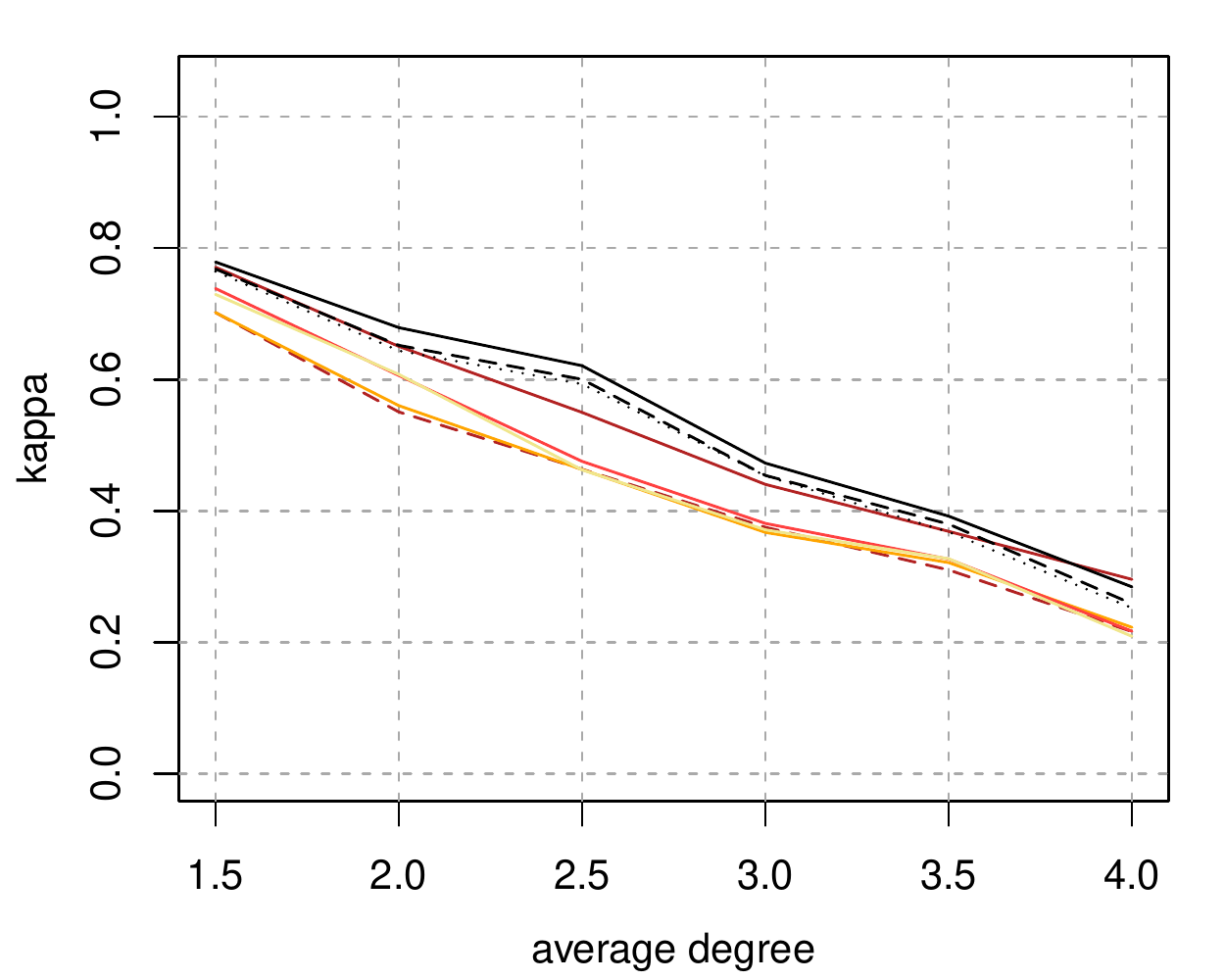}
		\end{minipage}%
	}%
	\hspace{0.01\textwidth}
	\subfigure[$n = 100$, $N=(500,150)$ \label{fig:learned:kappa_100_150}]{
		\begin{minipage}[t]{0.3\textwidth}
			\centering
			\includegraphics[width=\textwidth]{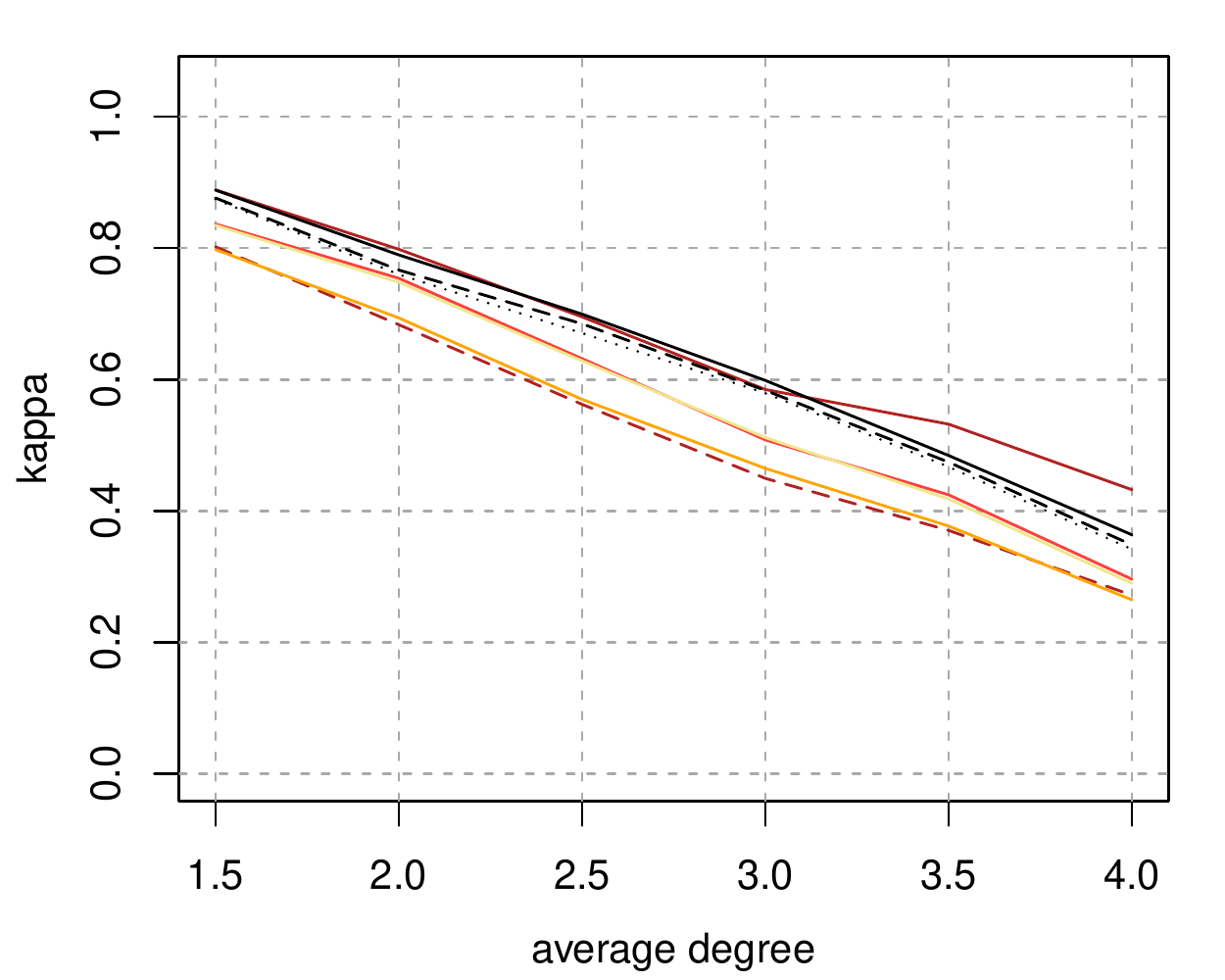}
		\end{minipage}%
	}%
	\caption{The Kappa coefficients of different methods on random graphs with positive weights. The graph structures are learned from data. $N=(N_{\rm graph}, N_{\rm effect})$ denotes the sample sizes for learning graphs and estimating causal effects. 
}
	\label{fig:learn:kappa}
\end{figure}

In this section, we further study experimentally our proposed methods when the true causal structures are not available. We used the variant of MB-by-MB~\citep{liu2020local} to learn the parents and siblings of the vertices of interest (denoted by ``local +"), and used the PC algorithm, the stable PC algorithm (denoted by ``PCS +") and GES to learn entire CPDAGs. The learned structures are then passed to the local ITC, the global ITC and the CE-based methods. For ease of presentation, we mainly report twelve methods in this section, but the conclusions obtained coincide with  all the experiments.

\begin{figure}[t!]
	\centering
	\subfigure{
		\begin{minipage}[t]{1\textwidth}
			\centering
			\includegraphics[width=\textwidth]{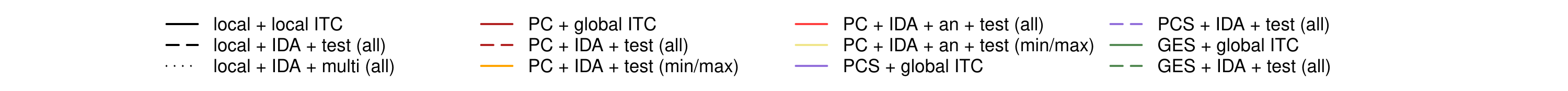}
		\end{minipage}%
	}%
	\vspace{-1em}
	\addtocounter{subfigure}{-1}
	
	\subfigure[$n=50$, $N=(100,100)$ \label{fig:learned:time_50_100_100}]{
		\begin{minipage}[t]{0.3\textwidth}
			\centering
			\includegraphics[width=\textwidth]{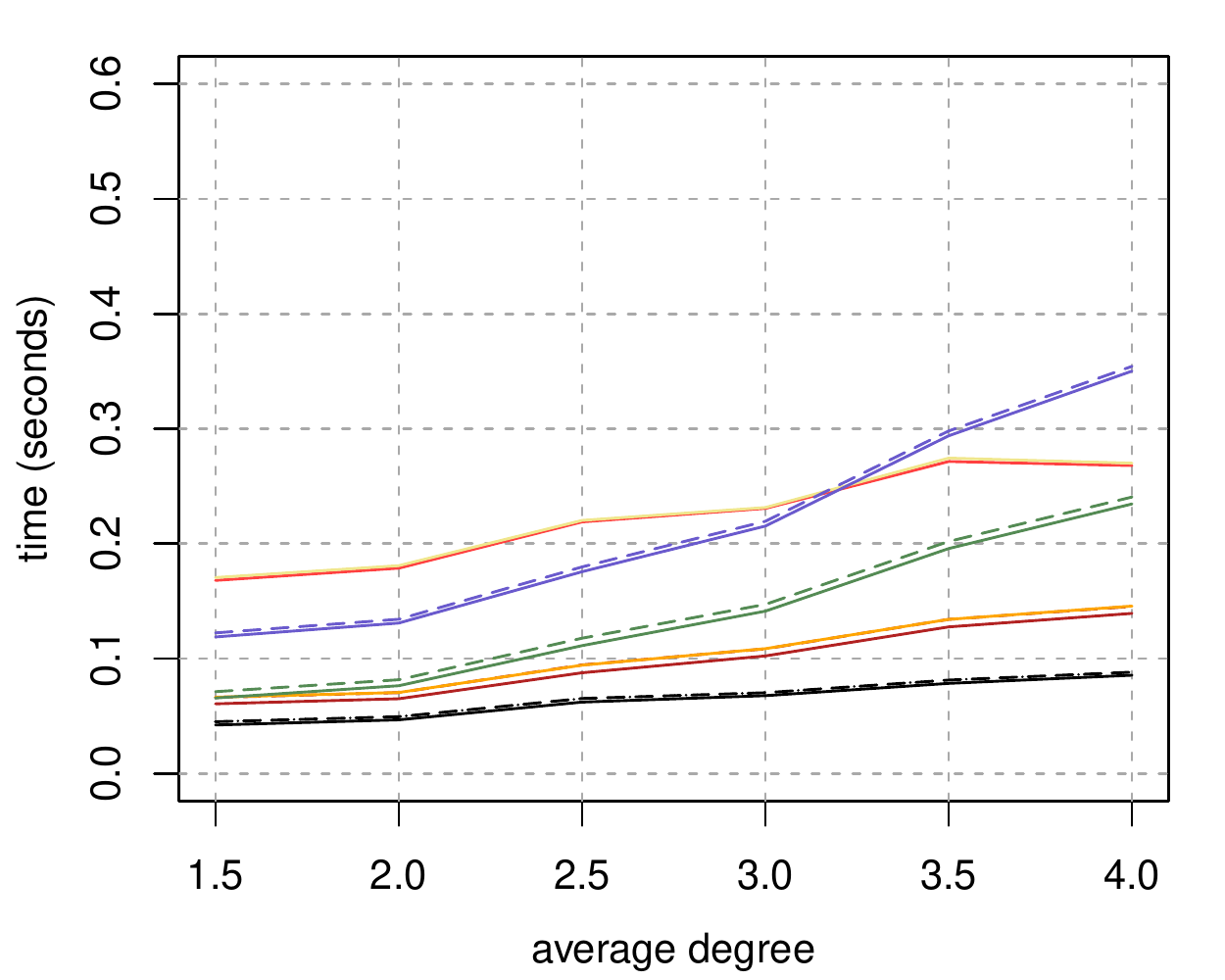}
		\end{minipage}%
	}%
	\hspace{0.01\textwidth}
	\subfigure[$n=50$, $N=(200,100)$ \label{fig:learned:time_50_200_100}]{
		\begin{minipage}[t]{0.3\textwidth}
			\centering
			\includegraphics[width=\textwidth]{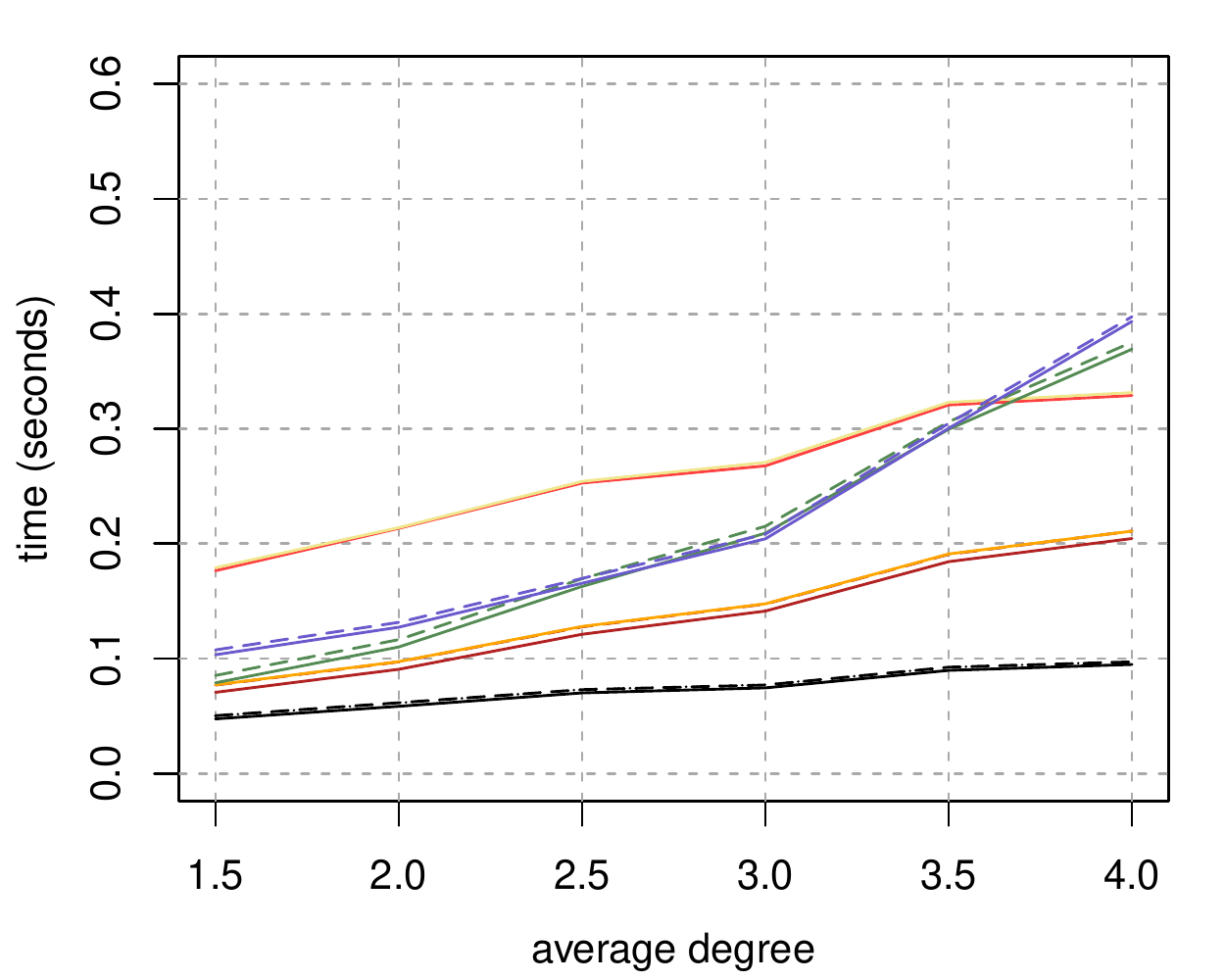}
		\end{minipage}%
	}%
	\hspace{0.01\textwidth}
	\subfigure[$n=50$, $N=(500,150)$ \label{fig:learned:time_50_500_150}]{
		\begin{minipage}[t]{0.3\textwidth}
			\centering
			\includegraphics[width=\textwidth]{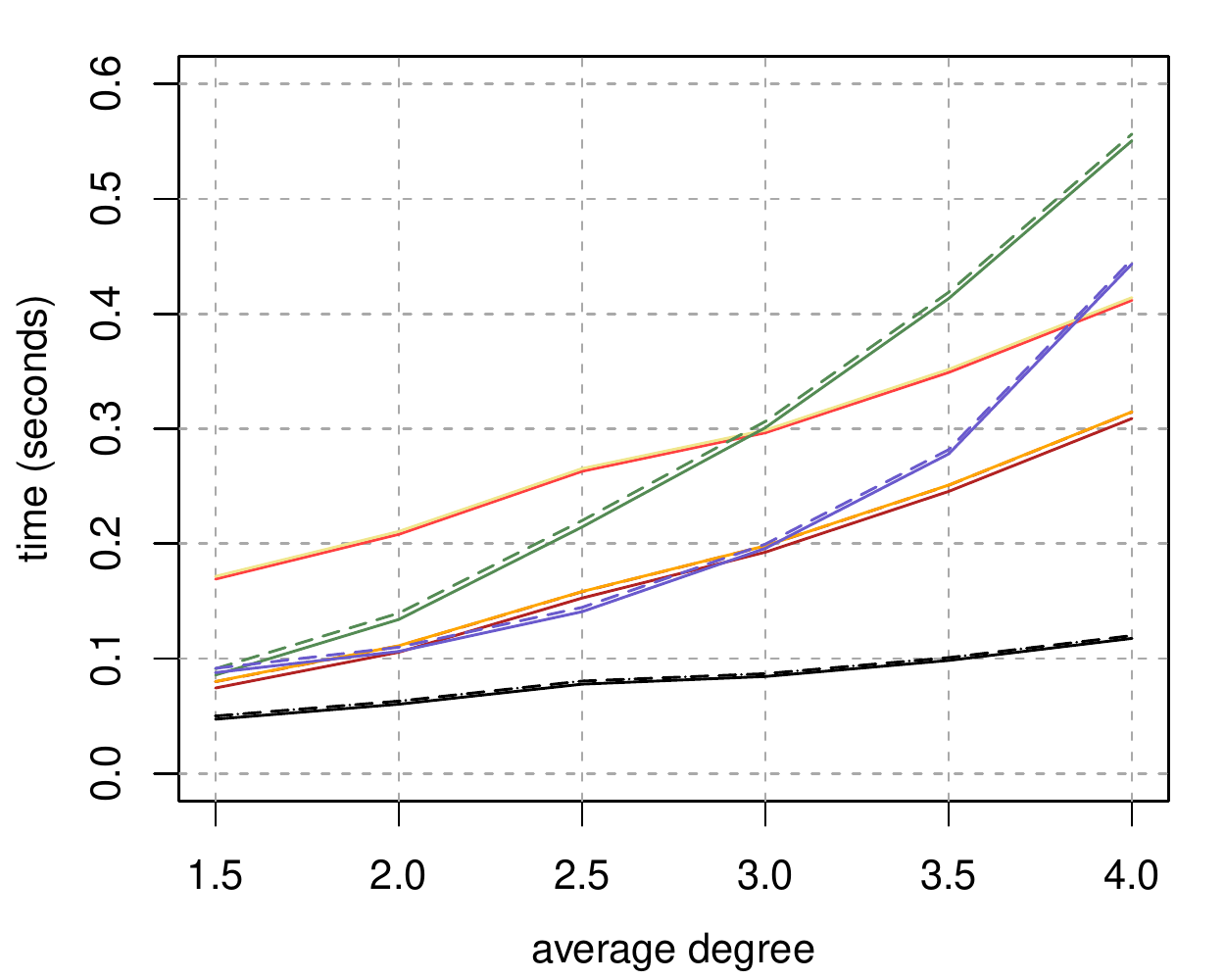}
		\end{minipage}%
	}%
	
	\subfigure[$n=100,N=(100,100)$ \label{fig:learned:time_100_50}]{
		\begin{minipage}[t]{0.3\textwidth}
			\centering
			\includegraphics[width=\textwidth]{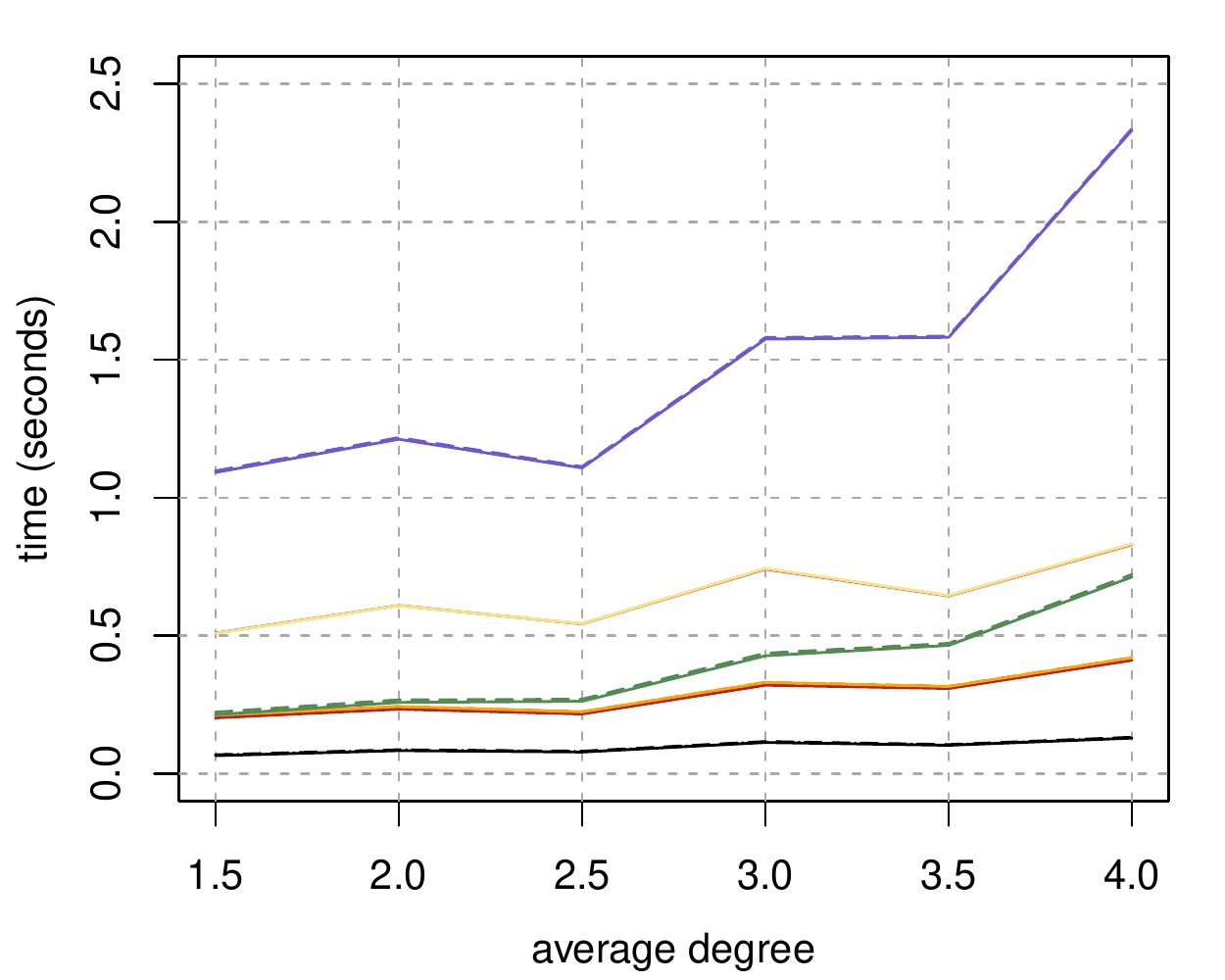}
		\end{minipage}%
	}%
	\hspace{0.01\textwidth}
	\subfigure[$n=100$, $N=(200,100)$ \label{fig:learned:time_100_100}]{
		\begin{minipage}[t]{0.3\textwidth}
			\centering
			\includegraphics[width=\textwidth]{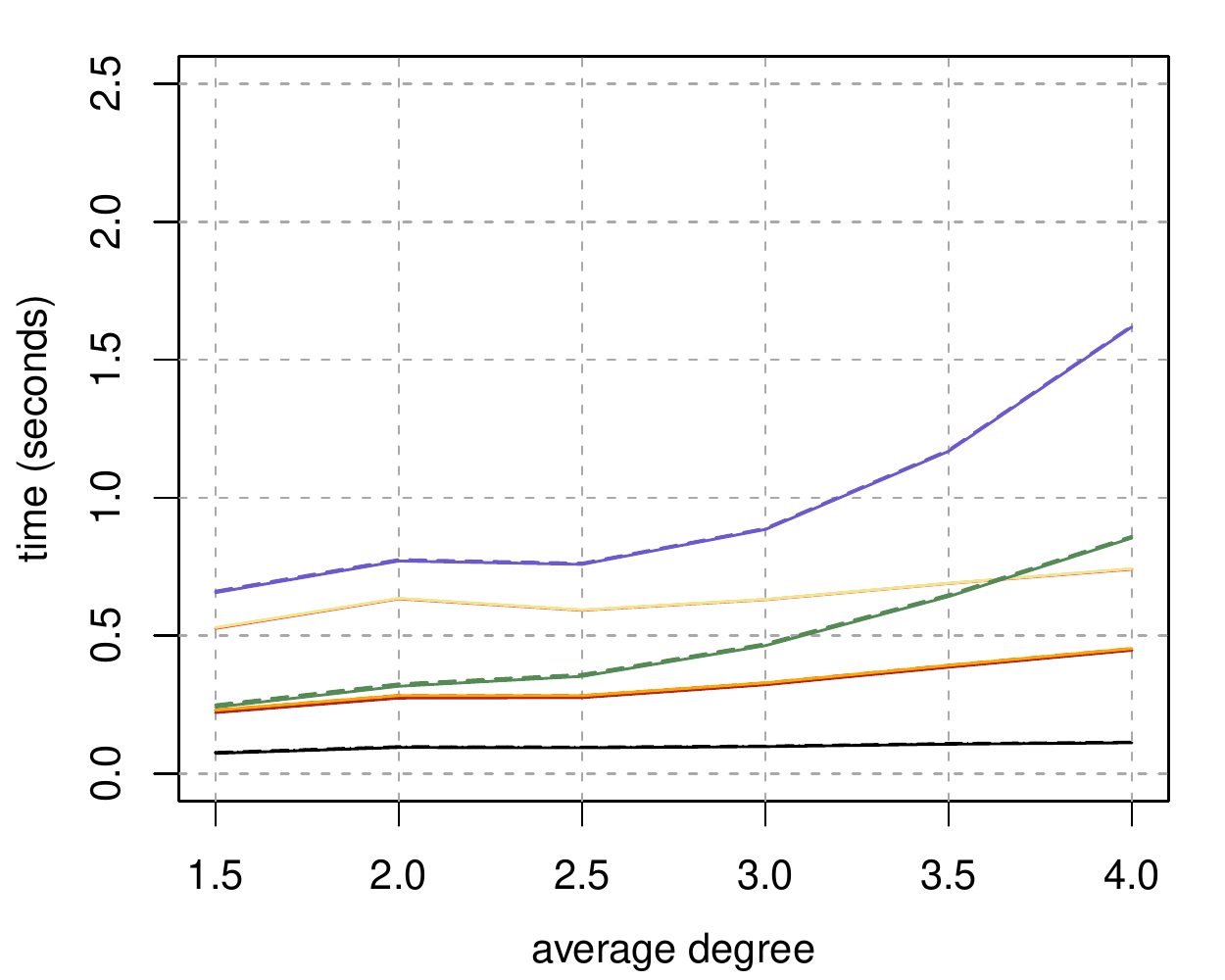}
		\end{minipage}%
	}%
	\hspace{0.01\textwidth}
	\subfigure[$n=100$, $N=(500,150)$ \label{fig:learned:time_100_150}]{
		\begin{minipage}[t]{0.3\textwidth}
			\centering
			\includegraphics[width=\textwidth]{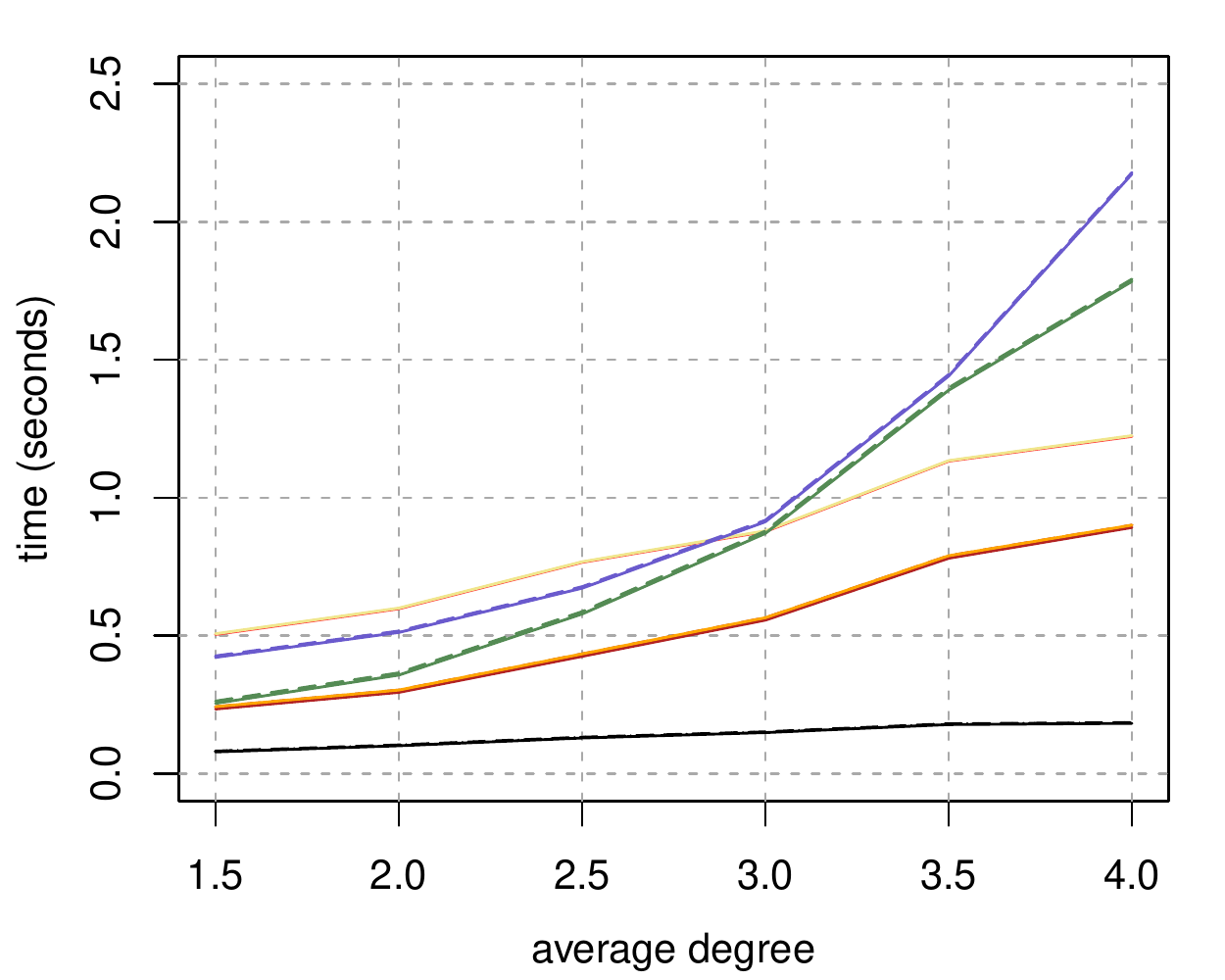}
		\end{minipage}%
	}%
	
	\caption{The total CPU time (in seconds) of different methods on random graphs with positive weights. The graph structures are learned from data. $N=(N_{\rm graph}, N_{\rm effect})$ denotes the sample sizes for learning graphs and estimating causal effects.}
	\label{fig:learn:time}
\end{figure}

Figure \ref{fig:learn:kappa} shows the Kappa coefficients based on 50- and 100-node graphs. For ease of presentation, we omit the results of the global ITC combined with PCS and GES as well as the results of the CE-based methods combined with PCS and GES, since the PCS-based methods perform similarly to PC and the GES-based methods do not perform well (see~\ref{app:app:detailed} for the detailed TPRs and FPRs).
As one can see, the proposed local ITC outperforms the other methods in almost all settings, especially when the sample size is small. The CE-based methods combined with the variant of MB-by-MB perform slightly worse than the local ITC. The global ITC combined with PC   is also competitive when the sample size is large. When $N_{\rm graph}\geq 200$, the global ITC combined with PC  outperforms the corresponding CE-based methods. Besides, the CE-based methods that use non-ancestral relations are usually better than the other CE-based methods, as they take the advantage of the correctly learned global graphical structure. Moreover, testing all enumerated effects performs similarly to testing the minimum and maximum absolute effects, rather than outperforms the latter as shown in Section~\ref{sec:sec:true}. 

We next compare the total computational time of different methods. 
As shown in Figure \ref{fig:learn:time}, 
the local ITC and the local versions of the CE-based methods are more efficient than the global ones. For the methods using the same structure learning algorithm, the ITC methods are more efficient than the CE-based methods that take  much more computational time to identify non-ancestral relations, and are slightly more efficient than the other   CE-based methods because  the  structure learning  generally dominates the computational time of these methods.

We also compare the time spend in identifying types of causal relations (excluding the time for structure learning from the total time), and the results are similar to those shown in Figure~\ref{fig:true:zoom}. Using non-ancestral information in CE-based methods usually makes them 40-80 times slower than the local ITC,
while the other CE-based methods without using non-ancestral relations are almost $2$ times slower than the local ITC.

  The experiments of the CE-based methods using the optimal IDA and the hybrid method can be found in  \ref{app:app:opt} and  in~\ref{app:app:hybrid} respectively. Briefly, these experiments show that the CE-based methods with the optimal IDA are usually better than the methods with the original IDA, but do not outperform the global ITC, and the hybrid method  is slightly better than the non-hybrid CE-based methods that use non-ancestral relations.

\subsection{An Application to the DREAM4 Data Sets}\label{sec:sec:d4}


  In this section, we apply our method to the synthetic gene expression data sets from the
DREAM4 \emph{in silico} challenge, to show the potential of our method for supporting causal inference. A detailed description of the data sets can be found at \texttt{https://dreamchallenges.org/dream-4-in-silico-\\network-challenge/}. In this study, we focus on $5$ data sets provided by the DREAM4 challenge, each of which contains a gene regulatory network (possibly cyclic) with 100 genes, observational gene expression data with 310 observations and interventional gene expression data. 
The used $5$ data sets, including the true  network structures, can be obtained from the \texttt{R-package DREAM4}.\footnote{The $5$ data sets are named by ``dream4\_100\_01" to ``dream4\_100\_05" in the \texttt{R-package DREAM4}.}
We normalized each data set such that each gene has a sample mean 0 and a sample variance 1. The marginal distributions of the variables are approximately Gaussian. Following ~\citet{maathuis2010nature},  we assume that the multivariate Gaussianity holds for all variables.

To evaluate the ``true" relationship for each pair of $X$ and $Y$, we first estimate the causal effect of $X$ on $Y$ using the back-door adjustment, based on the observational data and the true network structure. Then, we use the t-test to decide the significance of the estimated causal effect at  the significance level $\alpha=0.001$.    All pairs of $X$ and $Y$ whose corresponding p-values of the t-tests are less than or equal to $\alpha$ are regarded as ``true" causal pairs and constitute the target set.

  Assuming that the true gene network is unknown,  we next use
the following three   methods  to  identify the type of causal relation for each treatment-target pair $(X, Y)$.

\begin{itemize}
	\item  {Method 1.} Using the PC algorithm to estimate a CPDAG first, and then calling the global ITC (Algorithm~\ref{algo:global}) to identify the type of causal relation for each $(X, Y)$.
	\item Method 2. Using the PC algorithm to estimate a CPDAG first, and then calling the local ITC (Algorithm~\ref{algo:local}) to identify the type of causal relation for each $(X, Y)$. The required local structures, i.e. $pa(X, {\cal G}^*)$ and the induced subgraph over $sib(X, {\cal G}^*)$, are read from the learned CPDAG.
	\item Method 3. Using the variant of MB-by-MB to estimate $pa(X, {\cal G}^*)$ and the induced subgraph over $sib(X, {\cal G}^*)$ for each $X$, and then identifying the type of causal relation for each $(X, Y)$ by using the local ITC (Algorithm~\ref{algo:local}).
\end{itemize}


   For each variable pair $(X, Y)$,
   we further use IDA to estimate all possible causal effects of $X$ on $Y$. To build a sequence of variable pairs based on the magnitude of the causal effects, we first rank (in descending order) the pairs of treatment and target whose corresponding causal relations are definite causal according to their minimum absolute effects. Then, we rank (in descending order) the pairs whose corresponding causal relations are definite non-causal and possible causal according to their maximum absolute effects.  Finally, we append the ordered sequence of definite non-causal and possible causal relation pairs to the ordered sequence of definite causal relation pairs, and select top $q$ pairs as the predicted pairs.

    Note that compared to the work of~\citet{maathuis2010nature},  this method  has two differences. First, we rank   definite non-causal and possible causal pairs by their maximum absolute effects, while \citet{maathuis2010nature} rank all pairs by their minimum absolute effects. This is because that  the minimum absolute effect of a definite non-causal or possible causal  pair   should be zero, while the maximum absolute effect is more informative since it measures the upper bound on the true causal effect. Second, we rank  the definite causal pairs before the other pairs regardless of their  possible estimated causal effects.

We compare Methods 1 and 2 to the IDA algorithm combined with PC, and compare Method 3 to the IDA algorithm combined with the variant of MB-by-MB. All significance levels used in these methods are set to be 0.001,  aligned with those used in the simulation
studies. For each method, we compare the predicted pairs to the target pairs for different $q$'s and compute the area under the receiver operating characteristic curve (AUC). Tables~\ref{tab:d4-0.001} shows the results. For each data set and each method, we also perform DeLong's test to test whether the AUC of the method is different from the AUC of the IDA algorithm~\citep{delong1988}, using the \texttt{R function roc.test} implemented in \texttt{R-package pROC}. The null hypothesis is that the difference in AUC is equal to 0. The p-values of the tests are reported in parentheses.



\begin{table}[!t]
	\centering
	\scalebox{0.72}{%
		\begin{tabular}{@{}lccccc@{}}
			\toprule
			Methods                                 & dream4\_100\_01 & dream4\_100\_02 & dream4\_100\_03 & dream4\_100\_04 & dream4\_100\_05 \\ \midrule
			\multirow{2}{*}{PC + IDA}               & \multirow{2}{*}{0.6578} & \multirow{2}{*}{0.6870} & \multirow{2}{*}{0.6894} & \multirow{2}{*}{0.6921} & \multirow{2}{*}{0.6823} \\
			&                         &                         &                         &                         &                         \\
			\multirow{2}{*}{Method 1 + IDA}    & 0.6439               & \textbf{0.6905}      & \textbf{0.6930}      & 0.6808               & \textbf{0.6898}      \\
			& (0.0000)             & (0.2809)             & (0.1049)             & (0.0000)             & (0.0050)             \\
			\multirow{2}{*}{Method 2 + IDA}    & 0.6481               & \textbf{0.6972}      & \textbf{0.6956}      & \textbf{0.6945}      & \textbf{0.6916}      \\
			& (0.0000)             & (0.0000)             & (0.0000)             & (0.0151)             & (0.0000)             \\ \midrule
			\multirow{2}{*}{local + IDA}                             & \multirow{2}{*}{0.6624}               & \multirow{2}{*}{0.6354}               & \multirow{2}{*}{0.7000}               & \multirow{2}{*}{0.6750}               & \multirow{2}{*}{0.6696}               \\
			&              &              &              &              &              \\
			\multirow{2}{*}{Method 3 + IDA} & \textbf{0.6672}      & 0.6349               & \textbf{0.7019}      & \textbf{0.6769}      & \textbf{0.6776}      \\
			& (0.0002)             & (0.2000)             & (0.0890)             & (0.0203)             & (0.0000)             \\
			\bottomrule
		\end{tabular}%
	}
	\caption{ AUC of different methods on DREAM4 data sets. The p-values of DeLong's tests are reported in parentheses, which test whether the AUC of a proposed method is significantly different from the AUC of the ``PC + IDA" or ``local + IDA" algorithm.}
	\label{tab:d4-0.001}
\end{table}

  Table~\ref{tab:d4-0.001} displays that  the modified versions of IDA with Methods 1 to 3 outperform the original IDA in 3, 4 and 4 data sets, respectively, and 9 of them are significant at the   level $0.1$. As a result, our proposed methods,  especially the local ITC (Methods 2 and 3), can improve the performance of the IDA algorithm when predicting the magnitude of a causal effect.



We remark that the above results on the DREAM4 data sets are proof-of-concept and show that identifying types of causal relations do have the potential to support causal inference. Apart from the listed Methods 1, 2 and 3, the practitioners may develop their own specific modifications. Of course, as discussed in~\citet{maathuis2010nature}, great care should be taken in real applications when the underlying assumptions, such as the multivariate Gaussianity and the faithfulness, are violated. Nevertheless, we hope that the example given in this section could motivate more studies on the use of the local and global ITC in observational studies.

\section{Concluding Remarks}\label{sec: conclusion}

In this paper, we present a local method for identifying types of causal relations without evaluating causal effects and learning a global causal structure. A sufficient and necessary graphical condition is provided to check  the existence of a causal path from a treatment to a target  based on a CPDAG. We also study the graphical properties of each type of causal relation. Inspired by these properties, we further propose a local identification criterion for each type of causal relation, which depends only on the induced subgraph of the true CPDAG over the adjacent variables of the treatment as well as some queries about d-separation relations. The local criteria naturally lead to a local learning algorithm for identifying types of causal relations if one assumes that the faithfulness condition holds. Experimental studies empirically prove that the proposed local algorithm performs well.

Our work introduces the local characterizations of types of causal relations, which are helpful for understanding causal relations hidden behind observational data. Except for the theoretical contributions, our results have many potential applications as well. Firstly, as mentioned in the introduction, some real-world problems, such as fault analysis in telecommunication networks and online product recommendation, qualitative analysis is enough for making decisions. 
Secondly, even in quantitative analysis,  when the causal effect is not uniquely identifiable due to Markov equivalence, we may also use the proposed methods to check whether the bounds on a causal effect cover zero. For example, as shown in the experiments, our methods can be used  to modify the current IDA-type algorithms to predict which interventions are likely to have a strong effect, as mentioned by~\citet{maathuis2010nature}. Thirdly, the proposed local method can be combined with the IDA algorithm to reduce the computational costs. For instance, if a treatment is a non-cause of a target, then without any computation we can conclude that all possible effects are zeros \citep{maathuis2009estimating}. Compared to the existing global method that depends on the global structure of an input CPDAG, the proposed local methods are  more effective and efficient, especially when  we are only interested in the causal effect of one treatment on one target, not the  causal effects of all treatments on all targets.
{Finally, \citet{Shi2021mediation} provided a method to find all ``mediators" lying on at least one directed path from a given treatment to a given target, assuming that the underlying DAG is identifiable from data. When the underlying DAG is not identifiable but a CPDAG is identifiable, our proposed methods are potentially useful for finding ``definite mediators",  which are not only the definite effects of the treatment but also the definite causes of the target. These variables must lie on at least one directed path from the treatment to the target, no matter which equivalent DAG is the true one.}

Our results can be easily extended to interventional essential graphs \citep{he2008active,hauser2012characterization}, which can be used to represent   Markov equivalence classes where some variables are intervened.  Basically,   interventional essential graphs  are also chain graphs and can be learned from the mixture of observational and interventional data. Extending our proposed concepts, theorems, and algorithms to interventional essential graphs is straightforward. A possible future work is to extend the global characterization for definite causal relations to maximal PDAGs. Maximal PDAGs are generalizations of CPDAGs, and have been frequently used for representing causal background knowledge \citep{perkovic2017interpreting, fang2020bgida, Perkovic2020mpdag,Witte2020efficient, Guo2020minimal}. Another interesting direction is to take hidden variables and selection biases into account. For example, one may extend the results to partially ancestral graphs~\citep{richardson2002ancestral,ali2012towards,zhang2008completeness}.



\section*{Acknowledgements}
We would like to thank the  editor and the three referees for their helpful comments and suggestions that greatly improved the previous versions of this paper. This work was supported by National Key R\&D Program of China (2018YFB1004300), NSFC (11671020,11971040,11771028,12071015).

\appendix

\section{Graph Terminology}\label{app:graph}

A  graph $\mathcal{G}$ is defined as  a vertex set (or node set) $\textbf{V}$   and an edge set $\textbf{E}$. A graph  is \emph{directed} (\emph{undirected}, \emph{partially directed}) if all edges in the graph are directed (undirected, a mixture of directed and undirected). The \emph{skeleton} of a graph $\cal G$ is an undirected graph resulted from turning every directed edge in $\cal G$ into an undirected edge. Given a subset $\textbf{V}'$ of $\textbf{V}$, the \emph{induced subgraph} of $\mathcal{G}$ over $\textbf{V}'$ is defined as ${\cal G}'=(\textbf{V}', \textbf{E}')$ where $\textbf{E}'\subset \textbf{E}$ contains only edges between vertices in $\textbf{V}'$. If a directed edge $X_i\rightarrow X_j$ occurs in $\cal G$, we call $X_i$  a \emph{parent} of $X_j$ and $X_j$  a \emph{child} of $X_i$. Two distinct vertices $X_i$ and $X_j$ are \emph{siblings} of each other if the undirected edge $X_i - X_j$ appears  in $\cal G$. If for any $\textbf{V}'\subset \textbf{V}$, there exist $X'\in \textbf{V}'$ and $X\in \textbf{V}\setminus \textbf{V}'$ such that $X$ and $X'$ are adjacent, then the graph is called \emph{connected}, otherwise, it is \emph{disconnected}. Furthermore, if there is an edge between any two vertices, then the graph is called \emph{complete}.

A \emph{path} is a sequence of distinct vertices $(X_{k_1},\cdots,X_{k_j})$ such that $X_{k_i}$ is adjacent to $X_{k_{i+1}}$. $X_{k_1}$ and $X_{k_j}$ are \emph{endpoints} of the path, while other vertices on the path are \emph{intermediate vertices (nodes)}. The length of a path is the number of vertices on the path minus one.  A path is called \emph{partially directed} from $X_{k_1}$ to $X_{k_j}$ if $X_{k_i} \leftarrow X_{k_{i+1}}$ does not occur in $\cal {G}$ for any $i=1, \ldots, j-1$. A partially directed path is \emph{directed} (\emph{undirected})  if all edges on the path are  directed (undirected). A \emph{cycle} is a path from a vertex to itself. A partially directed (directed, undirected) cycle can be defined similarly.   We note that both directed paths (cycles) and undirected paths (cycles) are partially directed.  A vertex $X_i$ is an \emph{ancestor} of $X_j$ and $X_j$ is a \emph{descendant} of $X_i$  if there is a directed path from $X_i$ to $X_j$ or $X_i = X_j$.
A \emph{chord} of a path (cycle) is any edge joining two nonconsecutive vertices on the path (cycle). A path (cycle) without any chord is called \emph{chordless}. Any path with length one is chordless. An undirected graph is \emph{chordal} if it has no chordless cycle with length greater than three.
  Given a chordal graph ${\cal C} = (\textbf{V}, \textbf{E})$, if the induced subgraph of $\cal C$ over $\textbf{V}'\subset \textbf{V}$ is complete, then $\textbf{V}'$ is called a \emph{clique} of $\cal C$. Moreover, if there is no $\textbf{V}''$ such that $\textbf{V}' \subset \textbf{V}''$ and $\textbf{V}''$ is a clique, then $\textbf{V}'$ is called a \emph{maximal clique}. A directed graph is \emph{acyclic} (DAG) if there are no directed cycles.

\section{Causal-Effect-Based Methods: Detailed Algorithms}\label{app:ce}

In Section~\ref{sec:sec:ce} we introduce four modifications of the IDA algorithm. In this section, Algorithms~\ref{algo:ce-based-min/max} and~\ref{algo:ce-based-an}   show the detailed procedures of the second and third modifications, respectively. For simplicity, we assume that the observed variables follow a linear-Gaussian structural equation model. This assumption, together with the causal faithfulness assumption, guarantees that $X$ has a non-zero total causal effect on $Y$ if and only if $X$ has a directed path to $Y$ in the underlying DAG. We note that this claim does not always   hold. For example, consider three binary variables $X$, $Y$ and $Z$, such that $X\to Y$, $Z\to X$ and $Z\to Y$. Assume that $P(Z=0)=0.4$, $P(X=0\mid Z=0)=0.2$, $P(X=0\mid Z=1)=0.4$, $P(Y=0\mid X=0,  Z=0)=0.5$, $P(Y=0\mid X=0,  Z=1)=0.1$, $P(Y=0\mid X=1,  Z=0)=0.2$ and $P(Y=0\mid X=1,  Z=1)=0.3$. It is easy to check that $P(X, Y, Z)$ is faithful to the DAG structure consisting of $X\to Y$, $Z\to X$ and $Z\to Y$. However, using the back-door adjustment, one can calculate that $P(Y=1\mid do(X=0))=P(Y=1\mid do(X=1))=0.26$.

\begin{algorithm}[!t]
	\caption{IDA + significance tests for the minimum and maximum absolute effects.}
	\label{algo:ce-based-min/max}
	\begin{algorithmic}[1]
		\REQUIRE
		A treatment $X$, a target $Y$, a CPDAG $\mathcal{G}^*$ over a vertex set $\bf V$, and a significance level $\alpha$.
		\ENSURE
		The type of causal relation between X and Y.
		
		\STATE{set $\Theta= [\;]$,}
		\FOR{each $\mathbf{Q}\subset sib(X, {\cal G}^*)$ such that    orienting $\mathbf{Q}\to X$ and $X\rightarrow sib(X, \mathcal{G}^{*})\setminus \textbf{Q}$ does not introduce any v-structure collided on $X$, let $\mathbf{S}=\mathbf{Q}\cup pa(X, {\cal G}^*)$ and}
		\STATE{estimate the causal effect $\theta$ of $X$ on $Y$ by adjusting for $\mathbf{S}$, and add the causal effect to $\Theta$,}
		\ENDFOR

		\STATE{let $\theta_{\rm max} = \max\{\theta\mid \theta\in\Theta\}$,  test the null hypothesis $\theta_{\rm max}=0$ and denote the corresponding  p-value by $p_{\rm max}$,}
		\STATE{let $\theta_{\rm min} = \min\{\theta\mid \theta\in\Theta\}$,  test the null hypothesis $\theta_{\rm min}=0$ and denote the corresponding p-value by $p_{\rm min}$,}
		
		
		\IF {$p_{\rm min}\leq \alpha$}
		\RETURN {$X$ is a \textbf{definite cause} of $Y$,}
		\ENDIF
		
		\IF {$p_{\rm max}>\alpha$}
		\RETURN {$X$ is a \textbf{definite non-cause}  of $Y$,}
		\ENDIF
		
		\RETURN {$X$ is a \textbf{possible cause} of $Y$.}

	\end{algorithmic}
\end{algorithm}

Algorithm~\ref{algo:ce-based-min/max} shows the detailed procedure of the second modification, which is IDA + significance tests for the minimum and maximum absolute effects. The key step is to test the significance of the minimum and maximum absolute effects. In the linear-Gaussian case, $\Theta$ is a collection of regression coefficients of $X$, each of which corresponds to a linear regression of $Y$ on $X$ and some adjustment set. Since the adjustment sets in different regression models may overlap, it is difficult to derive the (asymptotic) distributions of $\theta_{\rm min}$ and $\theta_{\rm max}$ under the null hypothesis. Therefore, in our implementation, we only test $\theta_{\rm min}$ and $\theta_{\rm max}$ in their own regression models.

\begin{algorithm}[!t]
	\caption{IDA + utilizing non-ancestral relations + significance tests for all estimated effects}
	\label{algo:ce-based-an}
	\begin{algorithmic}[1]
		\REQUIRE
		A treatment $X$, a target $Y$, a CPDAG $\mathcal{G}^*$ over a vertex set $\bf V$, and a significance level $\alpha$.
		\ENSURE
		The type of causal relation between X and Y.
		
		\STATE{set $\Theta= [\;]$,}
		\FOR{each $\mathbf{Q}\subset sib(X, {\cal G}^*)$ such that    orienting $\mathbf{Q}\to X$ and $X\rightarrow sib(X, \mathcal{G}^{*})\setminus \textbf{Q}$ does not introduce any v-structure collided on $X$, let $\mathbf{S}=\mathbf{Q}\cup pa(X, {\cal G}^*)$ and}
		
		\STATE{Orient $\mathbf{Q}\to X$ and $X\rightarrow sib(X, \mathcal{G}^{*})\setminus \textbf{Q}$ in ${\cal G}^*$, and complete the orientations with Meek's rules~\citep{meek1995causal}. Denote the resulting graph by $\cal H$,}
		
		\STATE{Check whether $X$ is a non-ancestor of $Y$ in $\cal H$~\citep{perkovic2017interpreting} ,}
		
		\IF {$X$ is a non-ancestor of $Y$ in $\cal H$,}
		\STATE{set $\theta=0$ and $p=1$,}
		\ELSE
		\STATE{estimate the causal effect $\theta$ of $X$ on $Y$ by adjusting for $\mathbf{S}$,}
		\STATE{test the null hypothesis   $\theta=0$, and compute the p-value $p$,}
		\ENDIF
		
		\STATE{add $\theta$ to $\Theta$, and add $p$ to ${\rm P}_{\rm val}$,}
		
		\ENDFOR
		
		
		\IF {every p-value in ${\rm P}_{\rm val}$ is less than or equal to $\alpha$}
		\RETURN {$X$ is a \textbf{definite cause} of $Y$,}
		\ENDIF
		
		\IF {every p-value in ${\rm P}_{\rm val}$ is greater than $\alpha$}
		\RETURN {$X$ is a \textbf{definite non-cause}  of $Y$,}
		\ENDIF
		
		\RETURN {$X$ is a \textbf{possible cause} of $Y$.}

	\end{algorithmic}
\end{algorithm}

Algorithm~\ref{algo:ce-based-an} shows the detailed procedure of the third modification, which is IDA + utilizing non-ancestral relations + significance tests for all estimated effects. The key steps are lines 3 and 4. In line 3, Algorithm~\ref{algo:ce-based-an} calls Meek's rules to complete the orientations $\mathbf{Q}\to X$ and $X\rightarrow sib(X, \mathcal{G}^{*})\setminus \textbf{Q}$ in ${\cal G}^*$. The resulting graph, which is denoted by $\cal H$, is a maximally partially directed acyclic graph (MPDAG) containing both directed and undirected edges. \citet[][Lemma 3.2]{perkovic2017interpreting} proved that $X$ is not a b-possible ancestor of $Y$ in an MPDAG if there is no b-possibly causal path from $X$ to $Y$, where a path from $X$ to $Y$ is b-possibly causal if none of the edge and chord on the path points towards $X$.  Consequently, Algorithm~\ref{algo:ce-based-an} graphically checks whether $X$ is a b-possible ancestor of $Y$ in $\cal H$. If $X$ is not a b-possible ancestor of $Y$ in $\cal H$, then $X$ is a non-ancestor of $Y$ in any DAG in $[\cal H]$, which is the restricted Markov equivalence class represented by $\cal H$~\citep{perkovic2017interpreting, fang2020bgida}, and thus Algorithm~\ref{algo:ce-based-an} sets $\theta=0$ and $p=1$. More information about MPDAGs can be found in~\cite{meek1995causal} and \cite{ perkovic2017interpreting}.

\section{Detailed Proofs}
\label{app:proofs}

The proofs of  lemmas, theorems and corollaries in the main text of this paper will be presented in this section. Before that,  we first introduce some prerequisite concepts and results.

Let $\pi = (v_0, v_1, ..., v_k)$ denote a path with length $k$. The subpath $\pi(v_i, v_j)$ of $\pi$, with $j>i$, is the path $(v_i, v_{i+1}, ..., v_{j-1}, v_j)$. If $k\geq2$, we say three consecutive vertices $v_i$, $v_{i+1}$ and $v_{i+2}$ form a triangle on $\pi$ if $v_i$ is adjacent to $v_{i+2}$. $\pi$ is called triangle-free if it does not contain any triangle. For a path in a chordal graph, we have the following result.

\begin{lemma}\label{lem:app:chordless}
	In any chordal graph, a path is chordless if and only if it is triangle-free.
\end{lemma}

\begin{proof}
	Let $\pi = (v_0, v_1, ..., v_k)$ denote a path with length $k\geq2$, If $\pi$ is chordless, then it is obviously triangle-free. Suppose $\pi$ is not chordless, then we can choose a chord $v_i-v_j$ such that the subpath $\pi(v_i, v_j)$ has no chord except for $v_i-v_j$. If $j=i+2$, then $v_i$, $v_{i+1}$ and $v_j$ form a triangle. If $j>i+2$, then $\pi(v_i, v_j)$ and $v_i-v_j$ form a cycle with length greater than 3. However, since the graph is chordal, we must have a chord $v_k-v_l$ with $i\leq k, l\leq j$ and  $l\geq k+2$ and $l-k<j-i$. This is contrary to our assumption.
\end{proof}

Lemma \ref{lem:app:chordless} is useful for finding chordless path, since checking whether a path is triangle-free is much easier. The following is another useful result for chordal graphs.


\begin{lemma}\label{lem:app:cycle}
	Let $\rho$ be a cycle with length greater than 3 in a given chordal graph, and $X$ be a vertex on $\rho$. If the two vertices adjacent to $X$ on $\rho$ are not adjacent to each other, then $\rho$ has a chord where $X$ is an endpoint.
\end{lemma}
\begin{proof}
	Let $v_1$ and $v_2$ be two vertices adjacent to $X$ on $\rho$. Suppose that $\rho$ does not have a chord where $X$ is an endpoint. Since $\rho$ has length greater than 3, $\rho$ must have a chord. Clearly, any chord of $\rho$ separates $\rho$ into two sub-cycles. By assumption, it is easy to check that at least one sub-cycle contains $X$, $v_1$ and $v_2$. If this sub-cycle still has a chord, then we can construct another cycle containing $X$, $v_1$ and $v_2$ but with shorter length. Finally, we will have a cycle containing $X$, $v_1$ and $v_2$ without any chord. Since $v_1$ and $v_2$ are not adjacent, the length of this cycle must be greater than 3, which is contradicted to the definition of chordal graph.
\end{proof}

A chordal graph $\cal C$ can be turned into a directed graph by orienting its edges. If the resulting directed graph is a DAG without v-structure, then these orientations form a v-structure-free acyclic orientation of $\cal C$~\citep{arXiv170509717b}. Any v-structure-free acyclic orientation of a connected chordal graph has a unique source, that is, a vertex which has no parent. Conversely, any vertex in a connected chordal graph can be the unique source in some v-structure-free acyclic orientation~\citep{blair1993chordal, arXiv170509717b}. Recall that the undirected subgraph of a CPDAG is the union of disjoint connected chordal graphs called chain components \citep{andersson1997characterization}. \citet{maathuis2009estimating} argued that any v-structure-free acyclic orientation of the edges in $\mathcal{G}^*_u$ corresponds to a DAG in the equivalence class represented by $\mathcal{G}^*$, and such an orientation can be considered separately for each of the disjoint chordal graphs (or chain components). Moreover, \citet{maathuis2009estimating} proved that,

\begin{lemma}\label{lem:app:valid-local}
	{\rm \textbf{\citep[Lemma~3.1]{maathuis2009estimating}}}
	Let ${\cal G}^*$ be a CPDAG, $X$ be a vertex of ${\cal G}^*$, and $\mathbf{S}\subset ne(X, {\cal G}^*)$. Then there is a DAG ${\cal G}\in [{\cal G}^*]$ such that $pa(X, {\cal G})=pa(X, {\cal G}^*)\cup \mathbf{S}$ if and only if orienting $S \rightarrow X$ and $X \rightarrow D$ for every $S\in\mathbf{S}$ and $D\in sib(X, {\cal G}^*)\setminus \mathbf{S}$ in ${\cal G}^*$  does not introduce any new v-structure.
\end{lemma}

\citet[Lemma~1]{meek1995causal} proved that if $Y\in pa(X, {\cal G}^*)$, then $Y\in pa(X', {\cal G}^*)$ for every $X'\in ne(X, {\cal G}^*)$. From this result we can prove that the condition in Lemma \ref{lem:app:valid-local} holds if and only if $\mathbf{S}$ is a clique. As we will see, Lemma \ref{lem:app:valid-local} plays a key role in proving the main results of this paper, as it provides a simple and local criterion for checking whether a subset of $X$'s siblings can be $X$'s parents in some equivalent DAGs.

Let $\pi$ denote a path. A subsequence of $\pi$ is obtained by deleting some vertices from $\pi$ without changing the order of the remaining vertices. The final prerequisite result is about the relation between directed paths and partially directed paths.

\begin{lemma}\label{lem:another-char}
	There is a directed path from $X$ to $Y$  in $\mathcal{G}^*$ if and only if there is a partially directed path from $X$ to $Y$  in $\mathcal{G}^*$  on which the node adjacent to $X$ is a child of $X$.
\end{lemma}
\begin{proof}
	The necessity is trivial. For sufficiency, let $\pi=(X, v, ..., Y)$ be the partially directed path from $X$ to $Y$ in $\mathcal{G}^*$ such that $X\to v$. Assume that $w$ is the first vertex from the side of $Y$ which is adjacent to $X$, then we have $X\to w$. Now consider $\pi(w, Y)$. As $\pi(w, Y)$ is also partially directed, by \citet[Lemma~3.6]{perkovic2017interpreting}, there is a subsequence $\pi^*$ of $\pi(w, Y)$ forms a chordless partially directed path from $X$ to $Y$ in $\mathcal{G}^*$. Let $\pi^{**}$ denote the path by concatenating $X\to w$ and $\pi^*$, then $\pi^{**}$ is a  partially directed path from $X$ to $Y$ on which the node adjacent to $X$ is a child of $X$. By construction, $X$ is not adjacent to any vertex on $\pi^{**}$ except for $w$. Thus, by \citet[Lemma~7.2]{maathuis2015generalized}, $\pi^{**}$ is a directed path.
\end{proof}

In the following Appendices C.1 to C.13, we will present the detailed proofs of the main results provided in the main text, with the help of the aforementioned concepts and lemmas.

\subsection{Proof of Lemma \ref{lem:child_critical}}\label{app:proof:1}
\begin{proof}
  Given a CPDAG ${\cal G}^*$, for any DAG ${\cal G}\in[{\cal G}^*]$, \citet[Lemma~2]{fang2020bgida} showed that a variable $X$ is not a cause of another variable $Y$ in $\cal G$ if and only if the critical set of $X$ with respect to $Y$ in ${\cal G}^*$, which is denoted by $\mathbf{C}$, is a subset of $pa(X, {\cal G})$. Consequently, $X$ is a cause of $Y$ in $\cal G$ if and only if $\mathbf{C}$ is not a subset of $pa(X, {\cal G})$.  That is, some vertex in $\mathbf{C}$ must be a child of $X$ in $\cal G$. The desired result comes from the definition of definite cause.
\end{proof}

\subsection{Proof of Lemma \ref{lem:critical_has_a_child}}

\begin{proof}
  We first show the necessity. By the definition, $\mathbf{C}\subseteq sib(X, \mathcal{G}^*)\cup ch(X, \mathcal{G}^*)$. Let ${\cal G} \in [\mathcal{G}^*]$ be an arbitrary DAG. If $\mathbf{C}\cap ch(X, {\cal G})=\emptyset$ and $\mathbf{C} \neq \emptyset$, then $\mathbf{C} \subseteq pa(X, {\cal G})$, and thus we have $\mathbf{C} \subseteq sib(X, {\cal G}^*)$. \citet[Lemma~3]{maathuis2009estimating} proved that a non-empty subset of $sib(X, {\cal G}^*)$ can be a part of $X$'s parent set in some equivalent DAG if and only if the subset induces a complete subgraph. Therefore, $\mathbf{C}$ induces a complete subgraph of $\mathcal{G}^*$. This completes the proof of the necessity. We next prove the sufficiency. If $\mathbf{C} = \emptyset$, then it is clear that $\mathbf{C}\cap ch(X, {\cal G})=\emptyset$ for some ${\cal G} \in [\mathcal{G}^*]$. Now assume that $\mathbf{C} \neq \emptyset$ and $\mathbf{C}$ induces a complete subgraph of $\mathcal{G}^*$ and $\mathbf{C}\cap ch(X, \mathcal{G}^*)= \emptyset$. As $\mathbf{C}\subseteq sib(X, \mathcal{G}^*)\cup ch(X, \mathcal{G}^*)$, we have $\mathbf{C} \subseteq sib(X, {\cal G}^*)$. Again, by \citet[Lemma~3]{maathuis2009estimating}, there is a DAG $\mathcal{G}$ in $[\mathcal{G}^*]$ such that $\mathbf{C} \subseteq pa(X, \mathcal{G})$. Therefore, $\mathbf{C}\cap ch(X, {\cal G})=\emptyset$.
\end{proof}

\subsection{Proof of Theorem \ref{thm:graphical_definite_cause}}
\begin{proof}
  Theorem \ref{thm:graphical_definite_cause} follows from Lemmas \ref{lem:child_critical} and \ref{lem:critical_has_a_child} directly.
\end{proof}

\subsection{Proof of Proposition \ref{prop:definite-cause}}
\begin{proof}
Denote the CPDAG containing $X$ and $Y$  by ${\cal G}^*$. It suffices to show that, if $X$ and $Y$ are in the same chain component, then there exists a DAG in $[{\cal G}^*]$ in which $Y$ is an ancestor of $X$. By Lemma~\ref{lem:app:valid-local}, there exists a DAG $\cal G$ in $[{\cal G}^*]$ such that $pa(Y, {\cal G})=pa(Y, {\cal G}^*)$ and $ch(Y, {\cal G})=ch(Y, {\cal G}^*)\cup sib(Y, {\cal G}^*)$. Let $\pi=(Y,v_1,...,X)$ be the shortest path from $Y$ to $X$. It is clear that $\pi$ has no chord. Moreover, the corresponding path of $\pi$ in ${\cal G}^*$ is undirected as $X$ and $Y$ are in the same chain component. On the other hand, $Y\to v_1$ is in ${\cal G}$ by our construction. Hence, according to \citet[Lemma~B.1]{perkovic2017interpreting}, $\pi$ is a directed path.
\end{proof}

\subsection{Proof of Proposition \ref{prop:chain-comp}}
\begin{proof}
	According to the definition of partially directed path, an undirected path is also partially directed, hence if $X$ and $Y$ are in the same chain component, they are possible causes of each other by Theorem \ref{thm:non-cause-no-bk} and Proposition \ref{prop:definite-cause}. Conversely, if $X$ and $Y$ are possible causes of each other, then by Theorem \ref{thm:non-cause-no-bk}, there is a partially directed path from $X$ to $Y$ as well as a partially directed path from $Y$ to $X$. Clearly, neither of these two paths contains a directed edge, otherwise,   a partially directed cycle containing directed edges  would occur. Therefore, $X$ and $Y$ are connected by an undirected path, which means they are in the same chain component.
\end{proof}
%

\subsection{Proof of Proposition \ref{prop:nec-implicit}}
\begin{proof}
	Let $Z$ be a vertex in the chain component containing $X$, then every partially directed path between $Z$ and $Y$, if any, must pass through $X$. Since there is a v-structure-free orientation of the chain component whose unique source is $X$, there is a DAG in the Markov equivalence class represented by $\mathcal{G}^*$ such that none of the vertex in the chain component is an ancestor of $Y$ except $X$ .
\end{proof}

\subsection{Proof of Proposition \ref{prop:equivalent-critical-set}}
\begin{proof}
   If $X$ and $Y$ are in the same chain component, then $\mathbf{Z}=\{Y\}$ and the equation trivially holds. Suppose that $X$ and $Y$ are not in the same chain component. We first prove that $\mathbf{C}_{XY}\subseteq \cup_{Z\in \mathbf{Z}}\mathbf{C}_{XZ}$. Without loss of generality, we can assume that $\mathbf{C}_{XY}\neq\emptyset$. By the definition of critical set, for any $C\in \mathbf{C}_{XY}$, there is a chordless partially directed path $\rho$ from $X$ to $Y$ on which $C$ is adjacent $X$. Since  $X$ and $Y$ are not in the same chain component, $\rho$ must contain a directed edge. Let $Z$ be the vertex on $\rho$ such that $\rho(Z, Y)$ starts with a directed edge and $Z$ is in the chain component containing $X$. By \citet[Lemma~7.2]{maathuis2015generalized} or \citet[Lemma~B.1]{perkovic2017interpreting}, $\rho(Z, Y)$ is a directed path. Therefore, $Z$ is an explicit cause of $Y$. Since $X$ is not an explicit cause of $Y$, we have $Z\neq X$, and thus $\rho(X, Z)$ is a chordless undirected path. This means $C\in \mathbf{C}_{XZ}$. As $C\in\mathbf{C}_{XY}$ is arbitrary, we have $\mathbf{C}_{XY}\subseteq\cup_{Z\in \mathbf{Z}}\mathbf{C}_{XZ}$. Conversely, for any $Z\in \mathbf{Z}$ and $C \in \mathbf{C}_{XZ}$, there is a chordless undirected path $\pi_1$ from $X$ to $Z$ on which $C$ is adjacent $X$. Let $\pi_2$ be the shortest directed path from $Z$ to $Y$. As $X$ and $Y$ are not in the same chain component, $Z\neq Y$. Hence, concatenating $\pi_1$ and $\pi_2$ results a partially directed path from $X$ to $Y$ with length greater than 1. Denote such a path by $\pi$. If $\pi$ is chordless, then we have $C\in\mathbf{C}_{XY}$. If this is not the case, then $\pi$ must have a chord connecting one vertex $v_1$ on $\pi_1$ and another vertex $v_2$ on $\pi_2$. Clearly, the edge between $v_1$ and $v_2$ should be directed, and the direction is $v_1\to v_2$. Since $X$ is not an explicit cause of $Y$, it holds that $v_1\neq X$. With out loss of generality, we assume that $v_1$ is the first vertex from $X$'s side who are adjacent to some $v_2$ on $\pi_2$, then concatenating $\pi(X, v_1)$, $v_1\to v_2$ and $\pi(v_2, Y)$ results another partially directed path $\pi'$ which is shorter than $\pi$. It is easy to verify that $\pi'$ is chordless, and $C$ is still adjacent to $X$ on $\pi'$. Therefore, $C\in\mathbf{C}_{XY}$, and consequently we have $\cup_{Z\in \mathbf{Z}}\mathbf{C}_{XZ}\subseteq \mathbf{C}_{XY}$. This completes the proof of Proposition \ref{prop:equivalent-critical-set}.
\end{proof}

\subsection{Proof of Theorem \ref{thm:non-cause-no-bk}}

\begin{proof}
	Suppose $X$ is a definite non-cause of $Y$, then for every DAG $\cal G$ in the Markov equivalence class represented by $\mathcal{G}^*$, $Y$ is a non-descendant of $X$. Since Lemma \ref{lem:app:valid-local} indicates that there is a DAG $\cal G$ such that $pa(X, \mathcal{G}) = pa(X, \mathcal{G}^*)$ and $ch(X, \mathcal{G}) = adj(X,\mathcal{G}^*)\setminus pa(X, \mathcal{G}^*)$, we have $X \indep Y \mid pa(X, \mathcal{G}^*)$ by local Markov property. On the other hand, if $X$ is a definite cause or a possible cause of $Y$, then by definition there is a  DAG $\cal G$ in the Markov equivalence class represented by $\mathcal{G}^*$ in which $X$ is an ancestor of $Y$. Assume that $\pi$ is a directed path from $X$ to $Y$ in $\cal G$. Since every vertex on $\pi$ is a non-collider and none of the vertices on $\pi$ is in $pa(X, \mathcal{G}^*)$, $X \nindep Y \mid pa(X, \mathcal{G}^*)$.
\end{proof}

%

\subsection{Proof of Theorem \ref{thm:explicit-cause-no-bk}}\label{proof:explict}
\begin{proof}
	If $X$ is an explicit cause of $Y$, then there is a directed path $\pi$ from $X$ to $Y$ in $\mathcal{G}^*$. Hence, for any DAG $\cal G$ in the Markov equivalence class represented by $\mathcal{G}^*$, $\pi$ is directed in $\mathcal{G}$, which means $\pi$ has no collider in $\mathcal{G}$. However, none of the vertices on $\pi$ is a member of $pa(X, \mathcal{G}^*)$ or $sib(X, \mathcal{G}^*)$, since otherwise, a directed cycle or a   partially directed cycle with directed edges  would occur in $\mathcal{G}^*$. Therefore, $\pi$ is active given $pa(X, \mathcal{G}^*)\cup sib(X, \mathcal{G}^*)$, which means $X \nindep Y \mid pa(X, \mathcal{G}^*)\cup sib(X, \mathcal{G}^*)$. Conversely, suppose $X$ is not an explicit cause of $Y$. In the following, we will prove that $X \indep Y \mid pa(X, \mathcal{G}^*)\cup sib(X, \mathcal{G}^*)$ holds. By Lemma \ref{lem:app:valid-local}, there is a DAG $\cal G$ in the Markov equivalence class represented by $\mathcal{G}^*$ such that $ch(X, \mathcal{G}) = sib(X, \mathcal{G}^*) \cup ch(X, \mathcal{G}^*) $ and $pa(X, \mathcal{G}) = pa(X, \mathcal{G}^*)$. Consider a path $\pi$ from $X$ to $Y$ in $\cal G$.  If the length of $\pi$ is $1$, then the corresponding path of $\pi$ in ${\cal G}^*$ must be $X \leftarrow Y$ or $X - Y$. Thus, $\pi$ is blocked given $pa(X, \mathcal{G}^*)\cup sib(X, \mathcal{G}^*)$. If the length of $\pi$ is greater than $1$, without loss of generality we can assume that $\pi=(X, v_1, ..., v_n, Y)$. If $v_1 \in pa(X, \mathcal{G})$, then $\pi$ is blocked by $pa(X, \mathcal{G}^*) \cup sib(X, \mathcal{G}^*)$ since $v_1$ cannot be a collider on $\pi$. If $v_1 \in ch(X, \mathcal{G}^*)$, then $\pi$ is not directed, since otherwise, the corresponding path in $\mathcal{G}^*$ would be a partially directed path from $X$ to $Y$ where the node adjacent to $X$ is a child of $X$. Therefore, there must be a collider on $\pi$. Let $v_i$ be the collider nearest to $X$. If $v_i \in an(pa(X, \mathcal{G}^*) \cup sib(X, \mathcal{G}^*), \mathcal{G})$,   there exists a partially directed cycle with directed edges  in $\mathcal{G}^*$, which is impossible. Thus, $v_i \notin an(pa(X, \mathcal{G}^*) \cup sib(X, \mathcal{G}^*), \mathcal{G})$, and $\pi$ is blocked by $pa(X, \mathcal{G}^*) \cup sib(X, \mathcal{G}^*)$. Finally, in the case where $v_1 \in sib(X, \mathcal{G}^*)$, if $v_1$ is a non-collider, $\pi$ is clearly blocked by $pa(X, \mathcal{G}^*) \cup sib(X, \mathcal{G}^*)$. If $v_1$ is a collider, then $v_2$ is adjacent to $X$, which means $v_2 \notin ch(X, \mathcal{G}^*)$, since otherwise, both $X\rightarrow v_2\rightarrow v_1-X$ and $X\rightarrow v_2- v_1-X$ are   partially directed cycles with directed edges. This means $v_2\in pa(X, \mathcal{G}^*) \cup sib(X, \mathcal{G}^*)$. Since $v_2$ is a non-collider on $\pi$, $\pi$ is blocked by $pa(X, \mathcal{G}^*) \cup sib(X, \mathcal{G}^*)$. This completes the proof of Theorem \ref{thm:explicit-cause-no-bk}.
\end{proof}


   We note that, the sufficiency of Theorem \ref{thm:explicit-cause-no-bk} can also be proved using the theories of chain graph models. Here we provide a sketch. \cite{andersson1997characterization} proved that every DAG $\cal G$ is LWF (globally) Markov equivalent to the CPDAG ${\cal G}^*$ representing $[\cal G]$. Therefore, a distribution $P$ faithful to $\cal G$ must be LWF globally Markovian to ${\cal G}^*$. Note that, \citet[][Corollary~34]{Sadeghi17faithful} proved that $P$ must satisfy the regularization condition (CI5) in \cite{Frydenberg1990}, thus $P$ should be LWF locally Markovian to ${\cal G}^*$~\citep[][Theorem~3.3]{Frydenberg1990}, which implies that $X \indep_{P} Y \mid pa(X, {\cal G}^*)\cup sib(X, {\cal G}^*)$ for any $Y \in \mathbf{V}\setminus (de(X, {\cal G}^*)\cup pa(X, {\cal G}^*)\cup sib(X, {\cal G}^*))$. Therefore, $X \indep Y \mid pa(X, {\cal G}^*)\cup sib(X, {\cal G}^*)$ for any $Y \in \mathbf{V}\setminus (de(X, {\cal G}^*)\cup pa(X, {\cal G}^*)\cup sib(X, {\cal G}^*))$ due to the faithfulness of $P$. This completes the proof of the sufficiency of Theorem~\ref{thm:explicit-cause-no-bk}.

\subsection{Proof of Theorem \ref{thm:implicit-cause-no-bk}}
\begin{proof}
	Let $\mathbf{C}$ be the critical set of $X$ with respect to $Y$ in $\mathcal{G}^*$. Suppose that $X$ is an implicit cause of $Y$, then by Theorem \ref{thm:explicit-cause-no-bk}, $X \indep Y \mid pa(X, \mathcal{G}^*)\cup sib(X, \mathcal{G}^*)$. For any $\textbf{M}_w \in \cal M$, from Theorem \ref{thm:graphical_definite_cause} we know that $\mathbf{C}\setminus\textbf{M}_w\neq\emptyset$. Therefore, according to Proposition \ref{prop:equivalent-critical-set}, there is a partially directed path from $X$ to $Y$, denoted by $\pi_w=(X-w_1-... -w_t - Z_w\rightarrow...\rightarrow Y)$, such that $X-w_1-... -w_t-Z_w$ is chordless and $w_1\notin \textbf{M}_w$.  Since every partially directed cycle in $\mathcal{G}^*$ is an undirected cycle, none of the vertices on $\pi_w$ is a parent of $X$ in $\mathcal{G}^*$. Moreover, due to the chordless-ness,  if $w_1 \neq Z_w$,  then none of $w_2, ..., w_t, Z_w$ is adjacent to $X$ and thus none of them is in $\textbf{M}_w$. (If $w_1 = Z_w$, then it is clear that $Z_w\notin\textbf{M}_w$.) Since by Lemma \ref{lem:app:valid-local} there is a DAG in the Markov equivalence class represented by $\mathcal{G}^*$ such that $\pi_w$ is directed, $\pi_w$ is active given $pa(X, \mathcal{G}^*)\cup \textbf{M}_w$. Therefore, $X \nindep Y \mid pa(X, \mathcal{G}^*)\cup \textbf{M}$ for any $\textbf{M}\in \mathcal{M}$.   Conversely, $X \indep Y \mid pa(X, \mathcal{G}^*)\cup sib(X, \mathcal{G}^*)$ implies $X$ is not an explicit cause of $Y$, which also means $Y\notin ch(X, \mathcal{G}^*)$. Moreover, $X \nindep Y \mid pa(X, \mathcal{G}^*)\cup \textbf{M}$ for any $\textbf{M}\in \mathcal{M}$ implies $Y\notin pa(X, \mathcal{G}^*)\cup sib(X, \mathcal{G}^*)$. Therefore, $X$ and $Y$ are not adjacent. Suppose that $X$ is not implicit. Since $X$ is not an explicit cause of $Y$, $\mathbf{C}\cap ch(X, {\cal G}^*)=\emptyset$. Thus, by Theorem \ref{thm:graphical_definite_cause}, there exists an $\textbf{M}\in \cal M$ such that $\textbf{C}$ is a subset of $\textbf{M}$. (If $\textbf{C}=\emptyset$, then for any $\textbf{M}\in \cal M$, $\textbf{C}\subset\textbf{M}$.) We will show that $pa(X, \mathcal{G}^*)\cup \textbf{M}$ d-separates $X$ and $Y$. By Lemma \ref{lem:app:valid-local}, there is a DAG $\cal G$ in the Markov equivalence class represented by $\mathcal{G}^*$ such that $ch(X, {\cal G}) = sib(X, \mathcal{G}^*) \cup ch(X, \mathcal{G}^*)\setminus \textbf{M} $ and $pa(X, {\cal G}) = pa(X, \mathcal{G}^*)\cup \textbf{M}$. Let $\pi = (X, v_1, ...,v_n, Y)$ be an arbitrary path connecting $X$ and $Y$ in $\cal G$. The length of $\pi$ should be greater than 1 as $X$ and $Y$ are not adjacent. If $v_1$ is a parent of $X$ in $\cal G$, then clearly $\pi$ is blocked by $pa(X, \mathcal{G}^*)\cup \textbf{M}$, since $v_1$ is a non-collider on $\pi$ and $v_1\in pa(X, \mathcal{G}^*)\cup \textbf{M}$ by the construction of $\cal G$. Now assume that $v_1$ is a child of $X$ in $\mathcal{G}$. If $v_1\in ch(X, \mathcal{G}^*)$, then there must be a collider on $\pi$, since otherwise, the corresponding path of $\pi$ in $\mathcal{G}^*$ is a partially directed path where the node adjacent to $X$ is a child of $X$, which means $X$ is an explicit cause of $Y$ according to Lemma \ref{lem:another-char}. Clearly, the collider nearest to $X$ on $\pi$ is not an ancestor of $pa(X, \mathcal{G}^*)\cup \textbf{M}$. Thus, $\pi$ is blocked by $pa(X, \mathcal{G}^*)\cup \textbf{M}$. For the same reason, if $v_1\in sib(X, \mathcal{G}^*)\setminus \textbf{M}$ and there is a collider on $\pi$, then $\pi$ is blocked by $pa(X, \mathcal{G}^*)\cup \textbf{M}$ due to the fact that the collider nearest to $X$ on $\pi$ can not be an ancestor of $pa(X, \mathcal{G}^*)\cup \textbf{M}$. Finally, if $v_1\in sib(X, \mathcal{G}^*)\setminus \textbf{M}$ and there is no collider on $\pi$, then $\pi$ is directed in $\cal G$, and the corresponding path of $\pi$ in $\mathcal{G}^*$ is partially directed. Let $Z$ be the vertex on $\pi$ such that the subpath $\pi(X, Z)$ is undirected in  $\mathcal{G}^*$ and $Z$ is an explicit cause of $Y$. Obviously, such $Z$ exists, and $Z\neq Y$ or $X$. Since $v_1\notin \textbf{M}$, we have $v_1\notin \textbf{C}$ and thus $\pi(X, Z)$ has a chord. By \citet[Lemma~3.6]{perkovic2017interpreting}, there is a subsequence $\pi^*$ of $\pi(X, Z)$ forms a chordless undirected path from $X$ to $Z$ in $\mathcal{G}^*$. Together with Proposition \ref{prop:equivalent-critical-set}, this result indicates that there is a vertex $w$ on $\pi(X, Z)$ such that $w \in \textbf{C}$. However, by construction, $w \in pa(X, \mathcal{G})$, which makes $\pi(X, w)$ and $w\rightarrow X$ form a directed cycle in $\mathcal{G}$. Thus, $\pi$ must contain a collider. This completes the proof.
\end{proof}

\subsection{Proof of Theorem \ref{thm:algo-local}}\label{app:proof:2}
\begin{proof}
	The proof directly follows from Theorems \ref{thm:graphical_definite_cause} to  \ref{thm:implicit-cause-no-bk}, as well as Propositions \ref{prop:chain-comp} and \ref{prop:equivalent-critical-set}.
\end{proof}

\section{Additional Experimental Results} \label{app:exp}
As a supplement to Section~\ref{sec:simulation}, we present additional experimental results in this section.

\subsection{Frequencies of Different Types of Causal Relations}\label{app:app:freq}

Table~\ref{tab:feq} reports the frequencies of different types of causal relations in all 50- and 100-node randomly sampled positive weight graphs used in our simulations. For instance, the value given in the upper left cell, $0.9536$, is the ratio of   the total number of definite non-causal relations in $5,000$ randomly sampled graphs with $n=50$, $d=1.5$ and positive edge weights, to  the total number of variable pairs ($5000\times 50\times 49$). The frequencies in mixed weight graphs are similar and thus omitted. As expected, a large portion of variable pairs correspond to the definite non-causal relations. Meanwhile, the frequencies of the possible causal relations and definite causal relations are similar to each other.

\begin{table}[!h]
	\centering
	\scalebox{0.8}{%
		\begin{tabular}{@{}crrrrrr@{}}
			\toprule
			\multirow{2}{*}{$d$}    & \multicolumn{3}{c}{$n=50$}                                                                            & \multicolumn{3}{c}{$n=100$}                                                                           \\ \cmidrule(l){2-7}
			& \multicolumn{1}{c}{Def. non-cause} & \multicolumn{1}{c}{Poss. cause} & \multicolumn{1}{c}{Def. cause} & \multicolumn{1}{c}{Def. non-cause} & \multicolumn{1}{c}{Poss. cause} & \multicolumn{1}{c}{Def. cause} \\ \midrule
			1.5                     & 0.9536                             & 0.0314                          & 0.0150                         & 0.9740                             & 0.0192                          & 0.0068                         \\
			2.0                     & 0.9328                             & 0.0402                          & 0.0270                         & 0.9636                             & 0.0242                          & 0.0122                         \\
			2.5                     & 0.9086                             & 0.0530                          & 0.0384                         & 0.9492                             & 0.0292                          & 0.0216                         \\
			3.0                     & 0.8852                             & 0.0596                          & 0.0552                         & 0.9298                             & 0.0372                          & 0.0330                         \\
			3.5                     & 0.8470                             & 0.0646                          & 0.0884                         & 0.9086                             & 0.0436                          & 0.0478                         \\
			4.0 & 0.8274                             & 0.0640                          & 0.1086                         & 0.8896                             & 0.0454                          & 0.0650                         \\ \bottomrule
		\end{tabular}%
	}
	\caption{The frequencies of different types of causal relations in all 50- and 100-node randomly sampled positive weight graphs used in our simulations.}
	\label{tab:feq}
\end{table}

\subsection{Mixed Edge Weights}\label{app:app:mix}

Figure~\ref{fig:true:mixed} shows the results on 100-node graphs with mixed edge weights. The true graph structures are provided in the experiments. It can be seen that the results are similar to those presented in the main text, though the Kappa coefficients of different methods drop in all cases. Nevertheless, the Kappa coefficients of the local ITC drop less than those of the CE-based methods. Thus, the differences between the local ITC and the CE-based methods increase. Note that, allowing mixed edge weights does not have much influence on the computational time, as the graph structures are generated according to the same model. As shown in Figure~\ref{fig:true:mixed}, the local ITC is more efficient, and the CE-based methods using non-ancestral relations are less efficient.



\begin{figure}[t!]
	\centering
	\subfigure{
		\begin{minipage}[t]{0.9\textwidth}
			\centering
			\includegraphics[width=\textwidth]{figs/true_graph_res/legend.pdf}
		\end{minipage}%
	}%
	\vspace{-1em}
	\addtocounter{subfigure}{-1}
	
	\subfigure[Kappa, $N_{\rm effect}=50$ \label{fig:true:kappa_100_50_mix}]{
		\begin{minipage}[t]{0.3\textwidth}
			\centering
			\includegraphics[width=\textwidth]{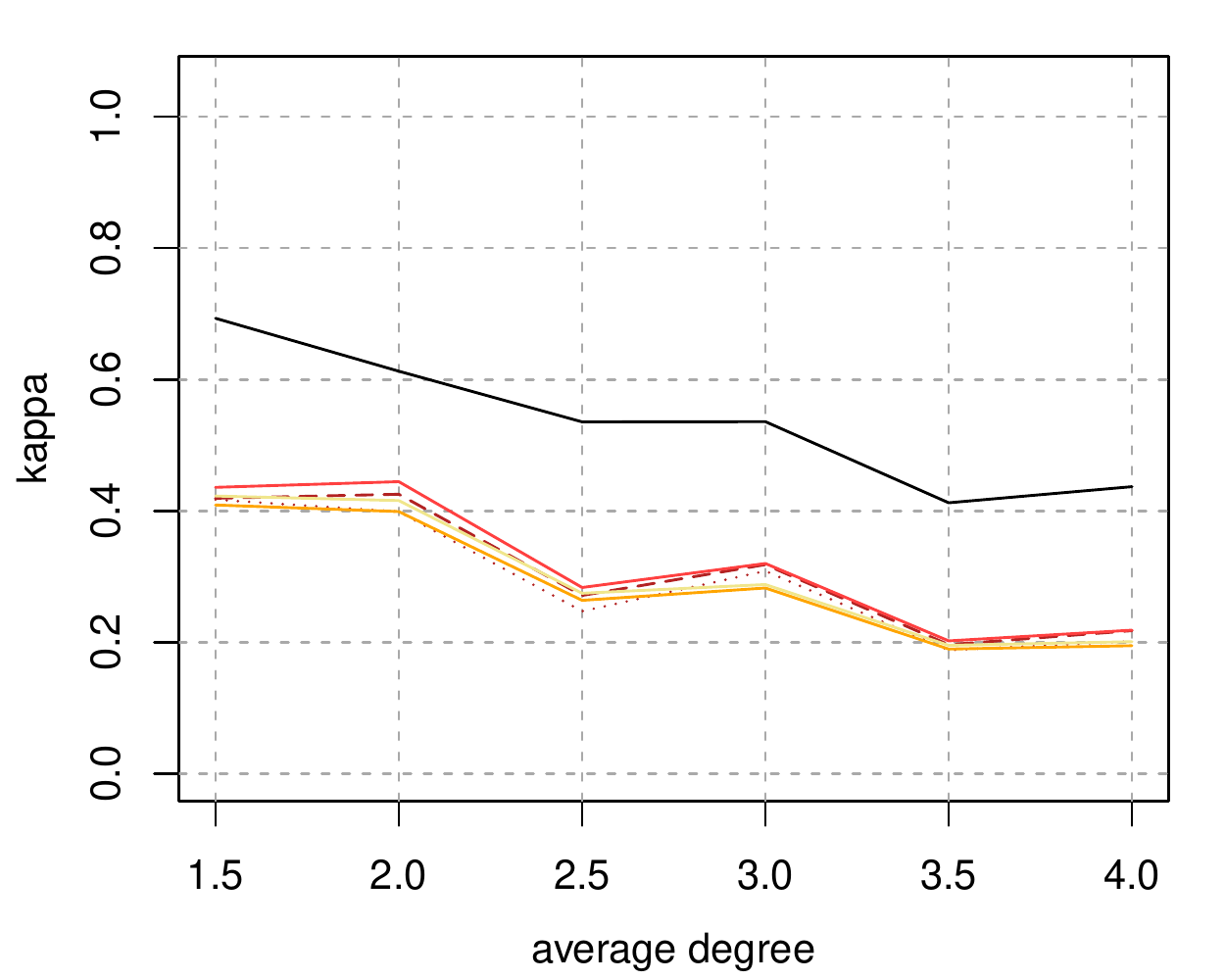}
		\end{minipage}%
	}%
	\hspace{0.01\textwidth}
	\subfigure[Kappa, $N_{\rm effect}=100$  \label{fig:true:kappa_100_100_mix}]{
		\begin{minipage}[t]{0.3\textwidth}
			\centering
			\includegraphics[width=\textwidth]{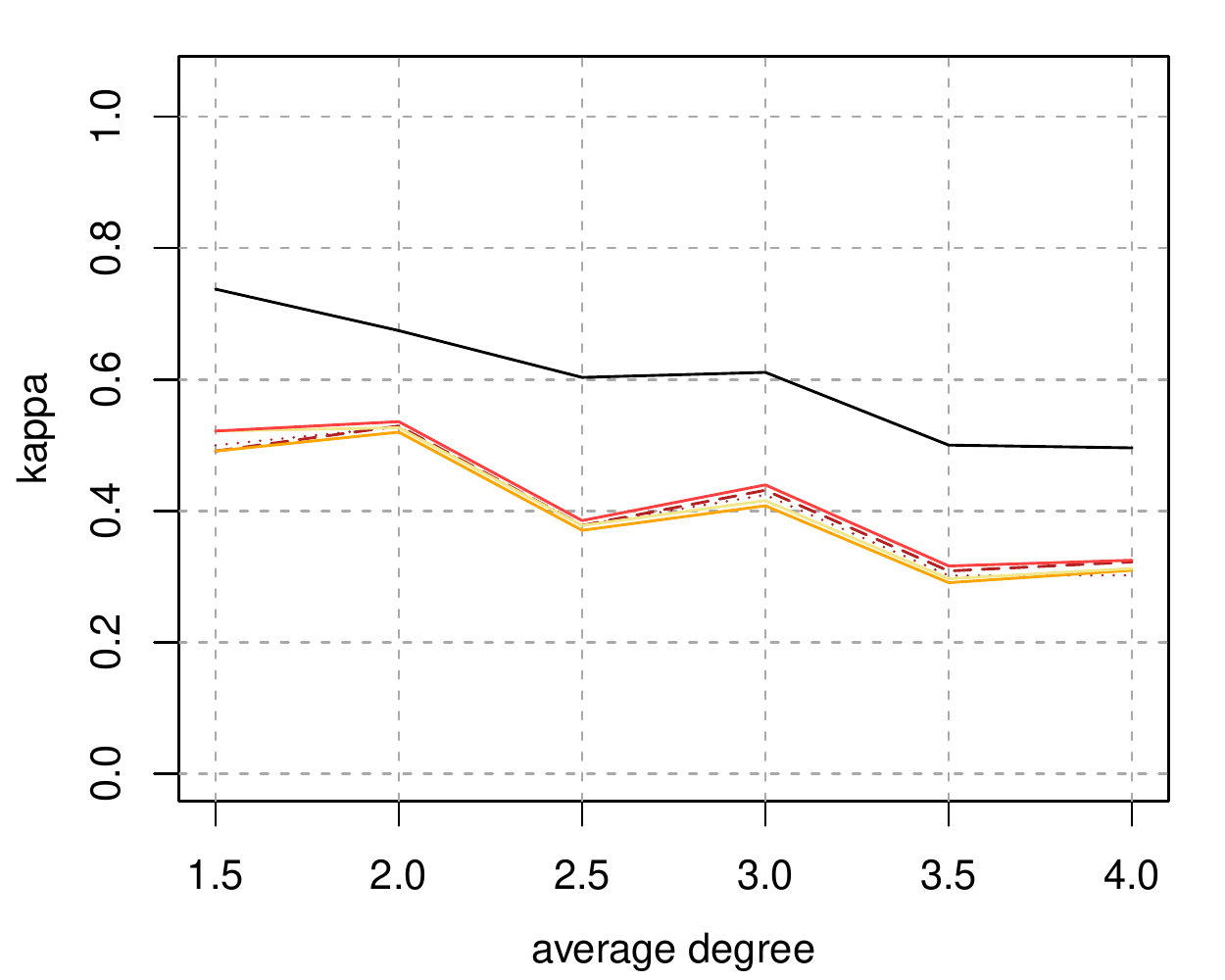}
		\end{minipage}%
	}%
	\hspace{0.01\textwidth}
	\subfigure[Kappa, $N_{\rm effect}=150$ \label{fig:true:kappa_100_150_mix}]{
		\begin{minipage}[t]{0.3\textwidth}
			\centering
			\includegraphics[width=\textwidth]{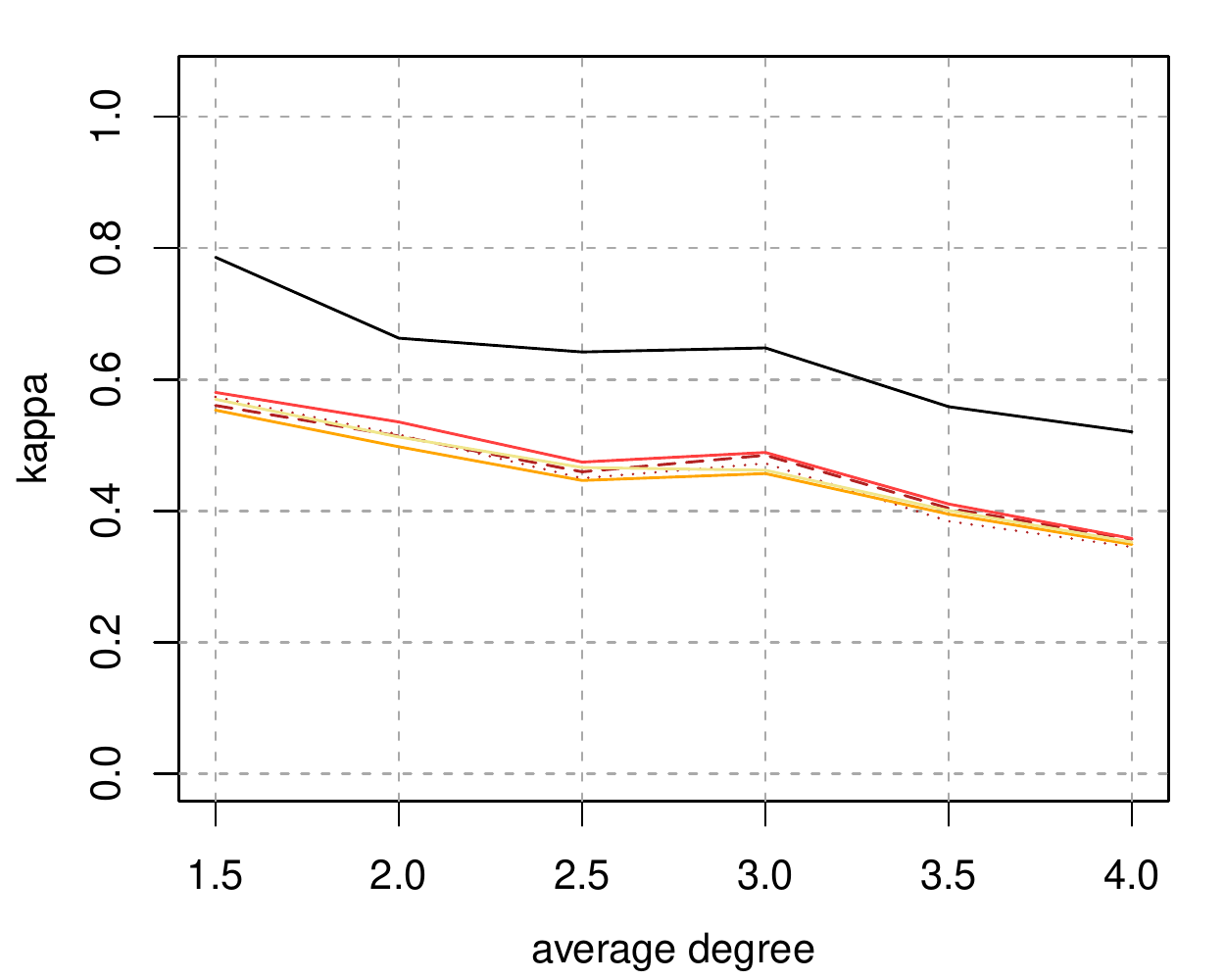}
		\end{minipage}%
	}%
	
	\subfigure[time, $N_{\rm effect}=50$ \label{fig:true:time_100_50_mix}]{
		\begin{minipage}[t]{0.3\textwidth}
			\centering
			\includegraphics[width=\textwidth]{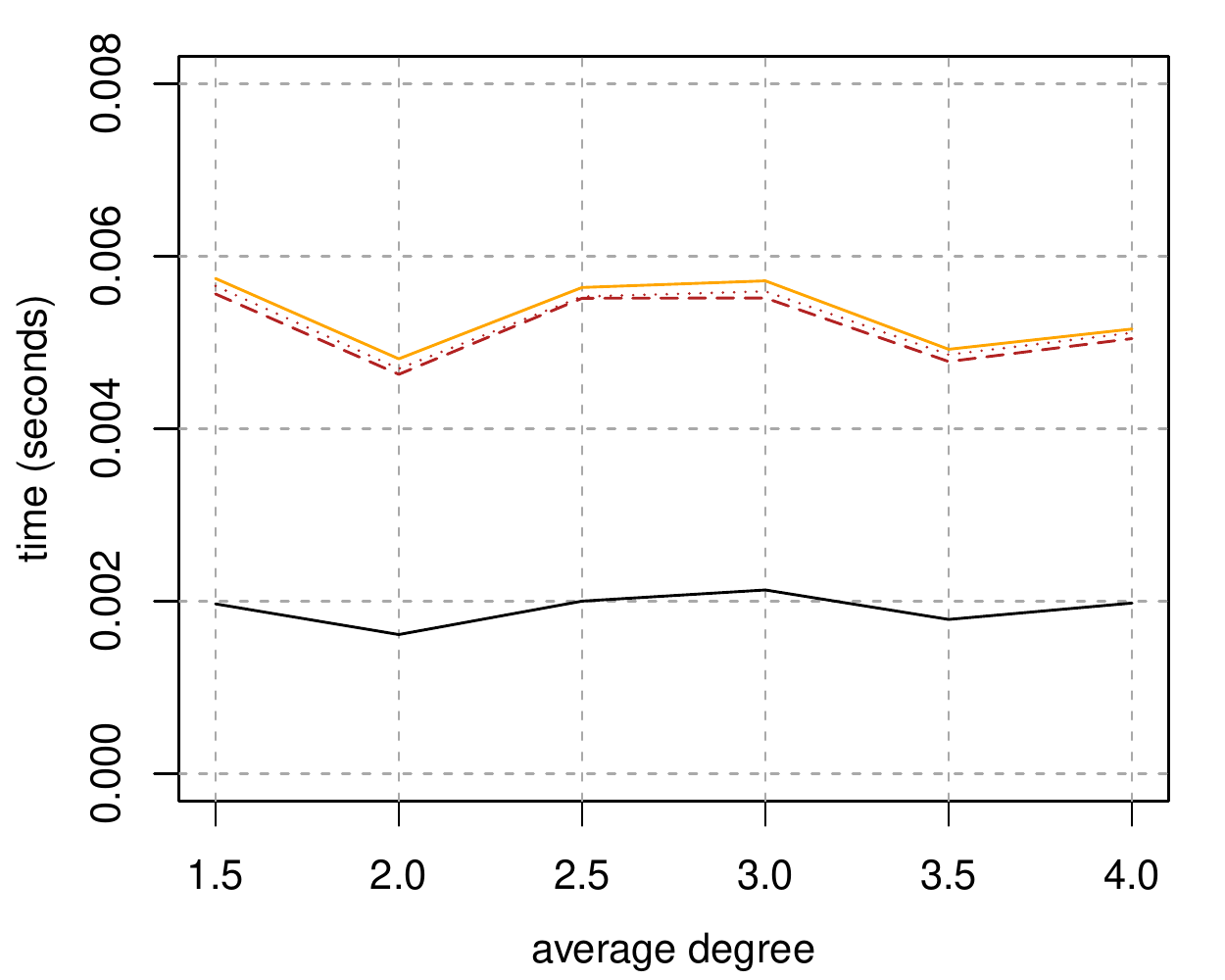}
		\end{minipage}%
	}%
	\hspace{0.01\textwidth}
	\subfigure[time, $N_{\rm effect}=100$ \label{fig:true:time_100_100_mix}]{
		\begin{minipage}[t]{0.3\textwidth}
			\centering
			\includegraphics[width=\textwidth]{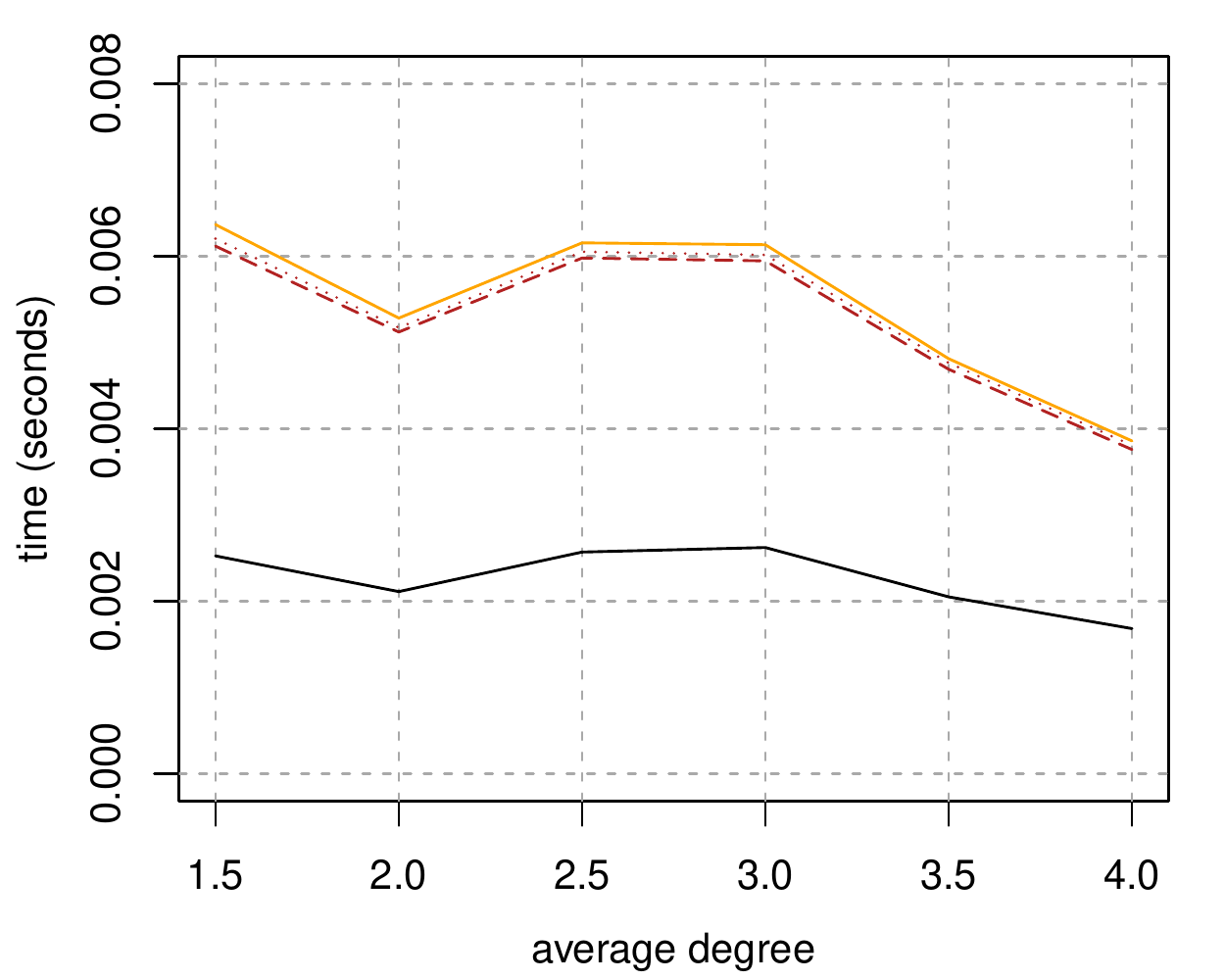}
		\end{minipage}%
	}%
	\hspace{0.01\textwidth}
	\subfigure[time, $N_{\rm effect}=150$ \label{fig:true:time_100_150_mix}]{
		\begin{minipage}[t]{0.3\textwidth}
			\centering
			\includegraphics[width=\textwidth]{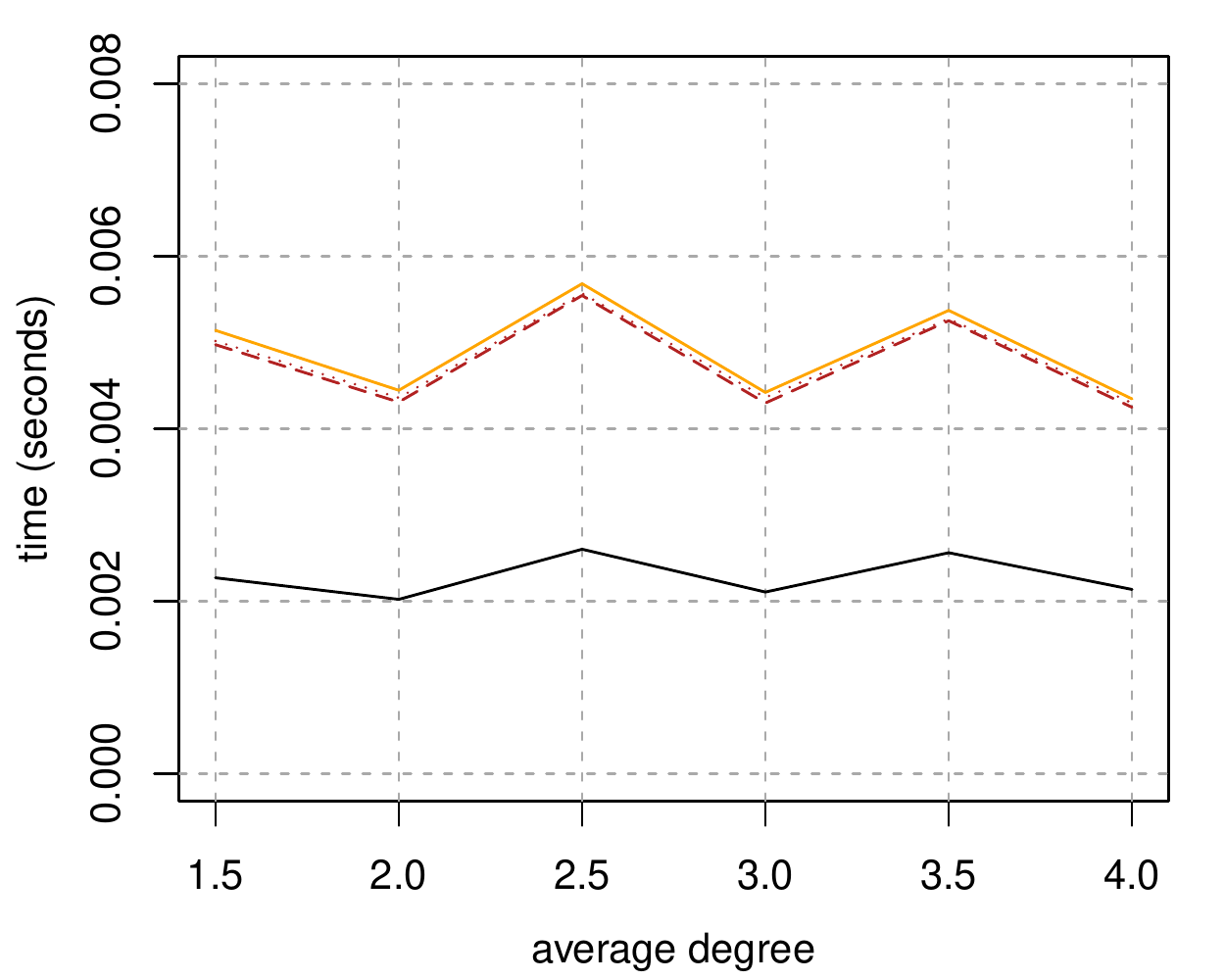}
		\end{minipage}%
	}%
	
	\caption{The Kappa coefficients and the CPU time (in seconds) of different methods on 100-node graphs with mixed edge weights. The true graph structures are provided. The CPU time of ``IDA + an + test (all)" and ``IDA + an + test (min/max)" is not shown, as they are more than 50-100 times slower than the other methods.}
	\label{fig:true:mixed}
\end{figure}

The existence of mixed edge weights generally increases the chance of  violations of the faithfulness assumption. Thus, when mixed edge weights are allowed, the discrepancy between the learned and the true graph structures  could be very large, and the performance of all methods declines. In our experiments, the Kappa coefficients of all methods drop down to $0.2\sim 0.4$. Nevertheless, the local ITC is still better than the other methods in most cases.

\subsection{Detailed TPRs and FPRs}\label{app:app:detailed}

In this section, we report the detailed TPRs and FPRs based on 100-node graphs  with  the average degree $d = 2$ and with positive or mixed edge weights.
  When the true graph structures are provided, the local ITC achieves the highest TPR and the lowest FPR in most cases, as shown in Table~\ref{tab:true}.    Table~\ref{tab:positive} and Table~\ref{tab:mixed} show the detailed TPRs and FPRs of different methods for identifying each type of causal relation based on the positive    and  mixed weight graphs, respectively. Note that, since all standard deviations of the reported TPRs and FPRs are below $0.002$, we only report the mean values. As one can see from the tables, the local ITC does not always outperform others. Nevertheless, the performance of the local ITC is more balanced in terms of both TPR and FPR.   On the other hand, the performance of the global ITC combined with GES is not as well as the other methods. We found that this is because in our simulations the CPDAG estimated by GES is relatively inaccurate.

\begin{table}[!t]
	\centering
	\resizebox{\textwidth}{!}{%
		\begin{tabular}{@{}clcccccc@{}}
			\toprule
			\multirow{2}{*}{} & \multicolumn{1}{c}{\multirow{2}{*}{Method}} & \multicolumn{2}{c}{Def. non-cause} & \multicolumn{2}{c}{Poss. cause} & \multicolumn{2}{c}{Def. cause} \\ \cmidrule(l){3-8}
			&                         & TPR            & FPR           & TPR              & FPR              & TPR              & FPR              \\ \midrule
			\multirow{6}{*}{}
			& true graph + local ITC               & \textbf{0.9994}            &\textbf{0.1209}           & \textbf{0.8512}              & \textbf{0.0014}              & \textbf{0.8361}              & \textbf{0.0004}              \\
			& true graph + IDA + test (all)        & 0.9985            & \textbf{0.1484}           & \textbf{0.8430}              & 0.0020              & \textbf{0.7705}              & 0.0006              \\
			& true graph + IDA + multi (all)         & 0.9992            & 0.1703           & 0.8099              & 0.0016              & 0.7541              & 0.0006              \\
			& true graph + IDA + test (min/max)            & 0.9985            & 0.1923           & 0.7934              & 0.0016              & \textbf{0.7705}              & 0.0006              \\
			& true graph + IDA + an + test (all)        & \textbf{1.0000}            & \textbf{0.1484}           & \textbf{0.8430}              & \textbf{0.0012}              & \textbf{0.7705}              & \textbf{0.0000}              \\
			& true graph + IDA + an + test (min/max)           & \textbf{1.0000}            & 0.1868           & 0.8017              & \textbf{0.0008}              & 0.7705              & \textbf{0.0000}              \\  \bottomrule
		\end{tabular}%
	}
	\caption{Some detailed TPRs and FPRs on 100-node graphs with positive edge weights. The graph structures are given.}
	\label{tab:true}
\end{table}

\begin{table}[!t]
	\centering
	\resizebox{\textwidth}{!}{%
		\begin{tabular}{@{}clcccccc@{}}
			\toprule
			\multirow{2}{*}{} & \multicolumn{1}{c}{\multirow{2}{*}{Method}} & \multicolumn{2}{c}{Def. non-cause} & \multicolumn{2}{c}{Poss. cause} & \multicolumn{2}{c}{Def. cause} \\ \cmidrule(l){3-8}
			&                         & TPR            & FPR           & TPR              & FPR              & TPR              & FPR              \\ \midrule
			\multirow{12}{*}{}
			& local + local ITC               & 0.9792            & \textbf{0.2527}           & \textbf{0.6446}              & 0.0180              & \textbf{0.5574}              & 0.0073              \\
			& local + IDA + test (all)        & 0.9765            & \textbf{0.2692}           & \textbf{0.6446}              & 0.0213              & 0.4918              & 0.0069              \\
			& local + IDA + multi (all)         & 0.9786            & \textbf{0.2802}           & \textbf{0.6364}              & 0.0193              & 0.4918              & 0.0067              \\
			& PC + global ITC           & 0.9927            & 0.5879           & 0.4215              & 0.0092              & 0.1803              & \textbf{0.0006}              \\
			& PC + IDA + test (all)        & 0.9653            & 0.2802           & 0.5950              & 0.0281              & 0.4262              & 0.0128              \\
			& PC + IDA + test (min/max)          & 0.9685            & 0.3462           & 0.5289              & 0.0236              & 0.4262              & 0.0134              \\
			& PC + IDA + an + test (all)    & \textbf{0.9936}            & 0.6099           & 0.4215              & 0.0086              & 0.1148              & \textbf{0.0004}              \\
			& PC + IDA + an + test (min/max)    & \textbf{0.9936}            & 0.6264           & 0.3967              & 0.0086              & 0.1148              & \textbf{0.0004}              \\
			& PCS + global ITC    & 0.9929            & 0.6264           & 0.3719              & \textbf{0.0082}              & 0.1475              & 0.0016              \\
			& PCS + IDA + test (all)   & 0.9647            & 0.3132           & 0.5289              & 0.0268              & 0.4262              & 0.0150              \\
			& GES + global ITC    & 0.6467            & 0.2912           & 0.0248              & \textbf{0.0066}              & \textbf{0.9508}              & 0.3519              \\
			& GES + IDA + test (all)    & \textbf{0.9979}            & 0.6264           & 0.0331              & \textbf{0.0004}              & \textbf{0.6066}              & 0.0071              \\ \bottomrule
		\end{tabular}%
	}
	\caption{Some detailed TPRs and FPRs on 100-node graphs with positive edge weights. The graph structures are learned from data.}
	\label{tab:positive}
\end{table}

\begin{table}[!t]
	\centering
	\resizebox{\textwidth}{!}{%
		\begin{tabular}{@{}clcccccc@{}}
			\toprule
			\multirow{2}{*}{} & \multicolumn{1}{c}{\multirow{2}{*}{Method}} & \multicolumn{2}{c}{Def. non-cause} & \multicolumn{2}{c}{Poss. cause} & \multicolumn{2}{c}{Def. cause} \\ \cmidrule(l){3-8}
			&                         & TPR            & FPR           & TPR              & FPR              & TPR              & FPR              \\ \midrule
			\multirow{12}{*}{}           & local + local ITC               & 0.9925            & \textbf{0.6099}           & \textbf{0.3448}              & 0.0068              & \textbf{0.2121}              & 0.0041              \\
			& local + IDA + test (all)        & 0.9911            & 0.6154           & \textbf{0.3448}              & 0.0078              & \textbf{0.2121}              & 0.0043              \\
			& local + IDA + multi (all)         & 0.9917            & 0.6264           & \textbf{0.3362}              & 0.0072              & 0.1970              & 0.0043              \\
			& PC + global ITC           & \textbf{0.9973}            & 0.7912           & 0.2155              & 0.0035              & 0.0606              & \textbf{0.0010}              \\
			& PC + IDA + test (all)        & 0.9832            & \textbf{0.6044}           & 0.3190              & 0.0147              & 0.1970              & 0.0063              \\
			& PC + IDA + test (min/max)          & 0.9846            & 0.6374           & 0.2931              & 0.0127              & 0.1970              & 0.0063              \\
			& PC + IDA + an + test (all)    & \textbf{0.9979}            & 0.7912           & 0.2241              & \textbf{0.0029}              & 0.0606              & \textbf{0.0008}              \\
			& PC + IDA + an + test (min/max)    & \textbf{0.9981}            & 0.8022           & 0.2069              & \textbf{0.0027}              & 0.0606              & \textbf{0.0008}              \\
			& PCS + global ITC    & \textbf{0.9973}            & 0.7967           & 0.2069              & 0.0033              & 0.0455              & 0.0014              \\
			& PCS + IDA + test (all)   & 0.9832            & \textbf{0.6099}           & 0.3017              & 0.0147              & 0.1818              & 0.0067              \\
			& GES + global ITC    & 0.6743            & \textbf{0.4121}           & 0.0259              & 0.0080              & \textbf{0.6667}              & 0.3223              \\
			& GES + IDA + test (all)    & 0.9965            & 0.7857           & \textbf{0.0000}              & 0.0000              & \textbf{0.2727}              & 0.0077               \\ \bottomrule
		\end{tabular}%
	}
	\caption{Some detailed TPRs and FPRs on 100-node graphs with mixed edge weights. The graph structures are learned from data.}
	\label{tab:mixed}
\end{table}

Compare Table~\ref{tab:mixed} to Table~\ref{tab:positive}, it can be seen that the TPRs of different methods for learning possible and definite causes decrease significantly, while the corresponding FPRs are stable. On the other hand, the FPRs for learning definite non-causes increase, but the corresponding TPRs are stable. These results suggest that, when the mixed edge weights are allowed, many possible and definite causes are wrongly identified as definite non-causes. This is probably due to the violation of the faithfulness assumption, since many causal paths are missing in the learned graph as two causal paths may cancel each other out, and the total causal effects and the  dependence of between a cause and a effect  may also vanish because of the canceling paths.


\subsection{Optimal IDA}\label{app:app:opt}

\begin{figure}[t!]
	\centering
	\subfigure{
		\begin{minipage}[t]{1\textwidth}
			\centering
			\includegraphics[width=1\textwidth]{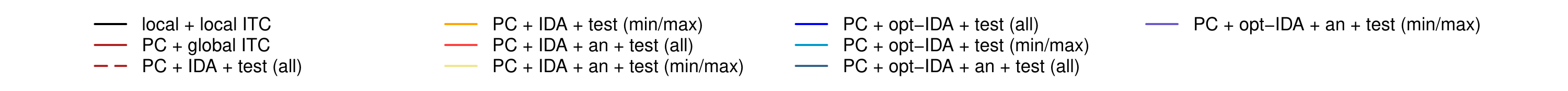}
		\end{minipage}%
	}%
	\vspace{-1.2em}
	\addtocounter{subfigure}{-1}
	
	\subfigure[$n=50$, $N=(100,100)$ and postive edge weights \label{fig:opt:kappa_50_100_100}]{
		\begin{minipage}[t]{0.3\textwidth}
			\centering
			\includegraphics[width=\textwidth]{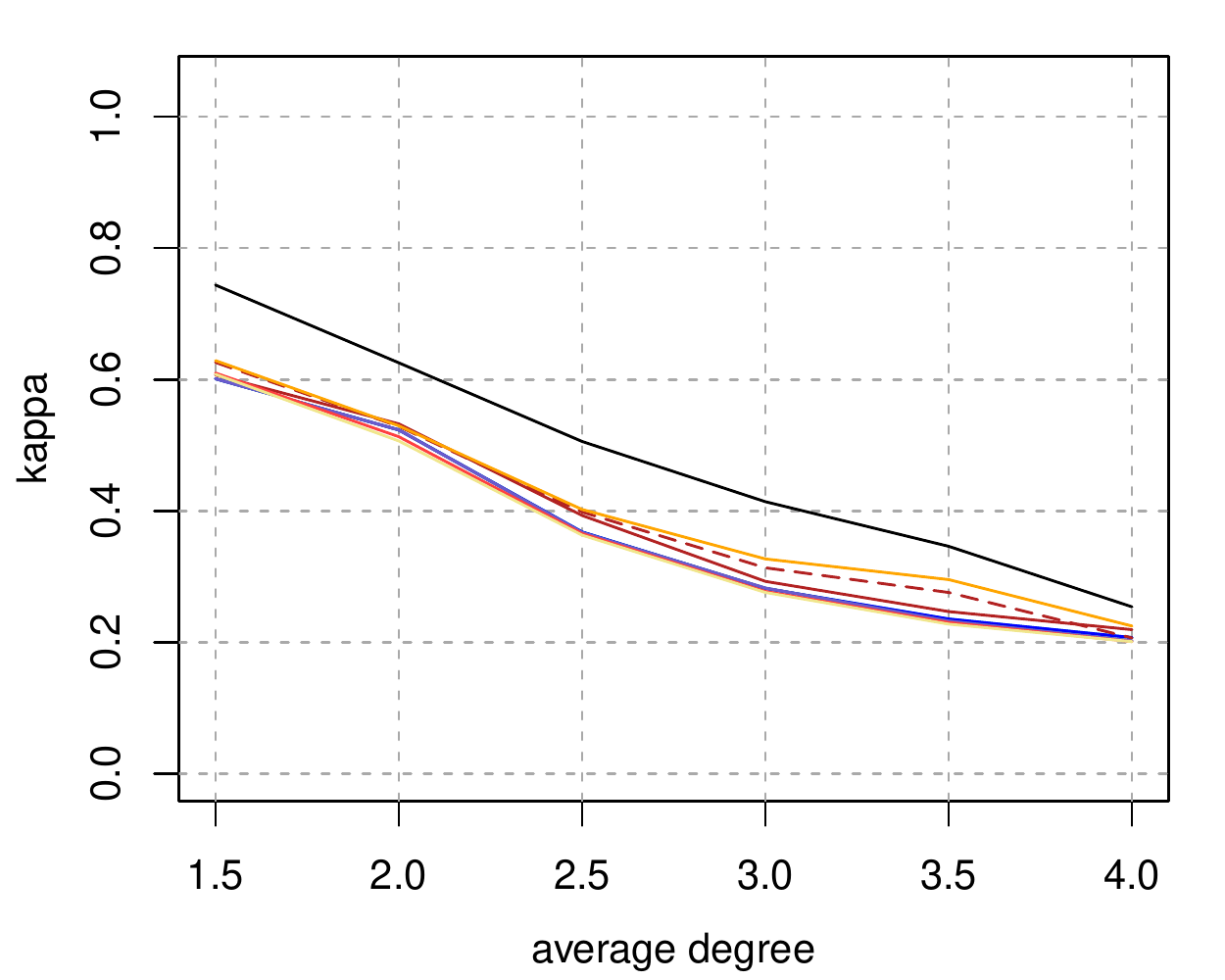}
		\end{minipage}%
	}%
	\hspace{0.01\textwidth}
	\subfigure[$n=50$, $N=(200,100)$  and positive edge weights \label{fig:opt:kappa_50_200_100}]{
		\begin{minipage}[t]{0.3\textwidth}
			\centering
			\includegraphics[width=\textwidth]{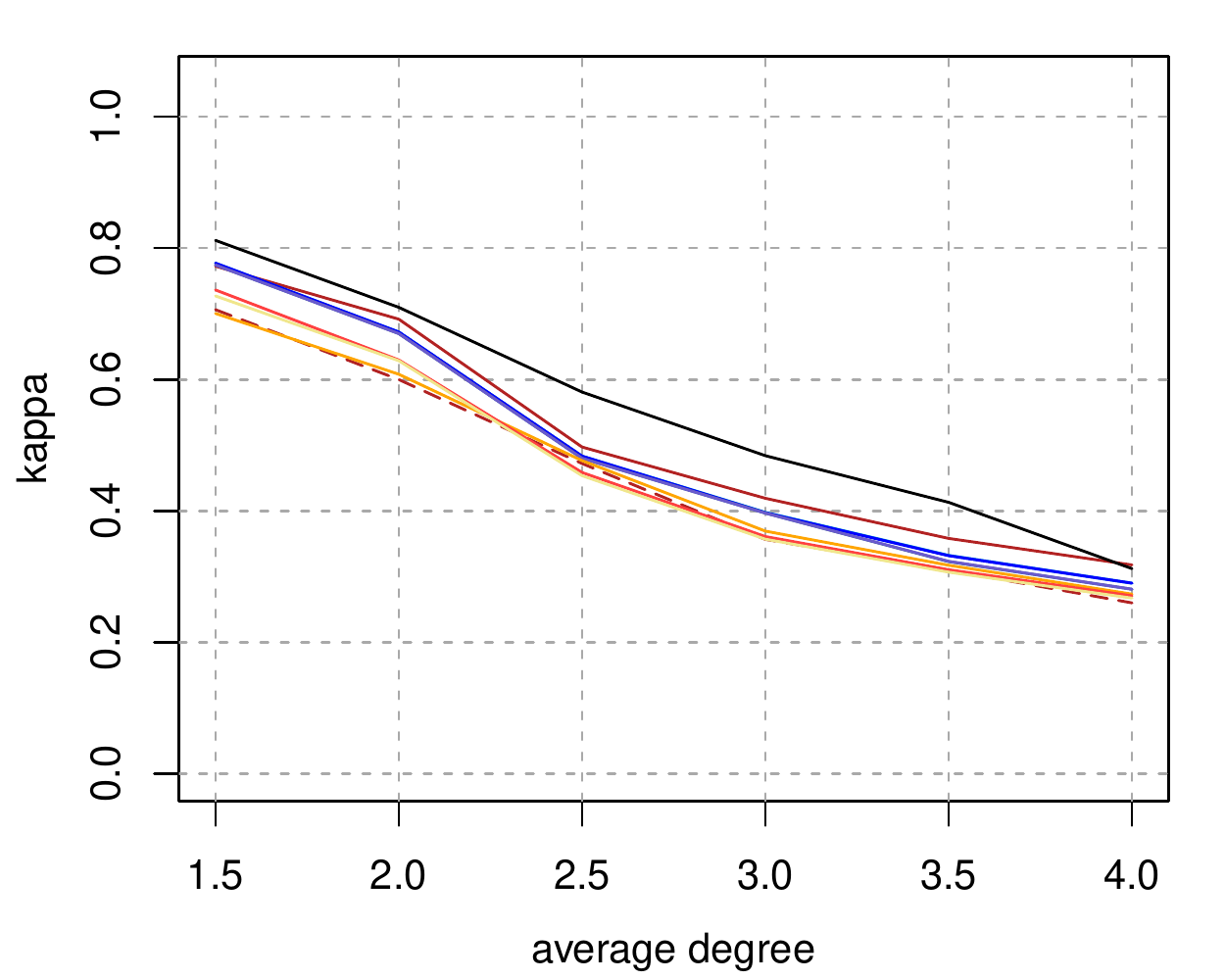}
		\end{minipage}%
	}%
	\hspace{0.01\textwidth}
	\subfigure[$n=50$, $N=(500,150)$ and postive edge weights \label{fig:opt:kappa_50_500_150}]{
		\begin{minipage}[t]{0.3\textwidth}
			\centering
			\includegraphics[width=\textwidth]{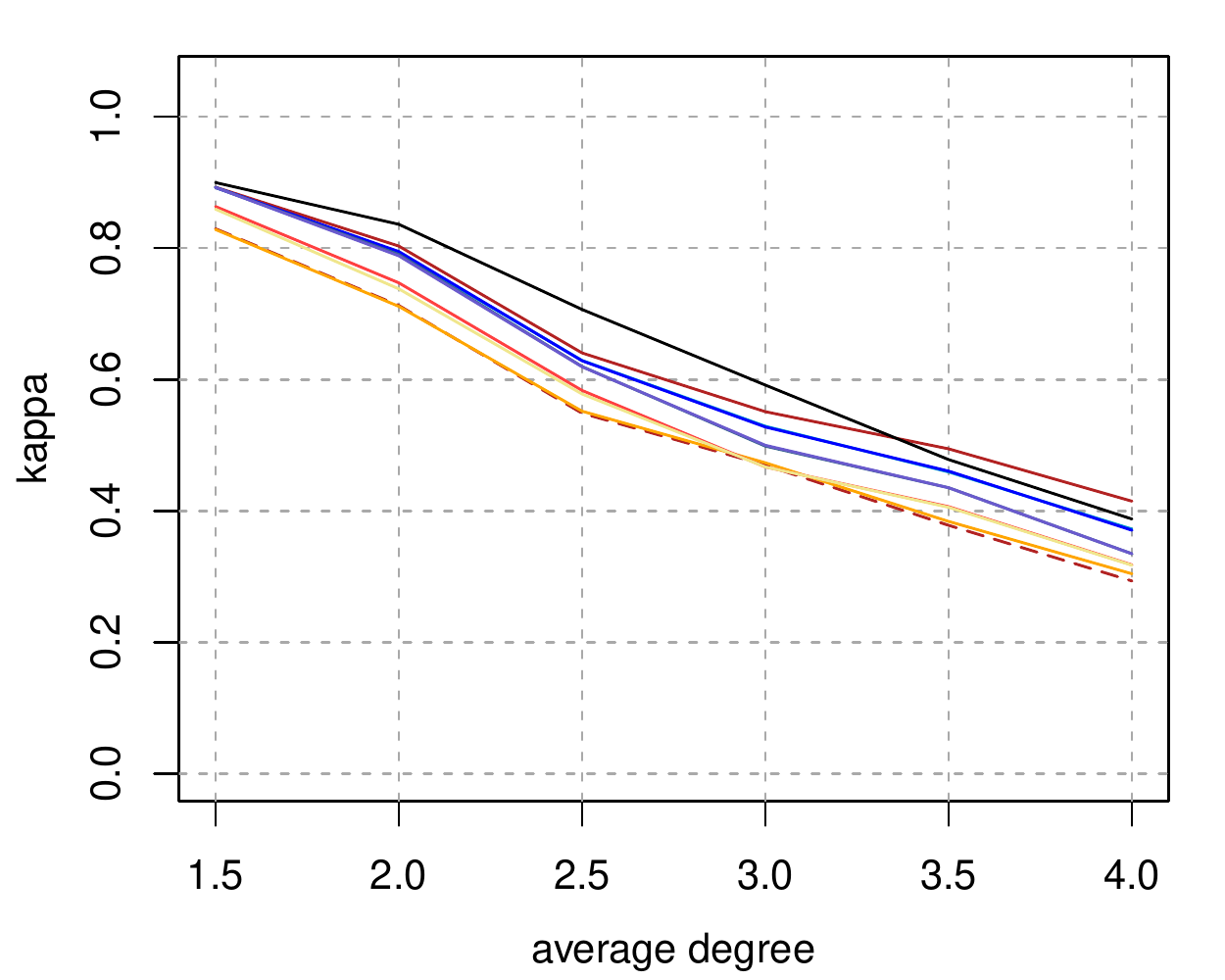}
		\end{minipage}%
	}%
	
	\subfigure[$n=100, N=(100,100)$ and postive edge weights \label{fig:opt:kappa_100_50}]{
		\begin{minipage}[t]{0.3\textwidth}
			\centering
			\includegraphics[width=\textwidth]{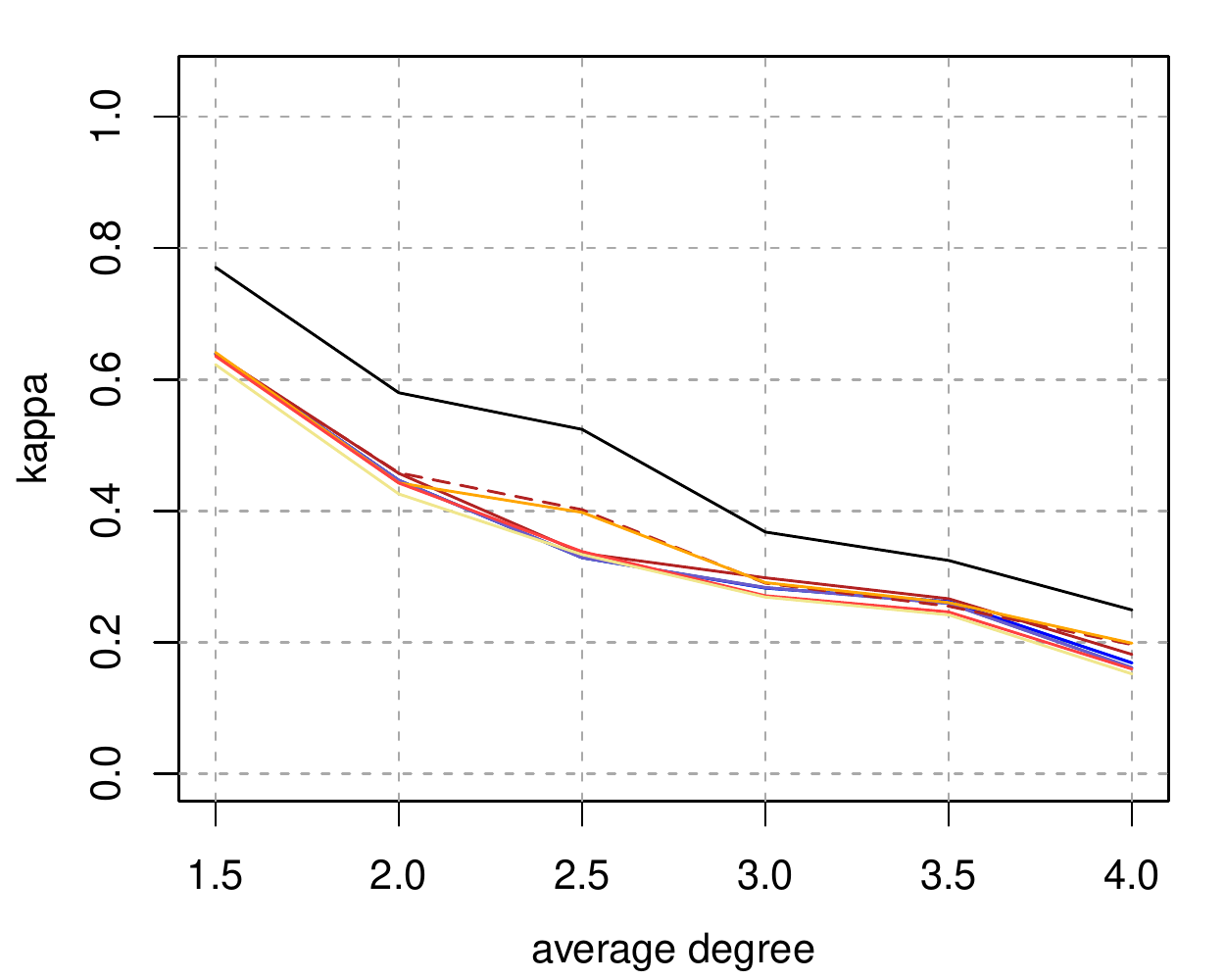}
		\end{minipage}%
	}%
	\hspace{0.01\textwidth}
	\subfigure[$n=100, N=(200,100)$ and postive edge weights \label{fig:opt:kappa_100_100}]{
		\begin{minipage}[t]{0.3\textwidth}
			\centering
			\includegraphics[width=\textwidth]{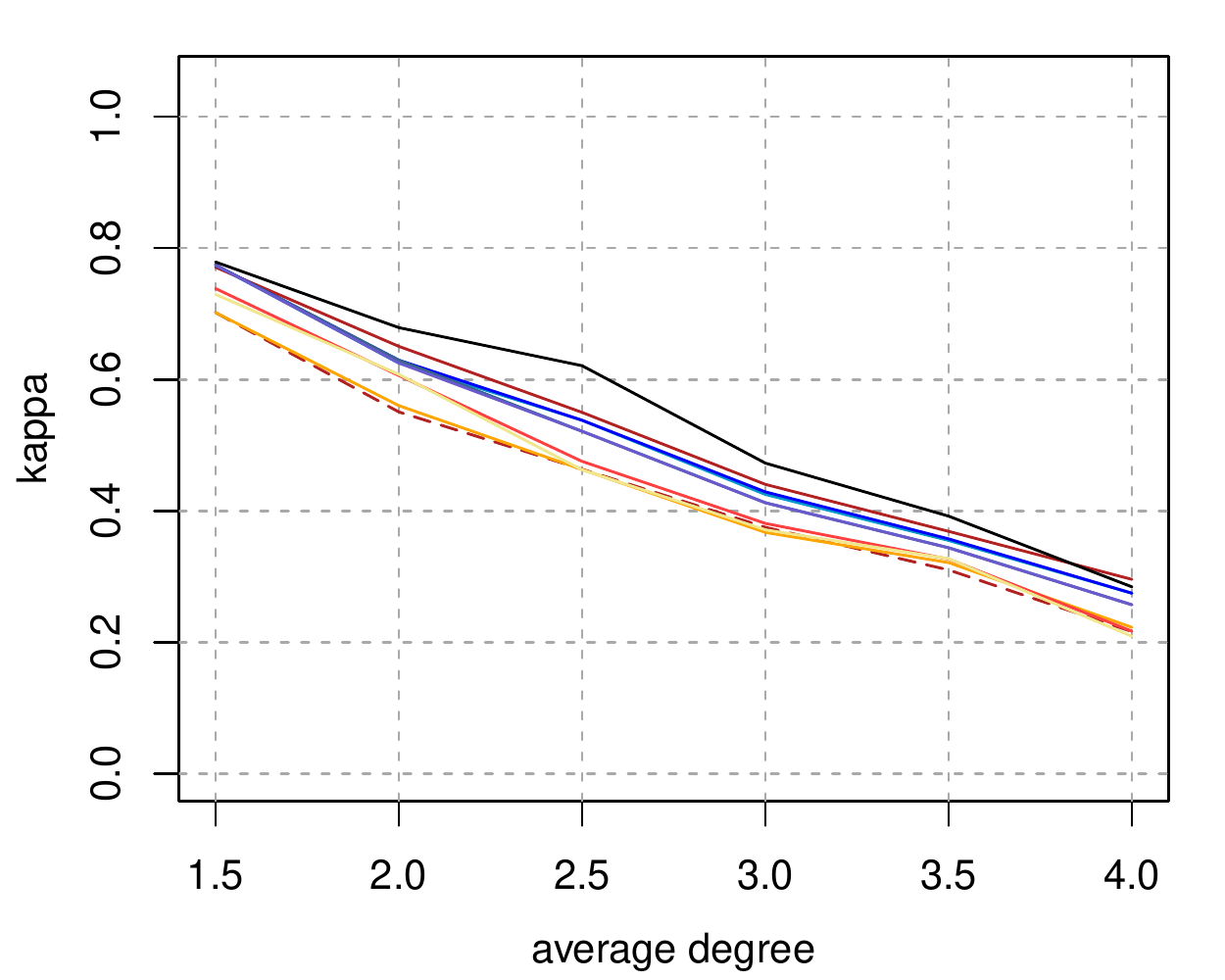}
		\end{minipage}%
	}%
	\hspace{0.01\textwidth}
	\subfigure[$n = 100, N=(500,150)$ and postive edge weights \label{fig:opt:kappa_100_150}]{
		\begin{minipage}[t]{0.3\textwidth}
			\centering
			\includegraphics[width=\textwidth]{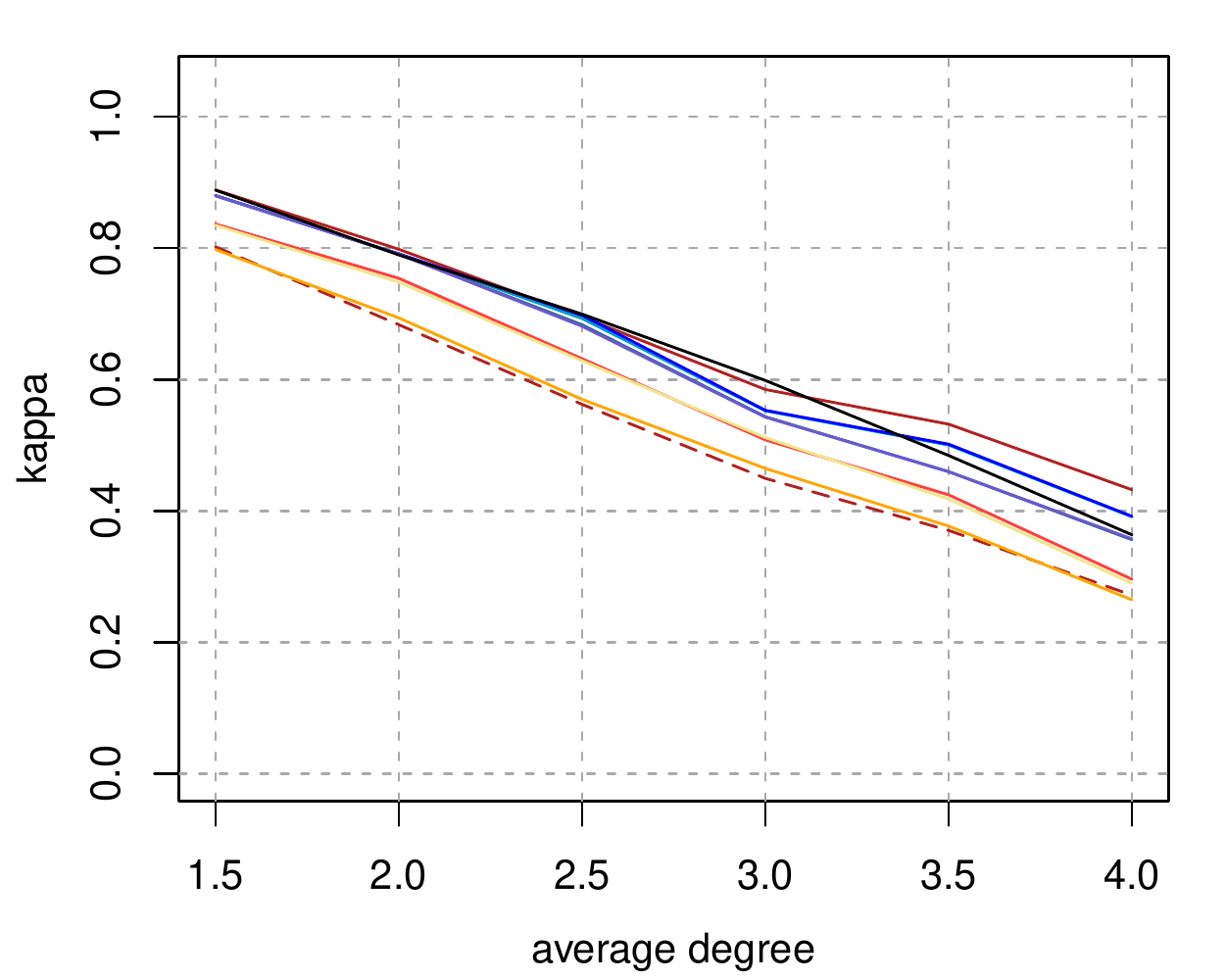}
		\end{minipage}%
	}%
	
	\subfigure[$n=100, N=(100,100)$ and mixed edge weights \label{fig:opt:kappa_100_50:mix}]{
		\begin{minipage}[t]{0.3\textwidth}
			\centering
			\includegraphics[width=\textwidth]{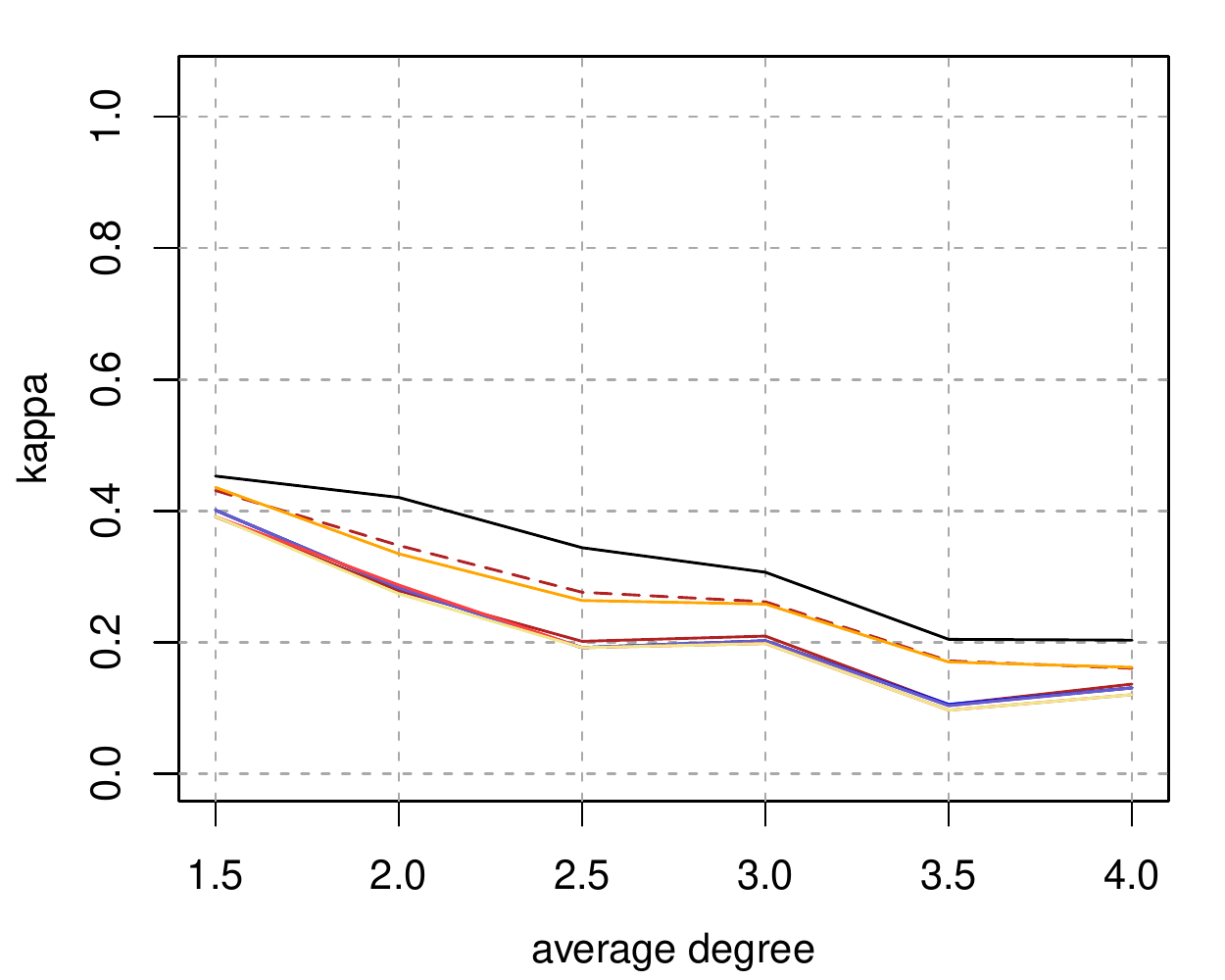}
		\end{minipage}%
	}%
	\hspace{0.01\textwidth}
	\subfigure[$n=100, N=(200,100)$ and mixed edge weights \label{fig:opt:kappa_100_100:mix}]{
		\begin{minipage}[t]{0.3\textwidth}
			\centering
			\includegraphics[width=\textwidth]{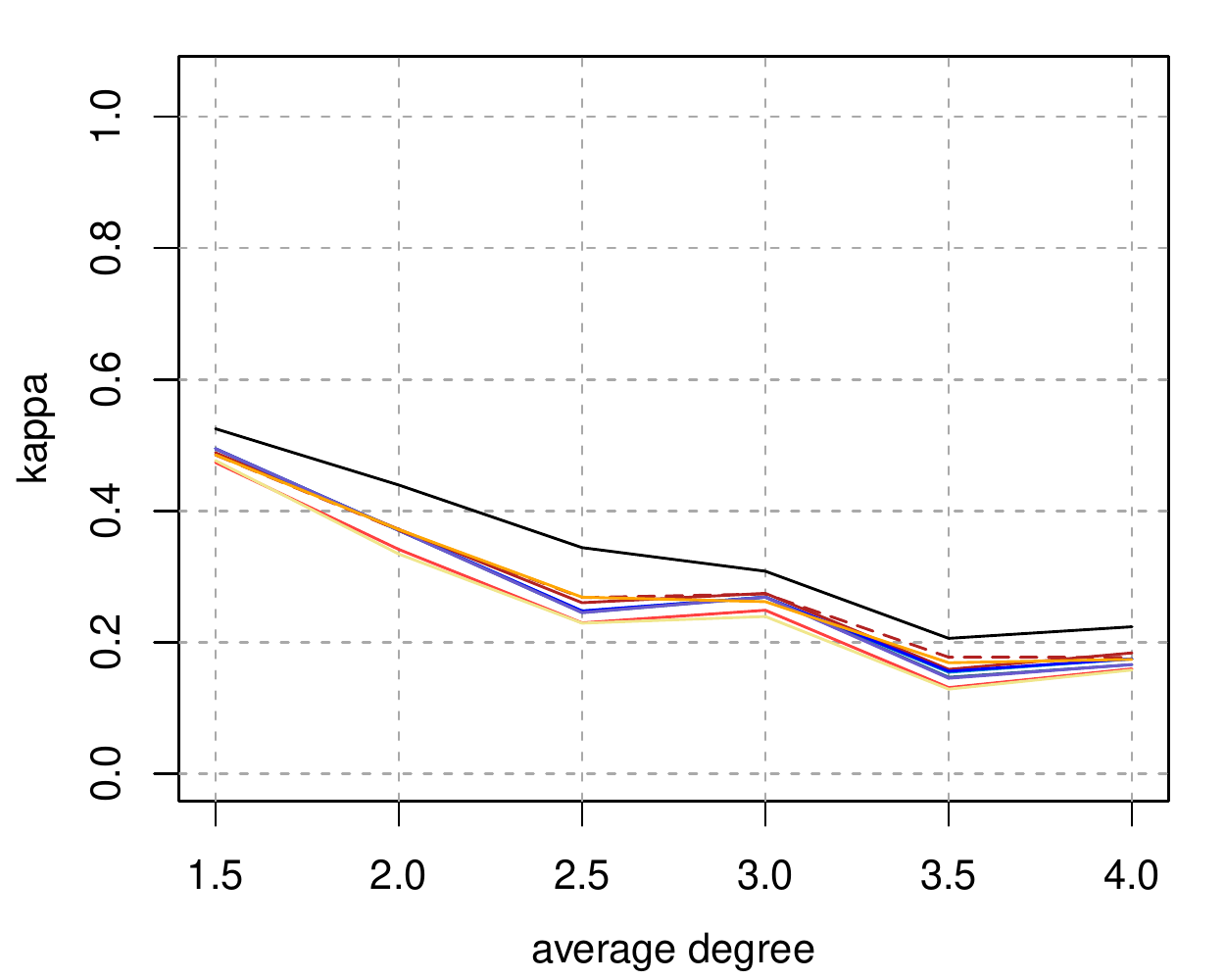}
		\end{minipage}%
	}%
	\hspace{0.01\textwidth}
	\subfigure[$n = 100, N=(500,150)$ and mixed edge weights \label{fig:opt:kappa_100_150:mix}]{
		\begin{minipage}[t]{0.3\textwidth}
			\centering
			\includegraphics[width=\textwidth]{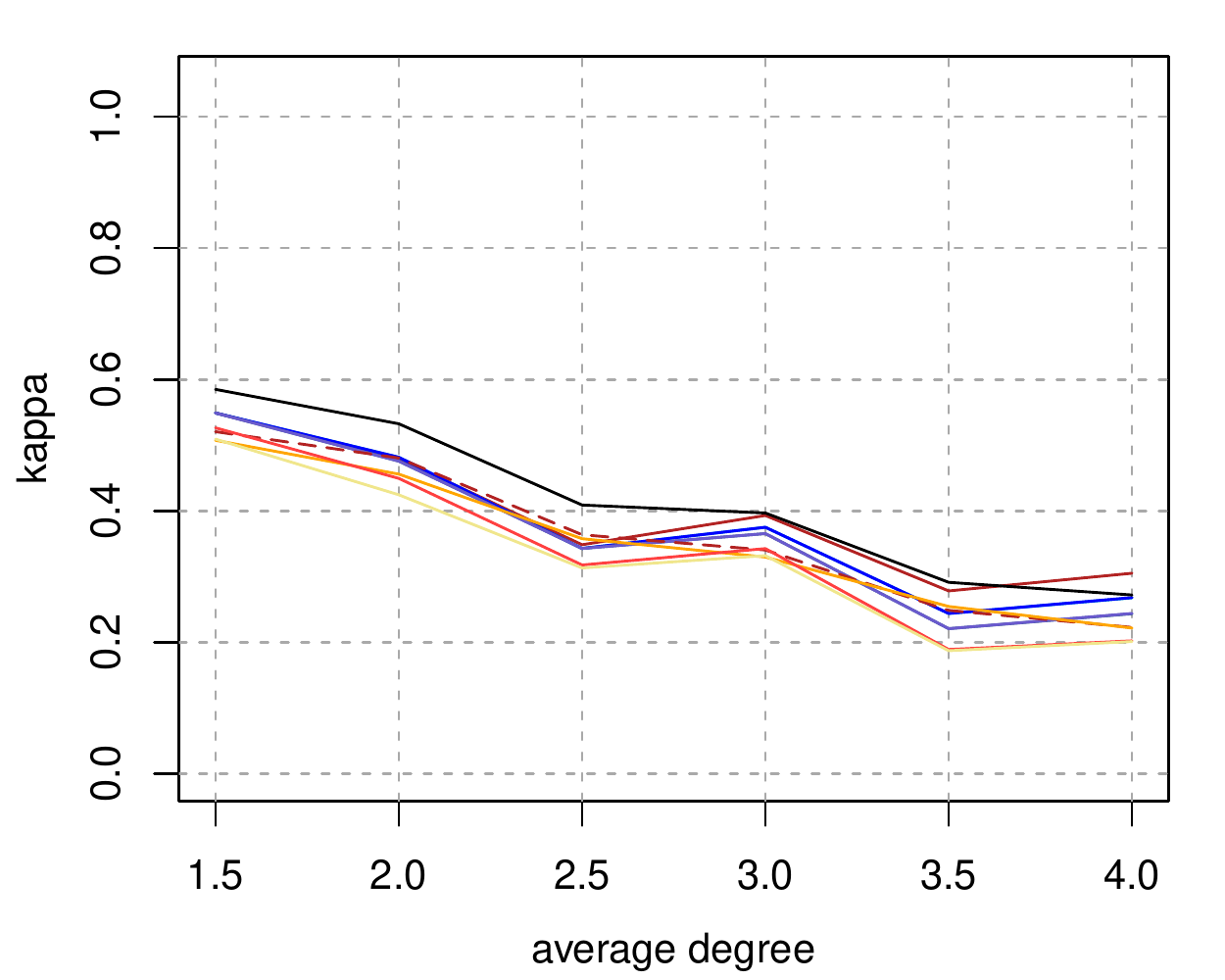}
		\end{minipage}%
	}%
	\caption{The Kappa coefficients of the CE-based methods using the optimal IDA. The graph structures are learned from data. $N=(N_{\rm graph}, N_{\rm effect})$ denotes the sample sizes for learning graphs and estimating causal effects.}
	\label{fig:opt:kappa}
\end{figure}

In this section, we study the CE-based methods with the optimal IDA instead of the original IDA, and compare them to the local ITC, global ITC and the other CE-based methods.

Figure~\ref{fig:opt:kappa} shows the Kappa coefficients of the local ITC (combined with the variant of MB-by-MB), the global ITC (combined with PC), four CE-based methods with the original IDA (combined with PC), and four CE-based methods with the optimal IDA (combined with PC). Note that, the optimal IDA is a semi-local algorithm, which uses Meek's rules and thus requires an entire CPDAG as input. Therefore, unlike the original IDA, the optimal IDA cannot be combined with the variant of MB-by-MB. In most cases, the local ITC is the best, and  the global ITC generally has better performance than the CE-based methods using the optimal IDA. Consider the CE-based methods, when the edge weights are all positive and the sample size is large, the CE-based methods with the optimal IDA is better than that with the original IDA. However, when the sample size is small or the edge weights are mixed, these methods have similar performance.

Comparing four CE-based methods with the optimal IDA, one can see that testing all estimated effects are better than   only testing the minimum and maximum absolute estimated effects. However, when using the optimal IDA, utilizing non-ancestral relations no longer has significant improvement on the results. This is probably due to the fact that the non-ancestral relations have already been implicitly considered when finding the optimal adjustment set.

\subsection{Hybrid Method}\label{app:app:hybrid}

 \begin{figure}[t!]
	\centering
	\subfigure{
		\begin{minipage}[t]{1\textwidth}
			\centering
			\includegraphics[width=1\textwidth]{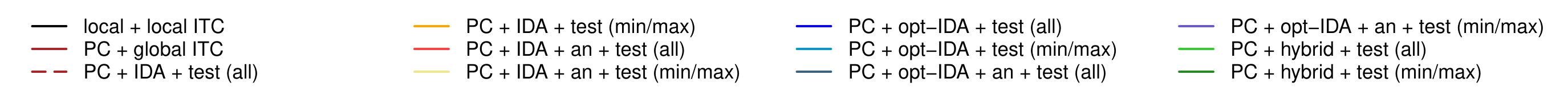}
		\end{minipage}%
	}%
	\vspace{-1.2em}
	\addtocounter{subfigure}{-1}
	
	\subfigure[Kappa, $N=(100,100)$]{
		\begin{minipage}[t]{0.3\textwidth}
			\centering
			\includegraphics[width=\textwidth]{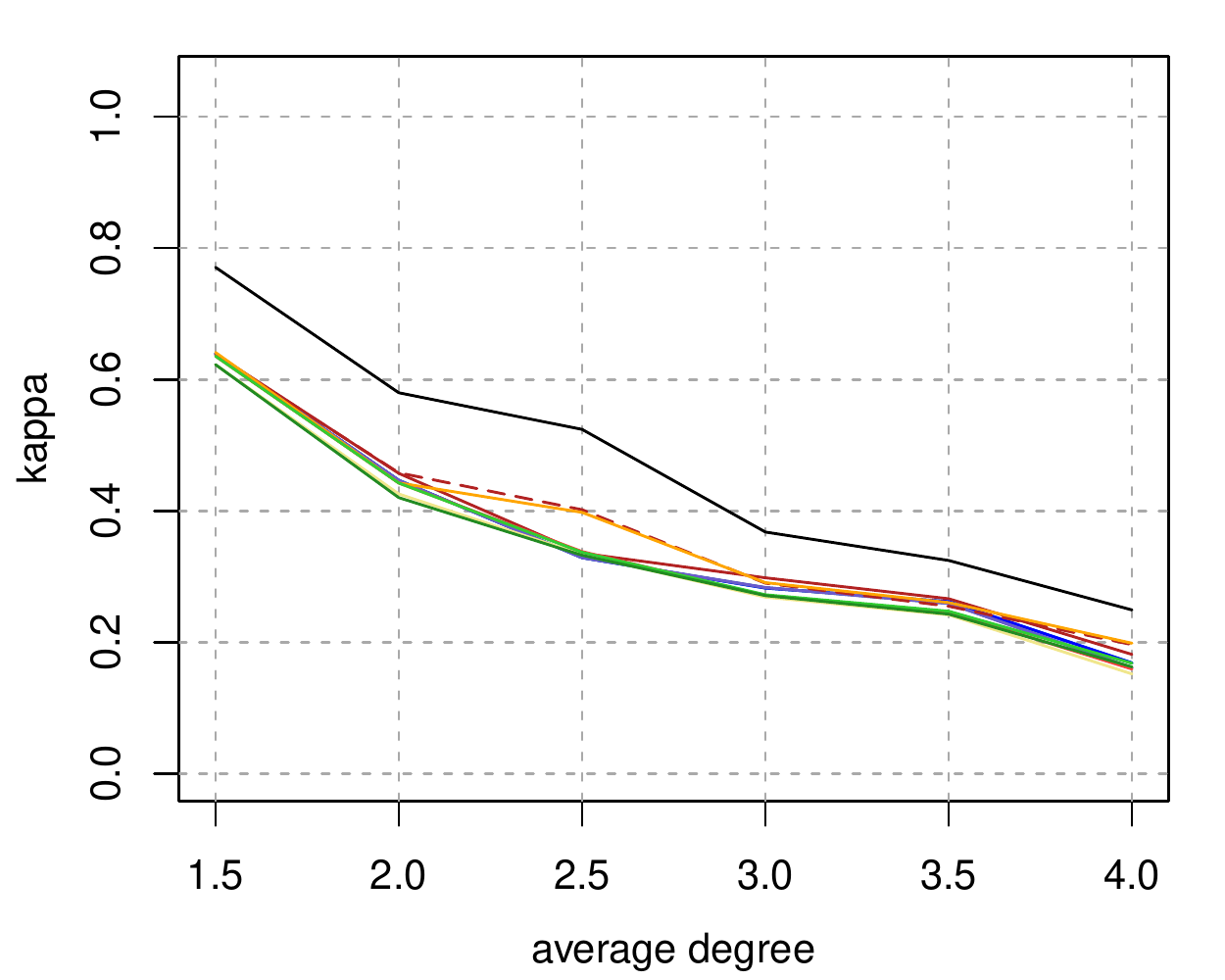}
		\end{minipage}%
	}%
	\hspace{0.01\textwidth}
	\subfigure[Kappa, $N=(200,100)$]{
		\begin{minipage}[t]{0.3\textwidth}
			\centering
			\includegraphics[width=\textwidth]{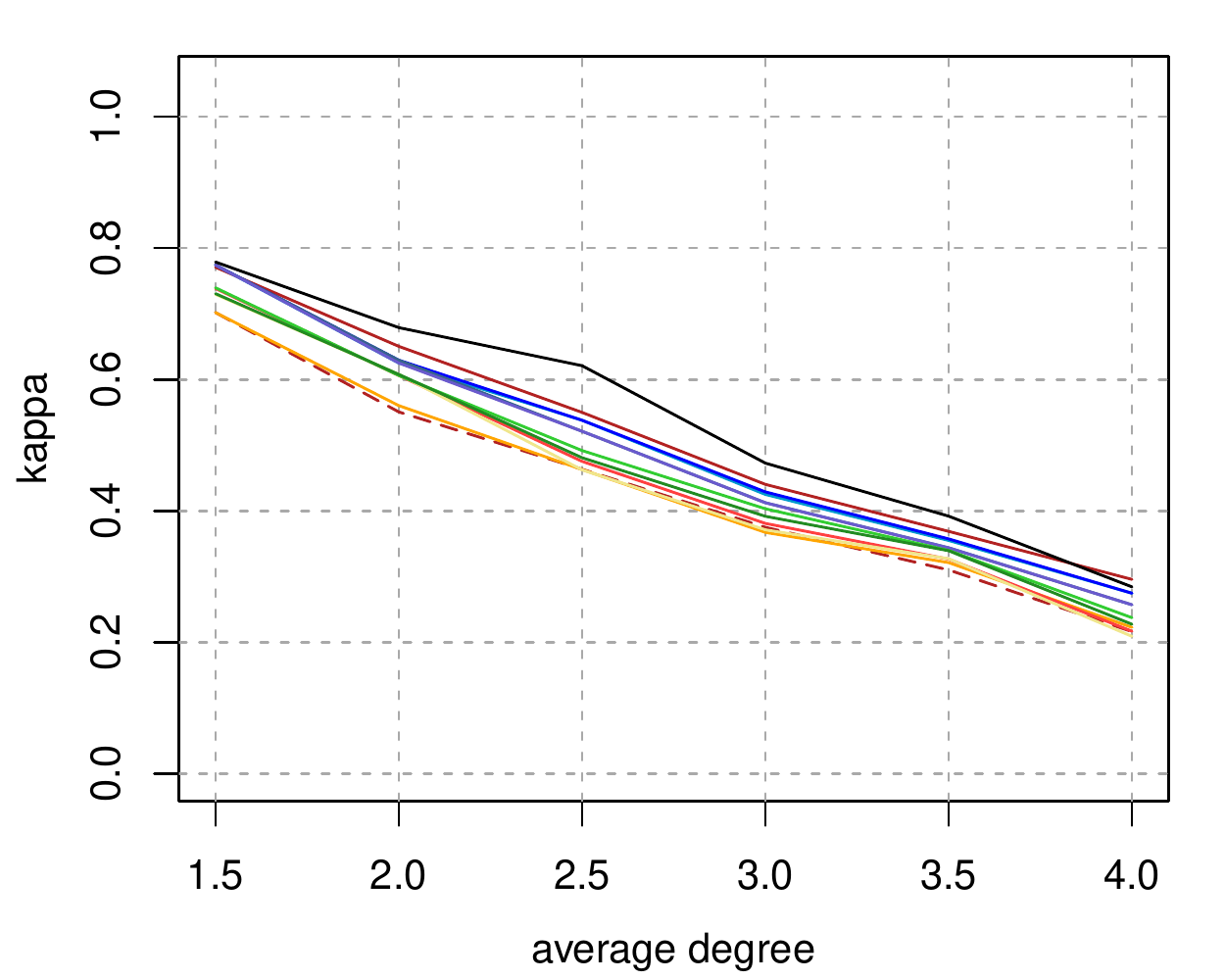}
		\end{minipage}%
	}%
	\hspace{0.01\textwidth}
	\subfigure[Kappa, $N=(500,150)$]{
		\begin{minipage}[t]{0.3\textwidth}
			\centering
			\includegraphics[width=\textwidth]{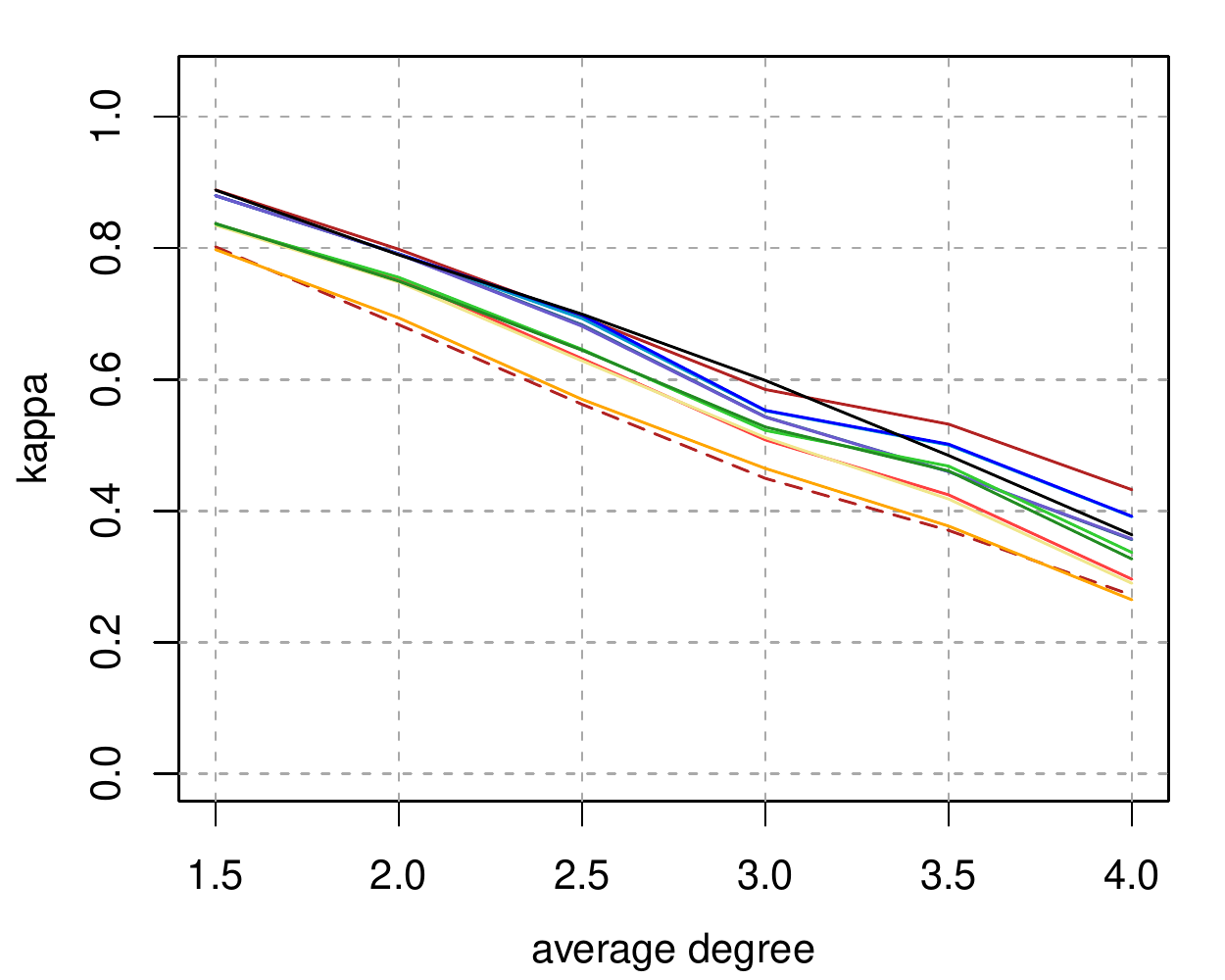}
		\end{minipage}%
	}%
	
	\subfigure[time, $N=(100,100)$]{
		\begin{minipage}[t]{0.3\textwidth}
			\centering
			\includegraphics[width=\textwidth]{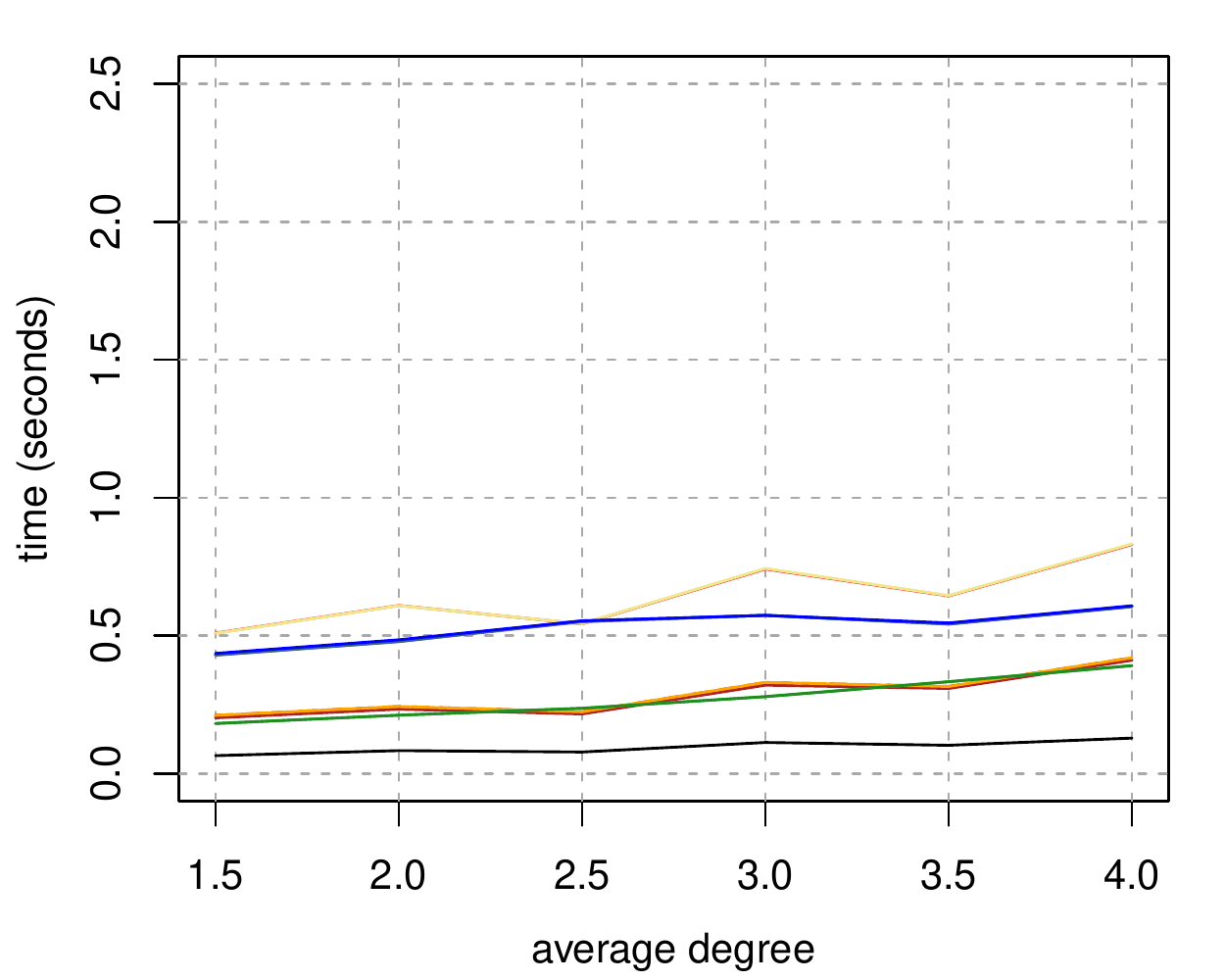}
		\end{minipage}%
	}%
	\hspace{0.01\textwidth}
	\subfigure[time, $N=(200,100)$]{
		\begin{minipage}[t]{0.3\textwidth}
			\centering
			\includegraphics[width=\textwidth]{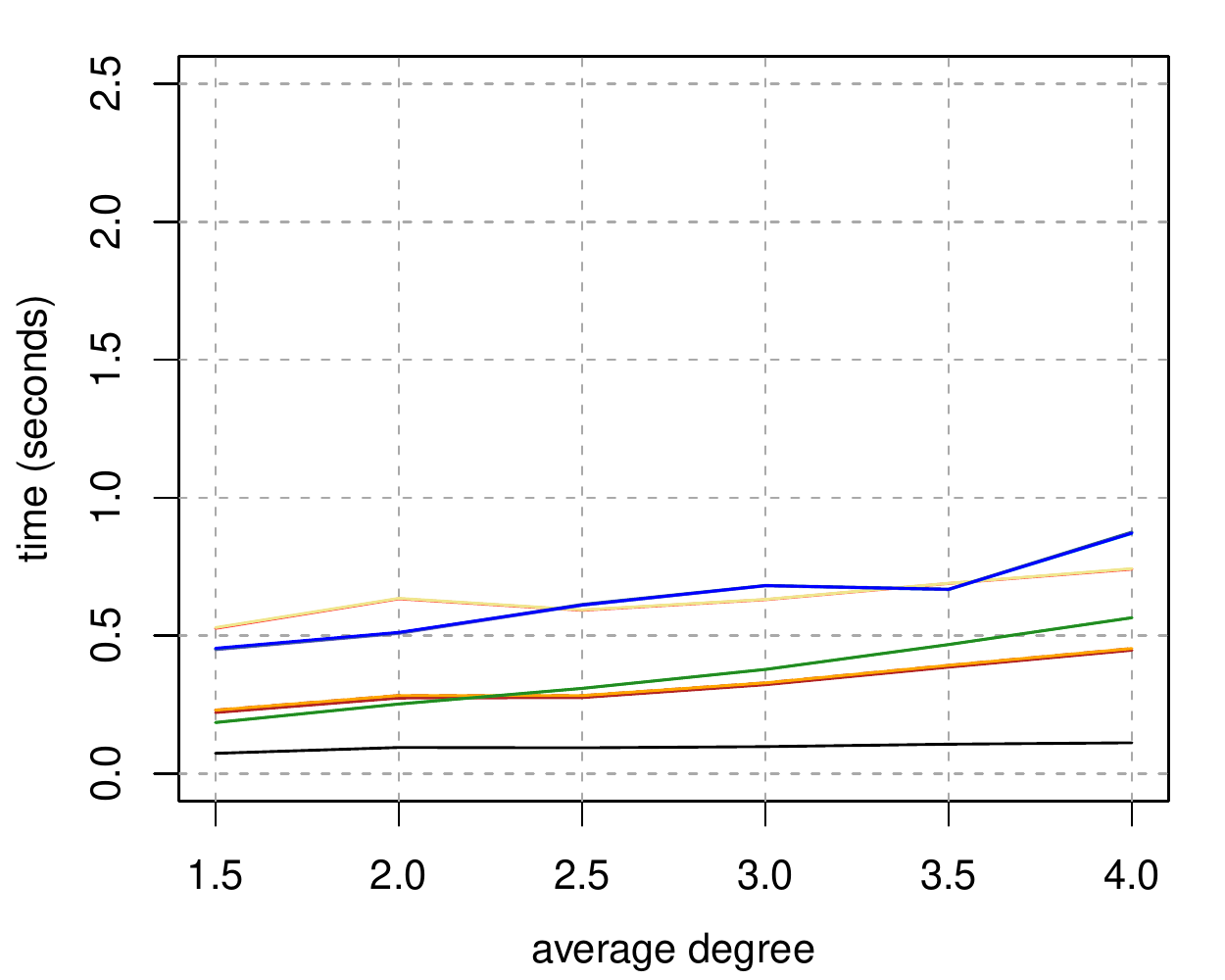}
		\end{minipage}%
	}%
	\hspace{0.01\textwidth}
	\subfigure[time, $N=(500,150)$]{
		\begin{minipage}[t]{0.3\textwidth}
			\centering
			\includegraphics[width=\textwidth]{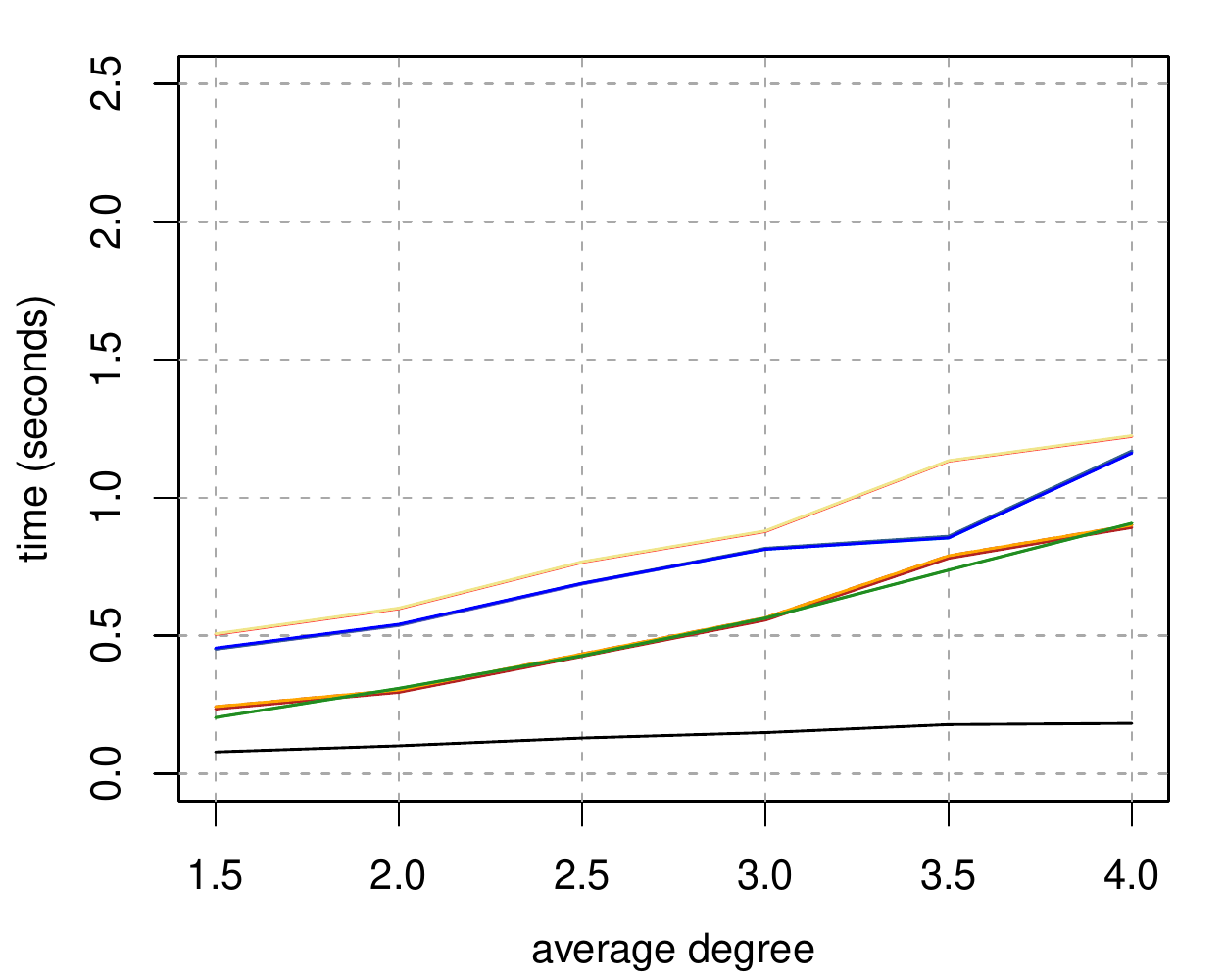}
		\end{minipage}%
	}%
	
	\caption{The experimental results of the hybrid method on $100$-node graphs with positive weights.}
	\label{fig:gdes}
\end{figure}

{
We also tested the hybrid method which checks whether $X$ is a definite non-cause of $Y$ based on a learned CPDAG and then calls a CE-based method if $X$ is not a definite non-cause of $Y$. We combined the hybrid method with the PC algorithm, and used two CE-based methods, including ``IDA + test (all)" and ``IDA + test (min/max)" to deal with the case where  $X$ is not a definite non-cause of $Y$.

Figure~\ref{fig:gdes} demonstrates the results on $100$-node graphs with positive weights. For comparison, we also include the results of the non-hybrid CE-based methods combined with PC. Considering the Kappa coefficients, the hybrid methods are slightly better than the non-hybrid CE-based methods that utilize non-ancestral relations when the sample size is relatively large. This is because that the hybrid methods take the advantage of the correctly learned causal graphs. On the other hand, since the two hybrid methods also need an entire CPDAG, their total computational time is similar to that of ``PC + IDA + test (all)" and ``PC + IDA + test (min/max)", respectively.}


\subsection{Confidence Intervals}

\begin{figure}[t!]
	\centering
	
	\subfigure[$d=1.5$]{
		\begin{minipage}[t]{0.3\textwidth}
			\centering
			\includegraphics[width=\textwidth]{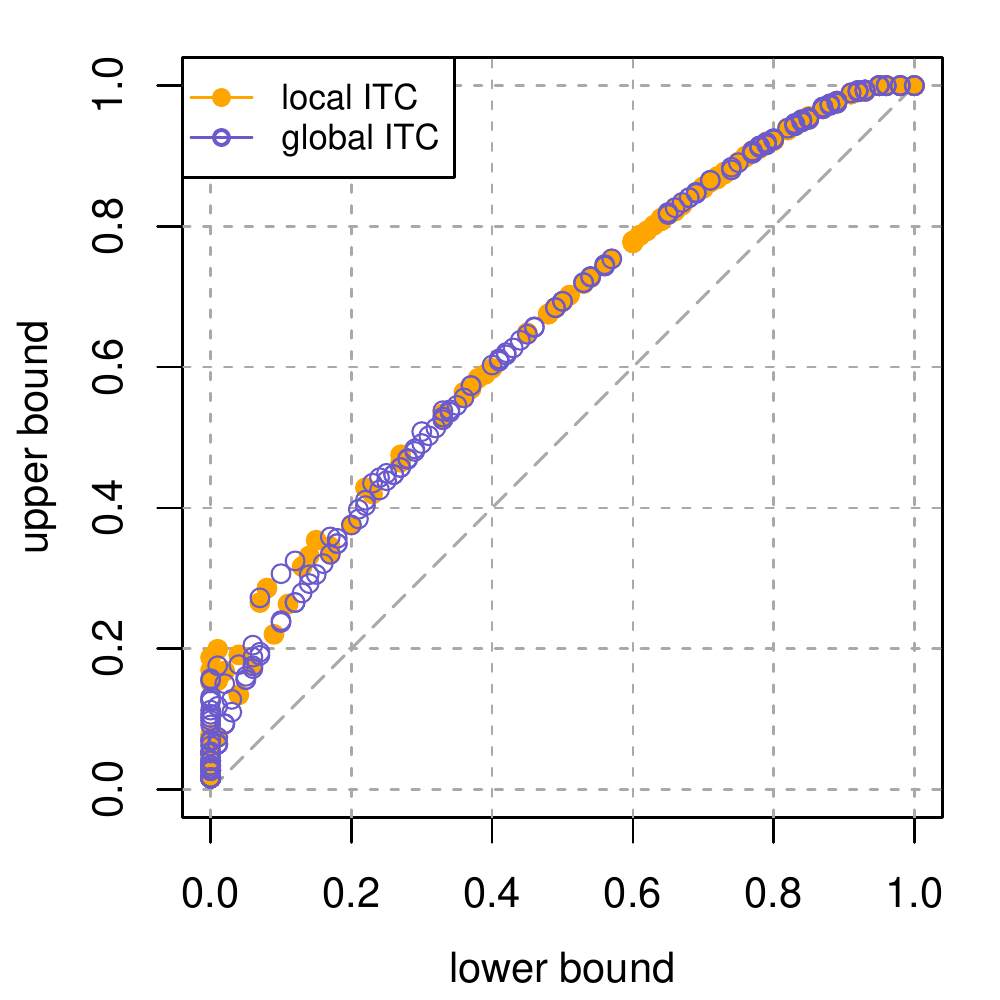}
		\end{minipage}%
	}%
	\hspace{0.01\textwidth}
	\subfigure[$d=2.0$]{
		\begin{minipage}[t]{0.3\textwidth}
			\centering
			\includegraphics[width=\textwidth]{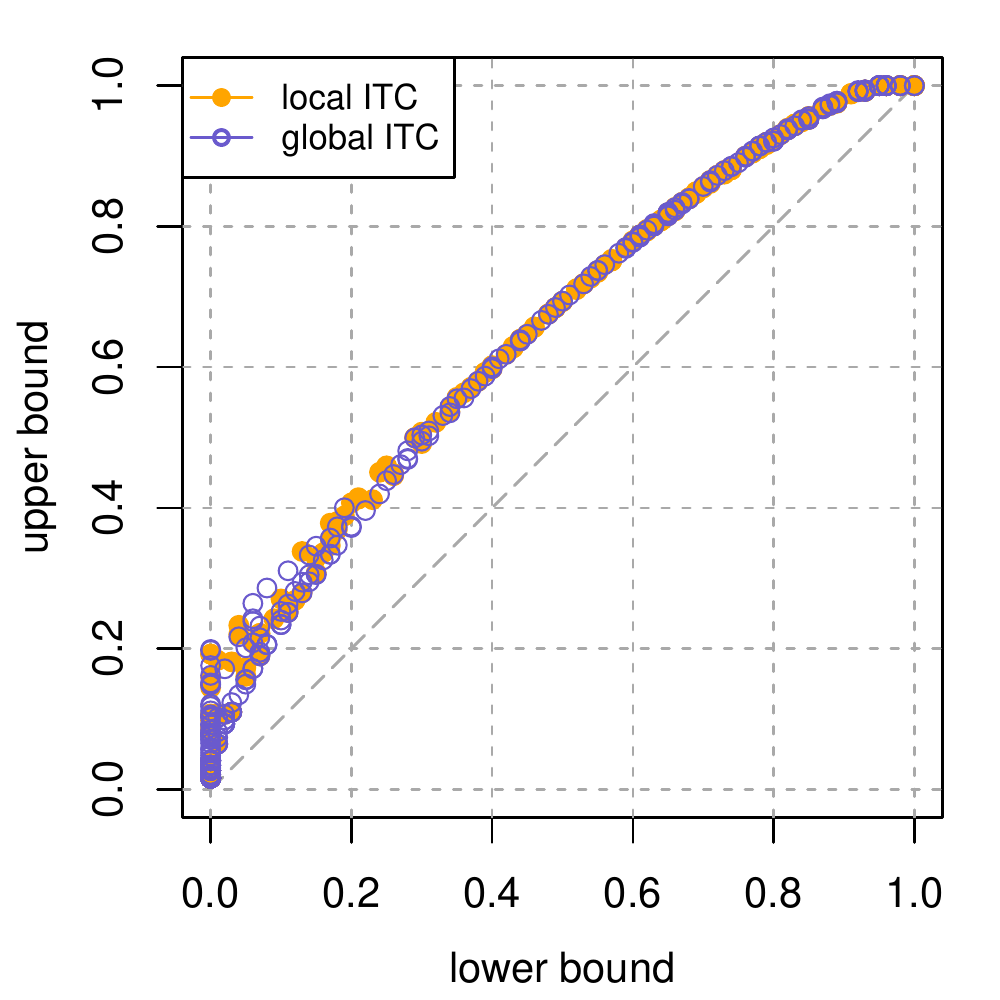}
		\end{minipage}%
	}%
	\hspace{0.01\textwidth}
	\subfigure[$d=2.5$]{
		\begin{minipage}[t]{0.3\textwidth}
			\centering
			\includegraphics[width=\textwidth]{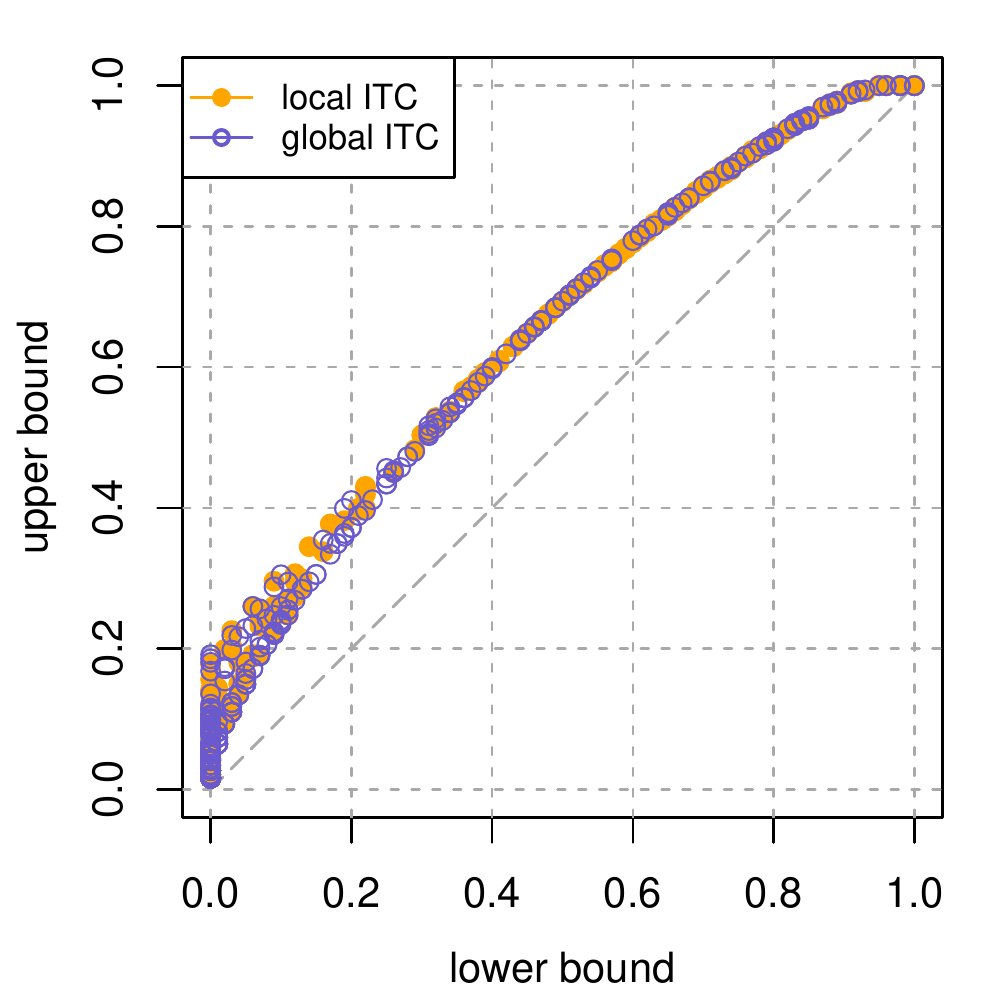}
		\end{minipage}%
	}%
	
	\subfigure[$d=3.0$]{
		\begin{minipage}[t]{0.3\textwidth}
			\centering
			\includegraphics[width=\textwidth]{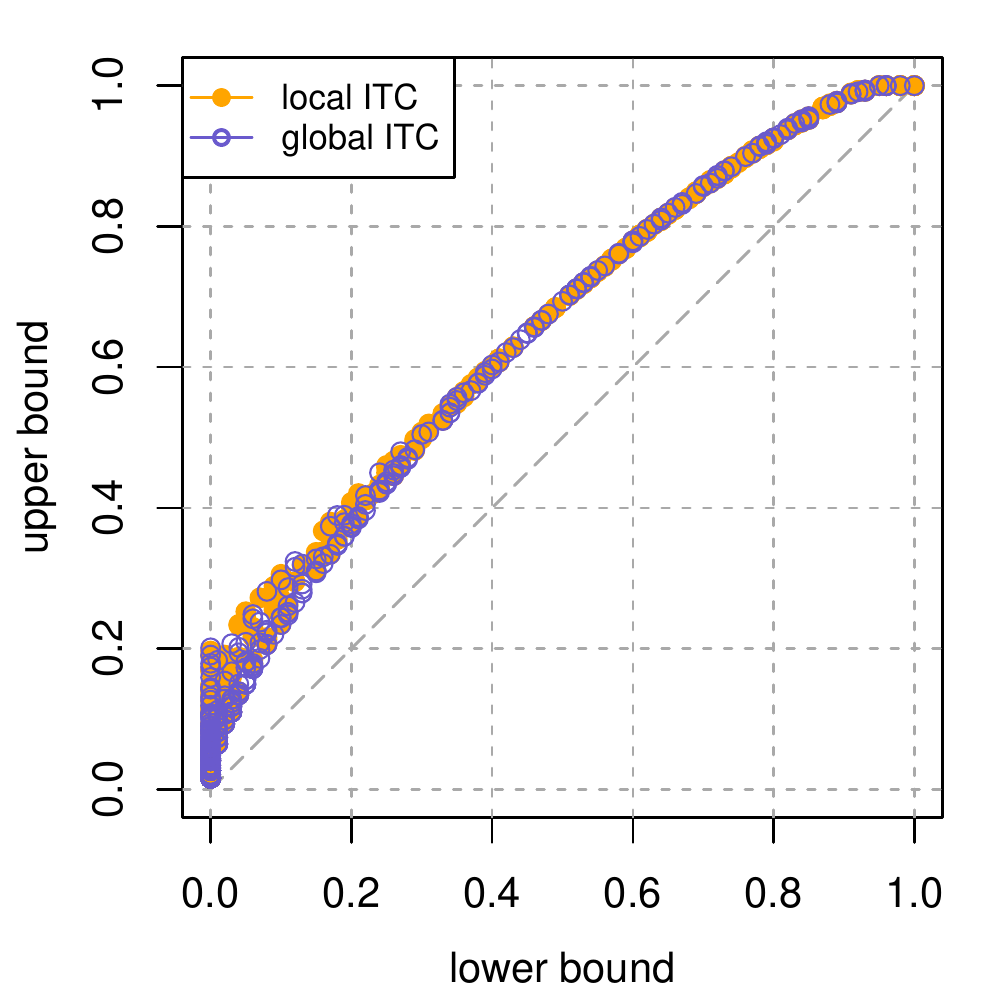}
		\end{minipage}%
	}%
	\hspace{0.01\textwidth}
	\subfigure[$d=3.5$]{
		\begin{minipage}[t]{0.3\textwidth}
			\centering
			\includegraphics[width=\textwidth]{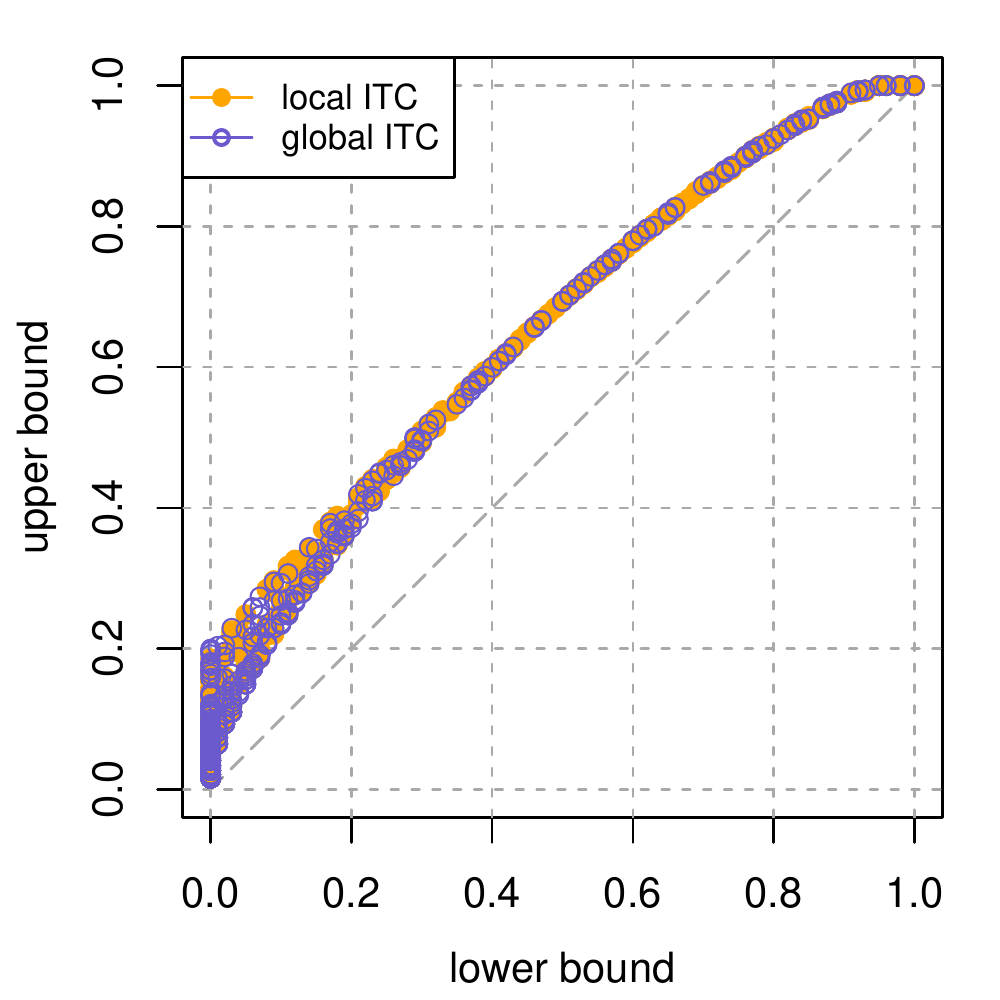}
		\end{minipage}%
	}%
	\hspace{0.01\textwidth}
	\subfigure[$d=4.0$]{
		\begin{minipage}[t]{0.3\textwidth}
			\centering
			\includegraphics[width=\textwidth]{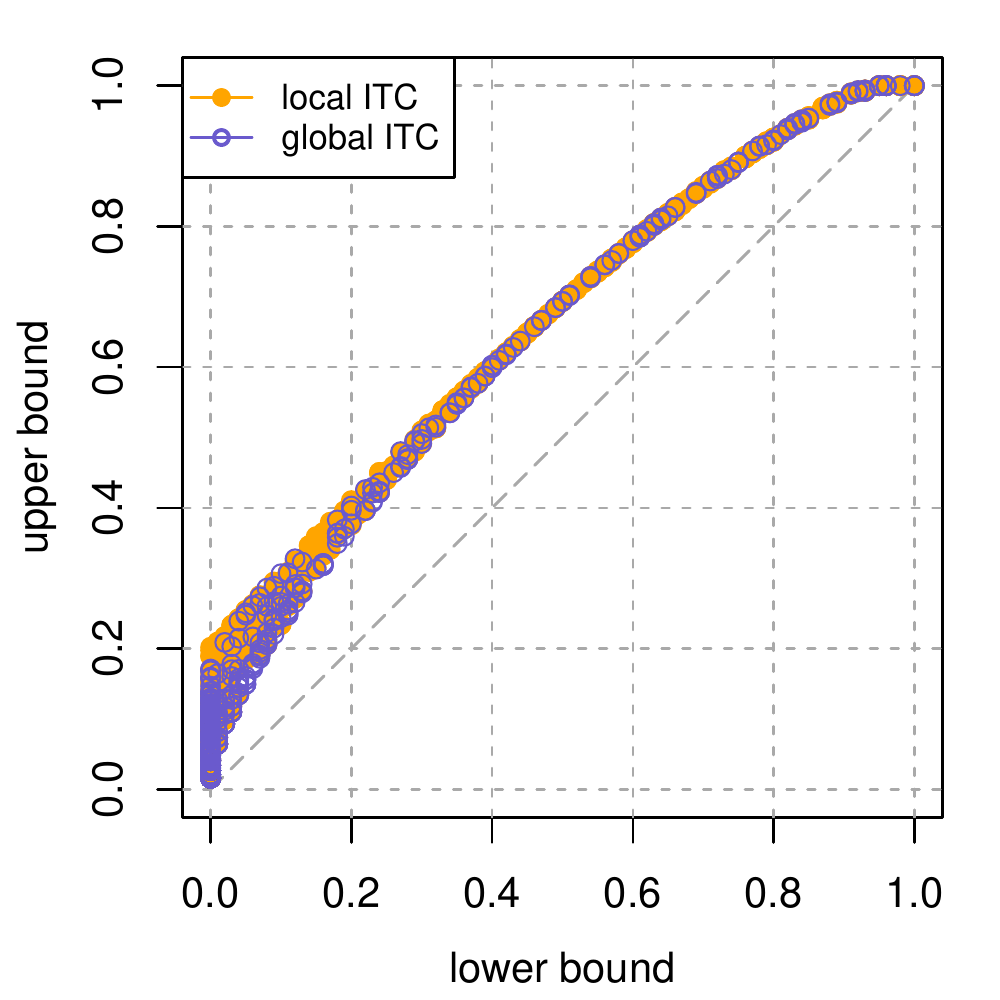}
		\end{minipage}%
	}%
	
	\caption{The estimated confidence intervals with the confidence level of 95\% for graphs with different average degrees ($d$), each of which is plotted as a point, with the form of ``(lower bound, upper bound)". Other settings: $n=50$, positively weighted, $N=(100, 100)$ and $m=100$.}
	\label{fig:learn:ci}
\end{figure}

The identification of types of causal relations can be regarded as a classification problem. Using re-sampling techniques, we may estimate the probability of each type of causal relation for a pair of treatment and target, as well as its confidence interval, which can be used to analysis how an inferred graph structure affects the performance of the global and local ITC.

Following the work of~\cite{Friedman19bootstrap}, for a given data set with $N_{\rm graph}$ observations for learning graphs and another $N_{\rm effect}$ observations for inferring types of causal relations, we first sample $m$ re-sampling data set, each of which contains $N_{\rm graph}$ observations sampled from the data for learning graphs with replacement. Then, for each re-sampling data set, we learn a graph structure, using either a global method such as PC, or a local method such as the variant of MB-by-MB. Finally,   we use the sub-dataset with $N_{\rm effect}$ observations to estimate the type of causal relation.  The above procedure results in a multinomial distribution with three categories. The point estimation of the probability for each category as well as its confidence interval can then be estimated from these results.

We focus on the definite and possible causal relations, whose proportions are usually smaller than 10\%   as suggested by Tables~\ref{tab:positive} and~\ref{tab:mixed}. For ease of demonstration, for a pair of treatment and target, we only   estimate the probability and its confidence interval of the true type    with the confidence level of 95\%.
We run  $5,000$ repeats on $50$-node, positive weight graphs for each average degree $d$ and each method.  
Figure~\ref{fig:learn:ci} shows the results  with $N=(N_{\rm graph}, N_{\rm effect})=(100, 100)$ and $m=100$. It can be seen that many points of the local ITC are concentrated at the upper right corner while many points of the global ITC are concentrated at the lower left corner, meaning that the local ITC is more accurate. Moreover, both methods give about the same length of confidence intervals when they   identify  the causal relations correctly. For instance, when  $d=4$ and the lower bounds of the confidence intervals are greater than $0.5$, the average length  of these confidence intervals of the local and global ITC are $0.122$ and $0.118$ respectively.







\bibliographystyle{abbrvnat}
\bibliography{ref220105}








\end{document}